\documentclass{article}

    \usepackage[preprint,nonatbib]{neurips_2022}

\usepackage[utf8]{inputenc} 
\usepackage[T1]{fontenc}    
\usepackage{hyperref}       
\usepackage{url}            
\usepackage{booktabs}       
\usepackage{amsfonts}       
\usepackage{nicefrac}       
\usepackage{microtype}      
\usepackage{xcolor}         

\usepackage{amsfonts}
\usepackage{amsmath}
\usepackage{amssymb}
\usepackage{amsthm}
\usepackage{algorithmic}
\usepackage[ruled,vlined,noend]{algorithm2e}
\usepackage{wrapfig}
\usepackage{mathtools}
\usepackage{subcaption}
\usepackage{array}

\newcommand{\ent}{\mathcal{H}}

\newtheoremstyle{questionstyle}
  {\topsep}   
  {0}         
  {\itshape}  
  {0pt}       
  {\bfseries} 
  {.}         
  {5pt plus 1pt minus 1pt} 
  {}          
\theoremstyle{questionstyle}\newtheorem{question}{Question}
\newtheorem{theorem}{Theorem}
\newtheorem{lemma}[theorem]{Lemma}
\newtheorem{corollary}{Corollary}
\newtheorem{definition}{Definition}

\newcommand{\prg}[1]{\textbf{#1}}

\newcommand{\bo}[1]{\textbf{#1}}

\newcommand{\tuple}[1]{\langle #1 \rangle}
\newcommand{\set}[1]{\{ #1 \}}
\newcommand{\expect}[2]{\mathbb{E}_{#1}\left[{#2}\right]}

\newcommand{\M}{\mathcal{M}}
\newcommand{\St}{\mathcal{S}}
\newcommand{\A}{\mathcal{A}}

\newcommand{\R}{R} 
\newcommand{\PS}{\mathcal{P}}

\DeclareMathOperator*{\argmax}{arg\!\max}

\newcommand{\voidarg}{{\,\cdot\,}}

\newcommand{\PD}{\textrm{PD}}
\newcommand{\CE}{\textrm{CE}}
\newcommand{\PE}{\textrm{PE}}
\newcommand{\KL}{D_\text{KL}}

\newcolumntype{M}[1]{>{\centering\arraybackslash}m{#1}}

\title{Maximum Entropy Population-Based Training for Zero-Shot Human-AI Coordination}

\author{Rui Zhao$^{1}$\thanks{Correspondence to: Rui Zhao {\tt\small $\lbrace$rui.zhao.ml@gmail.com$\rbrace$}.},\ Jinming Song$^{1}$,\ Yufeng Yuan$^{1}$,\\
\textbf{Haifeng Hu$^{1}$,\ Yang Gao$^{2}$,\ Yi Wu$^{2}$,\ Zhongqian Sun$^{1}$,\ Yang Wei$^{1}$}\\
$^{1}$Tencent AI Lab
$^{2}$Tsinghua University
}

\begin{document}

\maketitle

\begin{abstract}
We study the problem of training a Reinforcement Learning (RL) agent that is collaborative with humans without using any human data. Although such agents can be obtained through self-play training, they can suffer significantly from distributional shift when paired with unencountered partners, such as humans. To mitigate this distributional shift, we propose \emph{Maximum Entropy Population-based training} (MEP). In MEP, agents in the population are trained with our derived \emph{Population Entropy} bonus to promote both pairwise diversity between agents and individual diversity of agents themselves, and a common best agent is trained by paring with agents in this diversified population via prioritized sampling. The prioritization is dynamically adjusted based on the training progress. We demonstrate the effectiveness of our method MEP, with comparison to Self-Play PPO (SP), Population-Based Training (PBT), Trajectory Diversity (TrajeDi), and Fictitious Co-Play (FCP) in the Overcooked game environment, with partners being human proxy models and real humans. A supplementary video showing experimental results is available at \url{https://youtu.be/Xh-FKD0AAKE}.

\end{abstract}


\section{Introduction}
\label{sec:intro}
Deep Reinforcement Learning (RL) has gained many successes against humans in competitive games, such as Go~\cite{AlphaZero}, Dota~\cite{OpenAIFiveFinals}, and StarCraft~\cite{AlphaStar}.
However, it remains a challenge to build AI agents that can coordinate and collaborate with humans that the agents have not encountered during training~\cite{kleiman2016coordinate,lerer2017maintaining,carroll2019utility,shum2019theory,hu2020other,knott2021evaluating}.
This challenging problem, namely zero-shot human-AI coordination, is particularly important for real-world applications, such as cooperative games~\cite{carroll2019utility}, communicative agents~\cite{foerster2016learning}, self-driving vehicles~\cite{resnick2018vehicle}, and assistant robots~\cite{ShadowDexterousHand}, because it removes the onerous and expensive step of involving human or human data in AI training. Thus, studying this problem could potentially make our ultimate goal of building AI systems that can assist humans and augment our capabilities~\cite{AugmentingHumans, AugmentingHumansDistill} more achievable.

An efficient scheme for training AI agents in collaborative or competitive settings is through self-play reinforcement learning~\cite{tesauro1994td, AlphaZero}. Due to its training paradigm, self-play-trained agents are very specialized since they only encounter their own policies during training and assume their partners will behave in a particular way.
Therefore, those agents can suffer significantly from distributional shift when paired with humans.
For example, in the Overcooked game, the self-play-trained agents only use a specific pot and ignore the other pots while humans use all pots. As a consequence, the AI agent ends up waiting unproductively for the human to deliver a soup from the specific pot, even though the human has instead decided to fill up the other pots~\cite{carroll2019utility}.

In this paper, we propose a robust and efficient approach \textit{Maximum Entropy Population-based training} (MEP), to train agents for zero-shot human-AI coordination based on the advances in maximum entropy RL~\cite{haarnoja2018soft}, diversity~\cite{eysenbach2018diversity}, and Multi-agent RL~\cite{MADDPG,foerster2018counterfactual}.
To encourage the diversity and explorability of policies of the individual agent in the population, we utilize the maximum entropy objectives~\cite{ziebart2008maximum,toussaint2009robot,ziebart2010modeling,rawlik2013stochastic,fox2015taming,haarnoja2017reinforcement,haarnoja2018soft,zhao2019maximum} for the individual policies. To acquire diverse and distinguishable behaviors~\cite{eysenbach2018diversity} between agents in the population, we further utilize the average Kullback–Leibler (KL) divergence between all agent pairs in the population to promote pairwise diversity.
We define this combination of individual diversity and pairwise diversity as \textit{Population Diversity} (PD) and derive a safe and computationally efficient surrogate objective \textit{Population Entropy} (PE), which is the lower bound of the original PD objective with linear runtime complexity. Analogous to maximum entropy RL training, each agent in the population is rewarded to maximize the centralized population entropy. With this diverse population, we train a best response agent by pairing it with the agents sampled from this population with a prioritization scheme based on the difficulty to collaborate with~\cite{schaul2015prioritized,AlphaStar,vinyals2019grandmaster,han2020tstarbot}. By doing so, this newly trained AI agent encounters a diverse set of strategies and could have better generalization~\cite{pan2009survey,tobin2017domain,ShadowDexterousHand}.

The contributions of this paper are three-fold. First, based on the novel population diversity objective that considers both individual diversity and pairwise diversity for agents in the population, we derive a safe and computationally efficient surrogate objective, the population entropy, which is the lower bound of the population diversity objective. Secondly, we propose the maximum entropy population-based training framework, which comprises training a diverse population and then training a robust AI agent using this population. Last but not least, we evaluate our method and other state-of-the-art methods on the Overcooked game environment~\cite{Overcooked}, with both human proxy models and also real humans.


\section{Preliminaries}
\prg{Markov Decision Process:} A two-player Markov Decision Process (MDP) is defined by a tuple $\M = \tuple{\St, \set{\A^{(i)}}, \PS, \gamma, \R}$~\cite{multiMDP}, where $\St$ is a set of states; $\A^{(i)}$ is a set of the $i$-th agent's actions, where $i\in [1,2]$; $\PS$ is the transition dynamics that maps the current state and all agents' actions to the next state; $\gamma$ is the discount factor; $\R$ is the reward function. The $i$-th agent's policy is $\pi^{(i)}$. 
A trajectory is denoted by $\tau$. The shared objective is to maximize the expected sum rewards, which is $\expect{\tau}{\sum_t \R(s_t, a_t)}$, where $a_t = (a^{(1)}_t, a^{(2)}_t)$.
We can extend the objective to infinite horizon problems by the discount factor $\gamma$ to ensure that the sum of expected rewards is finite.
In the perspective of a single agent, the other agent can be treated as a part of the environment. In this case, we can reduce the process to the partially observable MDP (POMDP) for that particular agent.

\prg{AI Agent, Population, and Human:}
In the case of human-AI coordination, we have a two-player MDP, in which one player is human, and the other is AI.
Throughout this paper, we use the phrase \emph{AI agent} to explicitly denote the agent that plays the AI role in human-AI coordination. 
The \emph{population} of agents is used to train the AI agent to make it capable of cooperating with different partner agents.
The \emph{human} policy is represented as $\pi^{(H)}$ and a model of the human policy is $\hat{\pi}^{(H)}$.
The {AI agent} is denoted as $\pi^{(A)}$.

\prg{Environment:}
We use the Overcooked environment~\cite{carroll2019utility} as the human-AI coordination testbed, see Figure~\ref{fig:overcooked}. In the Overcooked game, it naturally requires coordination and collaboration between the two players to have a high score. The players are tasked to cook soups.

\prg{Maximum Entropy RL:}
Standard reinforcement learning maximizes the expected sum of rewards $\expect{\tau}{\sum_t \R(s_t, a_t)}$,
At the beginning of learning, almost all actions have equal probability. 
After some training, some actions have a higher probability in the direction of accumulating more rewards. Subsequently, the entropy of the policy is reduced over time during training~\cite{mnih2016asynchronous}.
while maximum entropy RL augments the standard RL objective with the expected entropy of the policy~\cite{ziebart2010modeling,haarnoja2018soft}, which incentives the agent to select the non-dominate actions.
The maximum entropy RL objective is defined as:
\begin{align}
\label{eq:maxent_objective}
J(\pi) = \sum_t \expect{(s_t, a_t) \sim \pi}{\R(s_t, a_t) + \alpha \ent(\pi(\voidarg | s_t))},
\end{align}
where parameter $\alpha$ adjusts the relative importance of the entropy bonus against the reward. 
The maximum entropy RL objective has several advantages.
First, the policy favors more exploration and mitigates the issue of early convergence~\cite{haarnoja2017reinforcement,schulman2017equivalence}.
Secondly, the policy can capture multiple modes of near-optimal behaviors and has better robustness~\cite{haarnoja2018latent,haarnoja2019learning}.

\section{Method}
\label{sec:method}
\begin{figure*}
    \centering
    \begin{minipage}{0.45\linewidth}
    \vskip 0.2in
        \includegraphics[width=\linewidth]{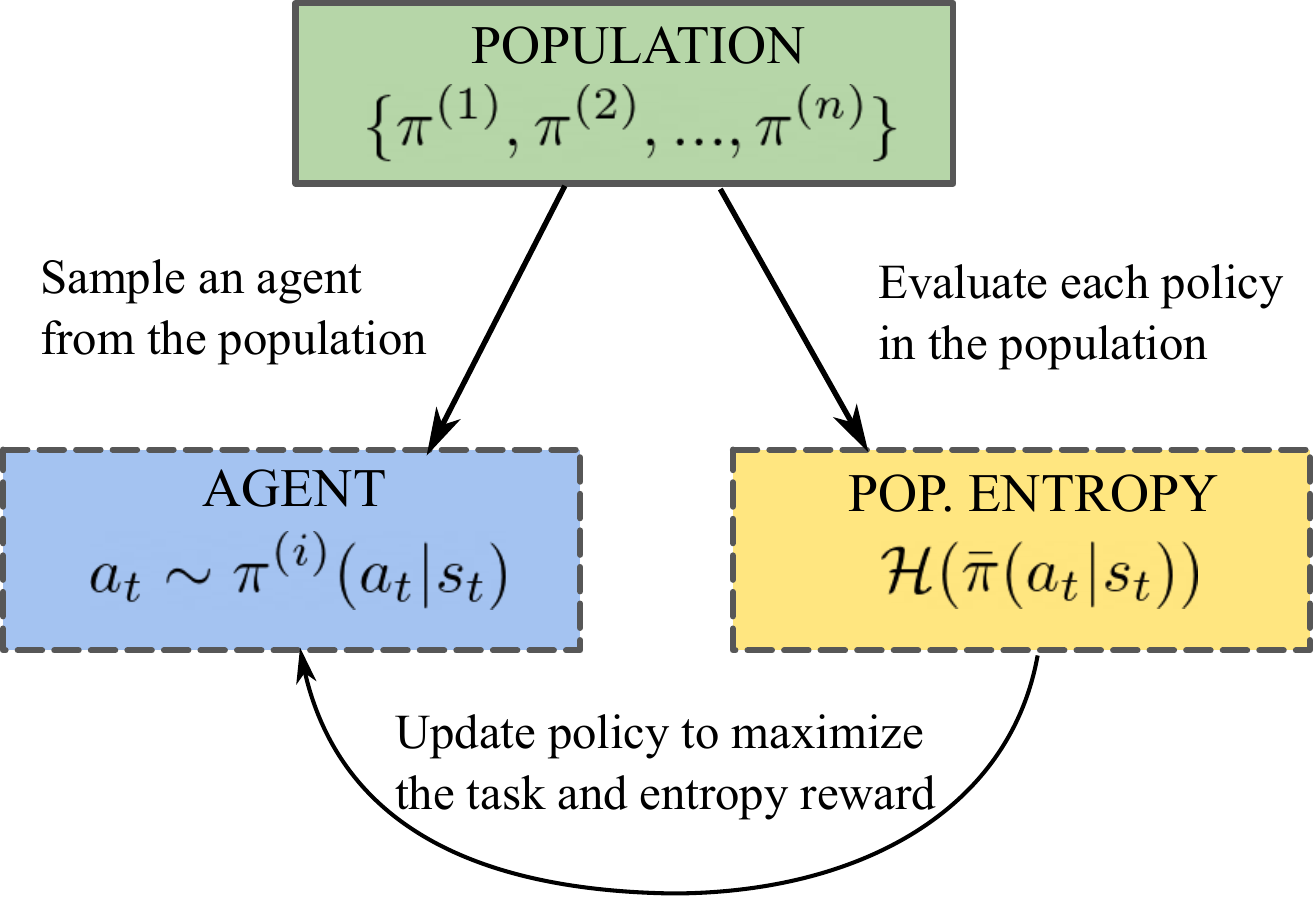}
    \end{minipage} 
    \begin{minipage}{0.54\linewidth}
    \begin{algorithm}[H]     
    \DontPrintSemicolon
    \SetAlgoLined
    \While{not converged}{
        Sample agent from population: $\pi^{(i)} \sim \set{\pi^{(1)}, \pi^{(2)}, ..., \pi^{(n)}}$\\
        \For{$t \leftarrow 1$ \KwTo $steps\_per\_episode$}{
                Sample action $a_t \sim \pi^{(i)}(a_t | s_t)$.\\
                Step environment $s_{t+1} \sim p(s_{t+1} \mid s_t, a_t)$.\\
                Calculate the population entropy reward and combine it with the task reward:\\$r = r(s_t,a_t) - \alpha \log(\bar{\pi}(a_t|s_t))$\\
                Update policy $\pi^{{(i)}}$ to maximize $\expect{\tau}{r}$.
              }
    } 
    \caption{Maximum Entropy Population}\label{algo:mep}
    \end{algorithm}
    \end{minipage}
    \caption{\textbf{Maximum Entropy Population}:
    We train each agent in the population to maximize its task reward as well as the population entropy reward to attain a maximum entropy population.\label{fig:mep}}
\end{figure*}
In this section, we first define the \textit{Population Diversity} objective, which includes average individual policy entropy and average pairwise difference among policies.
Secondly, we derive its safe and computationally efficient lower bound, \textit{Population Entropy}, as the surrogate objective for optimization.
Thirdly, we illustrate the \textit{Maximum Entropy Population-based training} framework, which comprises training a maximum entropy population and training a robust AI agent via prioritized sampling using the population.

\subsection{Population Diversity}
\label{sec:pd}
Motivated by maximum entropy RL, we want to make the policies in the population exploratory and diverse.
First, by utilizing the maximum entropy bonus, we encourage each policy itself to be exploratory and multi-modal.
Secondly, to encourage the policies $\set{\pi^{(1)}, \pi^{(2)}, ..., \pi^{(n)}}$ in the population to be complementary and mutually different, we utilize the KL divergence of each policy pair in the population as part of our objective. 
Formally, we define the \textit{Population Diversity} (PD) as a combination of the average entropy of each agent's policy and the average KL-divergence between each agent pair in the population. 
Mathematically, 
\begin{align}
\label{eq:population_diversity}
\PD(\set{\pi^{(1)}&, \pi^{(2)}, ..., \pi^{(n)}}, s_t) \coloneqq
\frac{1}{n}\sum_{i=1}^n \ent(\pi^{(i)}(\voidarg | s_t)) + \frac{1}{n^2}\sum_{i =1}^n \sum_{j=1}^n D_\text{KL} (\pi^{(i)}(\voidarg | s_t), \pi^{(j)}(\voidarg | s_t))  
,
\end{align}
where KL-divergence ($\KL$) and entropy ($\ent$) are defined as follows:
\begin{align}
&\KL(\pi^{(i)}(\voidarg | s_t), \pi^{(j)}(\voidarg | s_t))=\sum_{a\in\mathcal{A}} \pi^{(i)}(a_t|s_t) \log \frac{\pi^{(i)}(a_t|s_t)}{ \pi^{(j)}(a_t|s_t)} , \\
&\ent(\pi^{(i)}(\voidarg|s_t)) = - \sum_{a\in\mathcal{A}} \pi^{(i)}(a_t|s_t) \log \pi^{(i)}(a_t|s_t) .
\end{align}
Although the PD objective not only captures a single agent's explorability but also encourages agents' policies to be mutually distinct, evaluating this objective requires a quadratic runtime complexity of $O(n^2)$, where $n$ is the population size. Besides, as the KL-divergence is unbounded, optimizing this objective as part of the reward function may lead to convergence issues.

\subsection{Population Entropy}
To improve the stability and the runtime complexity of the PD objective, we derive a bounded and efficient surrogate objective \textit{Population Entropy} (PE) for optimization, which is defined as the entropy of the mean policies of the population. Mathematically, 
\begin{align}
\label{eq:population_entropy}
&\PE(\set{\pi^{(1)}, \pi^{(2)}, ..., \pi^{(n)}}, s_t) \coloneqq \ent(\bar{\pi}(\voidarg | s_t)) ,  \text{where}\ \bar{\pi}(a_t | s_t) \coloneqq \frac{1}{n} \sum_{i=1}^n \pi^{(i)}(a_t|s_t) .
\end{align}
PE serves as a lower bound of the PD objective.
\begin{theorem}
Let the population diversity be defined as Equation~(\ref{eq:population_diversity}).
Let the population entropy be defined as Equation~(\ref{eq:population_entropy}).
Then, we have 
\begin{align}
&\PD(\set{\pi^{(1)}, \pi^{(2)}, ..., \pi^{(n)}}, s_t) \geq \PE(\set{\pi^{(1)}, \pi^{(2)}, ..., \pi^{(n)}}, s_t) ,
\end{align}
where $n$ is the population size.
$\textit{Proof.}$ See Appendix A. {\hfill $\square$}
\end{theorem}
The PE objective serves as a lower bound for the PD objective, which requires only a linear runtime complexity $O(n)$. Moreover, when defined on categorical distribution, the PE objective is bounded, which makes it desirable to be optimized as part of the reward function. Therefore, we use the derived PE objective for optimization.

\subsection{Training a Maximum Entropy Population}
\label{sec:train_maximum_entropy_population}
With the PE objective, we can train a population of agents, which can cooperate well with each other with mutually distinct strategies. Therefore, similar to the objective in maximum-entropy RL, we define the objective for MEP training as follows:
\begin{align}
\label{eq:maxent_pbt}
J(\bar \pi) = \sum_t \expect{(s_t, a_t) \sim \bar \pi}{\R(s_t, a_t) + \alpha \ent(\bar \pi(\voidarg | s_t))},
\end{align}
where $\bar \pi$ is the mean policy of the population and $\alpha$ determines the relative weight of the population entropy term with respect to the task reward.
As $\bar\pi(a_t | s_t)$ can be written as $\frac{1}{n} \sum_{i=1}^n \pi^{(i)}(a_t|s_t)$, to optimize Equation~(\ref{eq:maxent_pbt}), we can uniformly sample agents from the population to maximize the augmented reward function $R(s_t, a_t) - \alpha \log\bar \pi(a_t|s_t)$. 
The task reward is related to the agent and its partner agent, which is a copied version of itself playing the partner role in our case.
When the centralized PE reward is calculated, it considers all the agents in the population.
We summarize the method of training a maximum entropy population in Algorithm~\ref{algo:mep} and Figure~\ref{fig:mep}.

\begin{figure*}
\centering
\includegraphics[height=.142\linewidth]{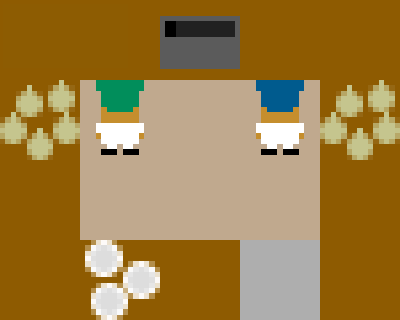}
\includegraphics[height=.142\linewidth]{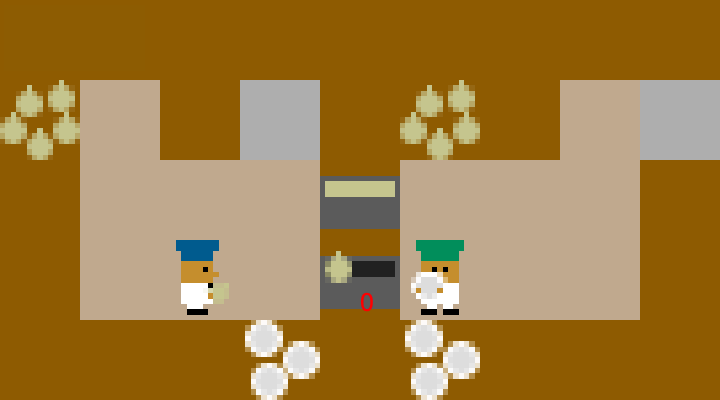}
\includegraphics[height=.142\linewidth]{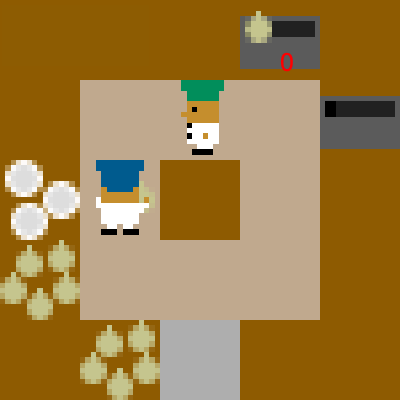}
\includegraphics[height=.142\linewidth]{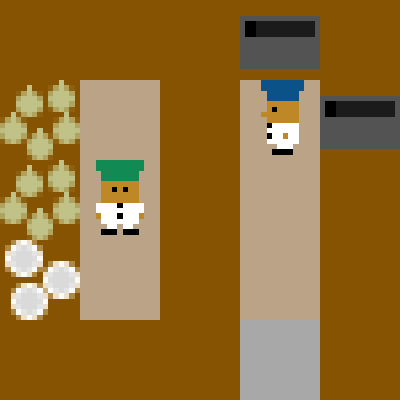}
\includegraphics[height=.142\linewidth]{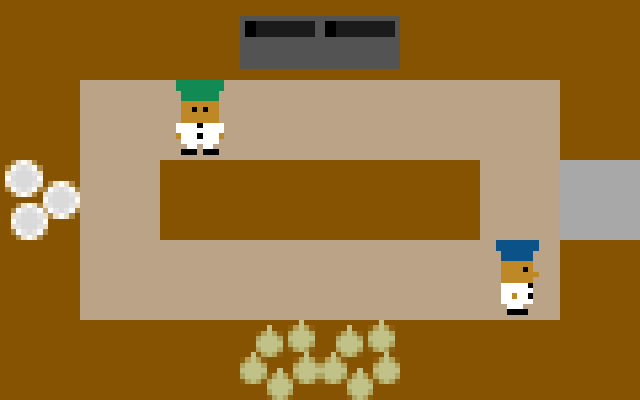}
\caption{\bo{Overcooked environment}: From left to right, the layouts are \emph{Cramped Room}, \emph{Asymmetric Advantages}, \emph{Coordination Ring}, \emph{Forced Coordination}, and \emph{Counter Circuit}.}
\label{fig:overcooked}
\end{figure*}

After having the maximum entropy population, we utilize this diverse set of agents to train a robust AI agent that can be readily paired with any player, including real human. 
The intuition behind MEP is that the AI agent should be more robust when paired with a group of diversified partners during training than trained only via only self-play.
In the extreme case, when the AI agent can coordinate well with an infinite set of different partners, it can also collaborate well with the human players.
In a more realistic sense, the more diverse the population is, the more likely the population covers most of the human behaviors in the training set.
Subsequently, the final AI agent should be less ``panicked'' when facing ``abnormal'' human actions.

\subsection{Training a Robust Agent via Prioritized Sampling}
\label{sec:prioritized_sampling}
With the maximum entropy population, training this robust agent is still non-trivial. If we train the robust agent $(A)$ by pairing it with $i$-th agent uniformly sampled from the population, the resulting policy gradient is:
\begin{align*}
\frac{1}{n}\sum_{i=1}^{n}\expect{\tau}{\sum_t \nabla_{\theta}\log\pi^{(A)}_{\theta}(a^{(A)}_{t}|s_t) \sum_t \R(s_t, a^{(A)}_t,  a^{(i)}_t)},
\end{align*}
where $n$ is the population size. This "maximize average" paradigm may lead to an agent $(A)$ that exploits the easy-to-collaborate partners as pairing with them will inevitably lead to much higher return, and therefore, ignores the hard-to-collaborate partners, which is orthogonal to our intention of training an robust agent. However, motivated by Vinyals et al.~\cite{vinyals2019grandmaster}, we can mitigate this issue by skewing the sampling distribution as follows: 
\begin{align*}
p(\pi^{(i)}) \propto 1 / \expect{\tau}{\sum_t \R(s_t, a^{(A)}_t,  a^{(i)}_t)},
\end{align*}
We assign a higher priority to the agents that are relatively hard to collaborate with. 
By doing so, we change the "maximize average" paradigm to a smooth approximation of "maximize minimal" paradigm to mitigate the issue of over exploitation of easy-to-collaborate partners.
In the extreme case, at each optimization step, if we always choose the hardest agent in the population to train the AI agent, we optimize a performance lower bound of the cooperation between the AI agent and any agent in the population. Mathematically,
\begin{align}
\label{eq:ps_lb_main}
\pi^{(A)} = \mathop{\arg\max} \min_{i\in \{1,...,n\}}J(\pi^{(A)}, \pi^{(i)}) ,
\end{align}
where $J(\pi^{(A)}, \pi^{(i)})$ denotes the expected return achieved by $\pi^{(A)}$ and $\pi^{(i)}$ collaborating with each other.
For more detail on the performance lower bound, Equation~(\ref{eq:ps_lb_main}), see Lemma 3 in Appendix B.
Furthermore, we derive the performance connection between two pairs of agents, $(\pi^{(A)}, \pi^{(i)})$ and $(\pi^{(A)}, \pi^{(j)})$, when the partner agent $\pi^{(i)}$ in the first pair is $\epsilon$-close~\cite{ko2006mathematical} to the other partner agent $\pi^{(j)}$ in the second pair, see Lemma 4 in Appendix C.
Based on Lemma 4, if the population we used for training is diverse and representative enough such that we can find an agent that is $\epsilon$-close to the human player's policy, then we have a performance lower bound of human-AI coordination.
In this case, with prioritized sampling, we optimize not only the performance lower bound of the AI agent and the population, see Equation~(\ref{eq:ps_lb_main}), but also the performance lower bound between the human player and the AI agent, see Corollary 1 in Appendix C.

In practice, we follow the rank-based approach as proposed by Schaul et al.~\cite{schaul2015prioritized} in consideration for stability. The probability of the $i$-th agent to be sampled is:
\begin{align*}
p(\pi^{(i)}) = \frac{\textrm{rank}\left(1 / \expect{\tau}{\sum_t \R(s_t, a^{(A)}_t,  a^{(i)}_t)}\right)^{\beta}}{\sum_{j=1}^n \textrm{rank} \left( 1 / \expect{\tau}{\sum_t \R(s_t, a^{(A)}_t,  a^{(j)}_t)} \right)^{\beta}},
\end{align*}
where $\mathrm{rank}(\cdot)$ is the ranking function ranging from $1$ to $n$ and $\beta$ is a hyper-parameter for adjusting the strength of the prioritization.


\section{Experiments}
\begin{figure*}
\begin{subfigure}[b]{0.3\textwidth}
  \centering
   \includegraphics[width=.82\linewidth]{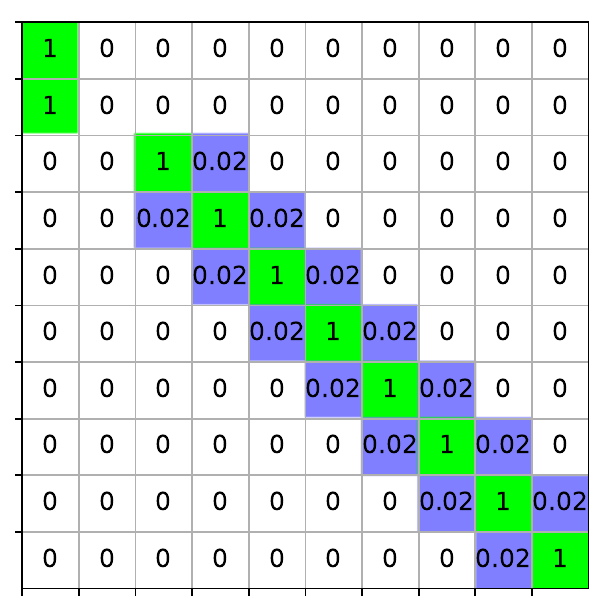}
   \vspace{1.5mm}
\end{subfigure}
\begin{subfigure}[b]{0.6\textwidth}
\includegraphics[height=.45\linewidth]{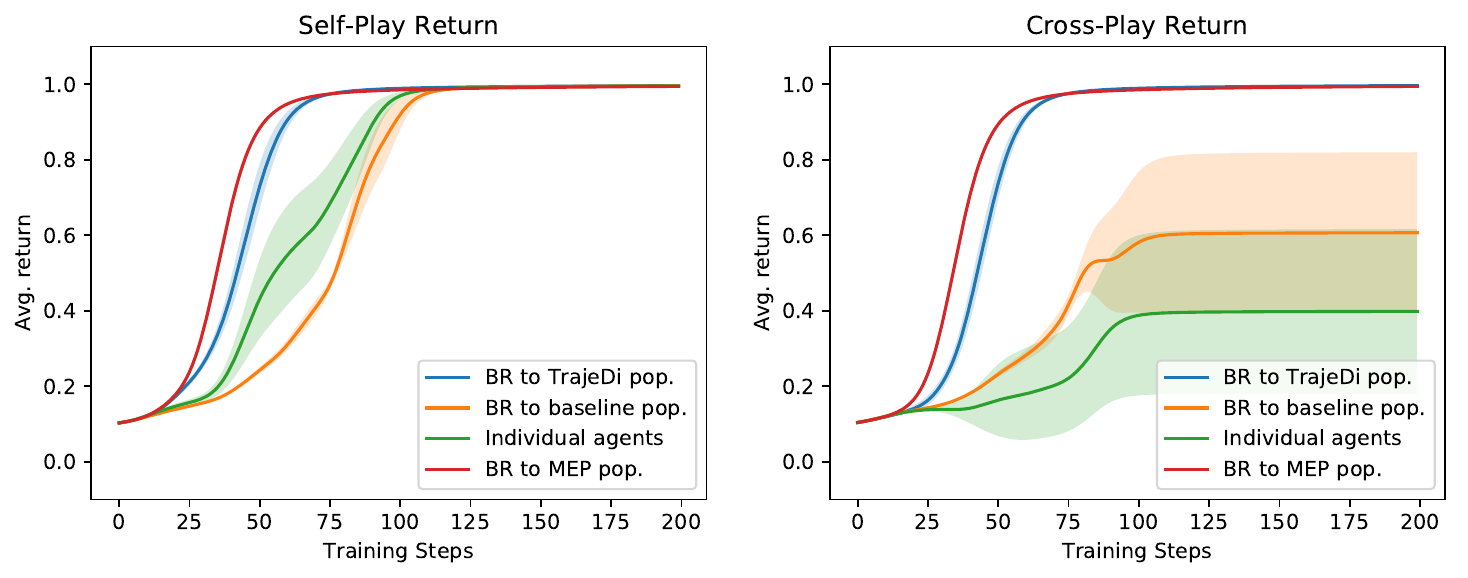}
\end{subfigure}
\caption{\bo{Performance comparison}: Training and test performances in the matrix game. Shown are the results for Best Responses (BRs) to MEP agents, BRs to TrajeDi populations, BRs to baseline populations, and individual agents. MEP allows faster learning compared to TrajeDi and others.}
\label{fig:matrix_game}
\end{figure*}
\prg{Environment:}
To evaluate the proposed method, we first use a toy environment, the matrix game~\cite{lupu2021trajectory}, see Figure~\ref{fig:matrix_game}, and then use the Overcooked environment~\cite{carroll2019utility}, see Figure~\ref{fig:overcooked}. The Overcooked game naturally requires human-AI coordination to achieve a high score.
The players are tasked to cook the onion soups as fast as possible.
The relevant objects are onions, plates, and soups. 
Players are required to place three onions in a pot, cook them for 20 timesteps, put the cooked soup on a plate, and serve the soup. 
Afterward, the players receive a reward of 20. 
The six actions are up, down, left, right, no-operation, and interact.
There are five different layouts, see Figure~\ref{fig:overcooked}, and each exhibits a unique challenge. 
For example, in \emph{Asymmetric Advantages}, the player on the left has the advantage to deliver the soup while the player on the right is closer to the onions. Good players should discover their advantages and play to their strengths.

\bo{Experiments:}
First, we train the population using the PE bonus and investigate the effect of the entropy weight $\alpha$. 
Secondly, we use the learned maximum entropy population to train the AI agent with the learning progress-based prioritized sampling and report the performance.
In the ablation study, we show the effectiveness of both population entropy and prioritized sampling. 
We also show the comparison between MEP and Maximum Population Diversity (MPD), which maximizes the population diversity objective, see Section~\ref{sec:pd}.
We compare our method with other methods, including Self-Play (SP) Proximal Policy Optimization~\cite{PPO,carroll2019utility}, Population Based Training (PBT)~\cite{PBT,carroll2019utility}, Trajectory Diversity (TrajeDi)~\cite{lupu2021trajectory}, Fictitious Co-Play (FCP)~\cite{strouse2021collaborating}, and MPD.
To evaluate these methods, we use the protocol proposed by Carroll et al.~\cite{carroll2019utility}, in which a human proxy model, $H_{Proxy}$, is used for evaluation.
The human proxy model is trained through behavior cloning~\cite{BehavioralCloning} on the collected human data.
Furthermore, we conduct a user study using Amazon Mechanical Turk (AMT), in which we deploy our models through web interfaces and let real human players play with the AI agents. 
We present both the quantitative results and the qualitative findings.
The experimental details are shown in Appendix D.
Our code is available at \url{https://github.com/ruizhaogit/maximum_entropy_population_based_training}.

\begin{question}
How does MEP perform in toy environments?
\end{question}
In the single-step collaborative matrix game~\cite{lupu2021trajectory}, player 1 must select a row while player 2 chooses a column independently. Both agents get the reward associated with the intersection of their choices at the end of the game. We use the same evaluation protocol as proposed by Lupu et al.~\cite{lupu2021trajectory}. 
As shown in Figure~\ref{fig:matrix_game}, MEP converges faster than TrajeDi in both self-play return and cross-play return. An extensive hyper-parameter search for TrajeDi can be found in Figure 6 in Appendix E.

\begin{table*}
\centering
\caption{\bo{Population entropy with different $\alpha$}: In this table, $\alpha$ denotes the weight of the PE reward in Equation~(\ref{eq:maxent_pbt}).}    
\begin{tabular}{c c c c c c}
\toprule
 \ \ \ \ $\alpha$ & Cramped Rm.      & Asymm. Adv.     & Coord. Ring     & Forced Coord.     & Counter Circ. \\
\midrule
0.000   & 0.971   & 1.120   & 0.878    & 0.970    & 0.988 \\
0.001   & 1.031   & 1.051   & 0.907    & 0.858    & 1.152 \\
0.005   & 0.949   & 1.075   & 0.901    & 0.889    & 1.038 \\
0.010   & 1.057   & 1.139   & 0.840    & 1.079    & 1.151 \\
0.020   & 1.029   & 1.074   & 0.947    & 1.093    & 1.171 \\
0.030   & 1.134   & 1.203   & 1.028    & 0.957    & 1.715 \\
0.040   & 1.194   & 1.353   & 1.122    & 1.460    & 1.791 \\
0.050   & 1.127   & 1.364   & 0.996    & 1.703    & 1.791 \\  \bottomrule

\end{tabular}
\label{tab:results}
\end{table*}

\begin{question}
Does PE reward increase the entropy of the population?
\end{question}

To check whether the PE reward increases the entropy of the population, we investigate the effect of different values of $\alpha$ and record the entropy of the population that corresponds to the best reward during training in Table~\ref{tab:results}. As shown in Table~\ref{tab:results}, the population entropy with $\alpha > 0$ is generally greater than that with $\alpha = 0$ and the overall trend is the population entropy increases as $\alpha$ gets larger. This empirical finding verifies that PE reward effectively increases the entropy of the population.

\begin{question}
What does an MEP population look like?
\end{question}
\label{que:population}
To have an intuition of what a maximum entropy population looks like, we show the behavior of the agents in the supplementary video from 0:01 to 0:21.
This video clip presents the population trained without and with the PE reward in the first and second row, respectively.
In the first row, the blue and green agents move in a synchronized way most of the time, and the routines among the five agent pairs are similar. 
For example, the blue agent first passes the onions to the green agent, and then, when the green agent puts all three onions in the pot, the blue agent starts to pass the plate. After that, the green agent delivers the onion soup. 
This kind of routine repeats till the end of the game.
However, in the second row of the video clip, the agents' behavior trained with population entropy reward is more diverse.
The movements of the blue agent and the green agent are less synchronized, and their routines are less predictable.
For example, in the second agent pair, the blue agent throws the onion into a random spot or passes the first and second plates in a row, and in the fifth agent pair, the green agent behaves in a highly unexpected way.
In general, we observe more diverse behaviors and stochasticity of the policies trained with population entropy rewards. 

\begin{figure*}
\centering
\begin{subfigure}[b]{1.\textwidth}
   \includegraphics[width=\linewidth]{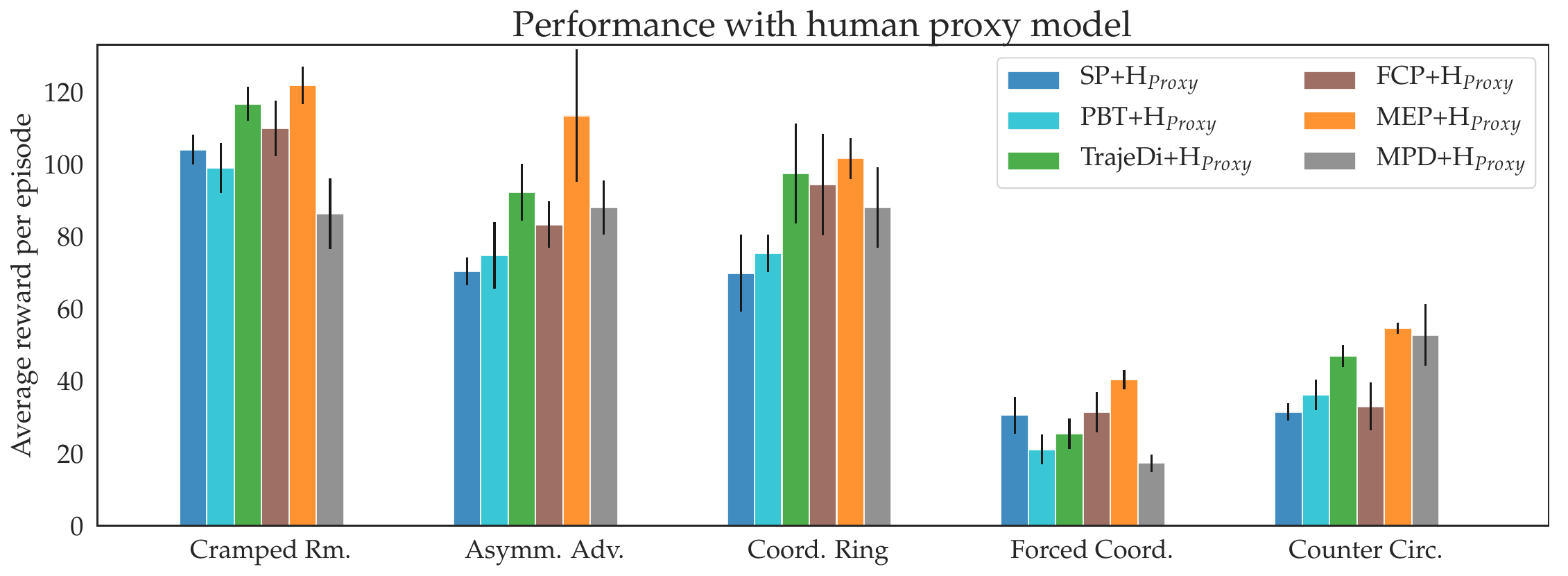}
   \caption{\bo{Performance Comparison}}
   \label{fig:MEP_human_proxy_performance} 
\end{subfigure}
\begin{subfigure}[b]{1.\textwidth}
   \includegraphics[width=\linewidth]{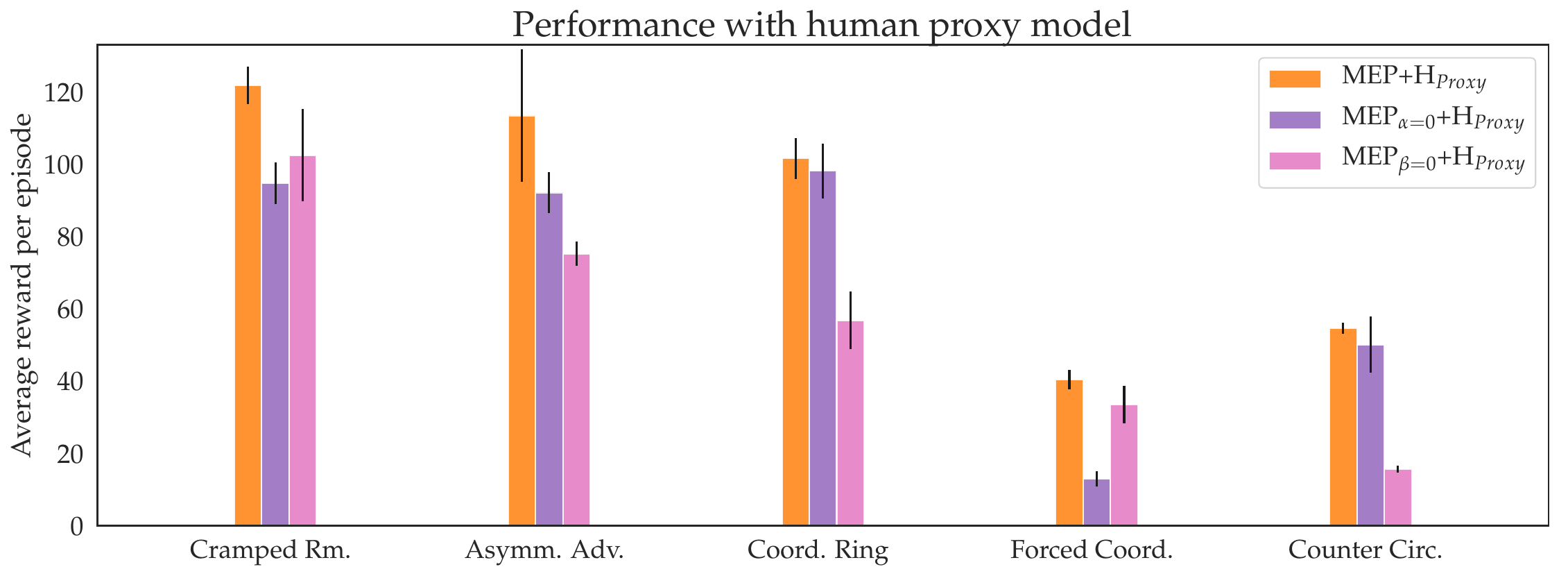}
   \caption{\bo{Ablation tests}}
   \label{fig:MEP_human_proxy_performance_ablation_test}
\end{subfigure}
\caption{\bo{Performance comparison and ablation test}: Average episode rewards over 400 timestep (1 min) trajectories for different methods, with standard error over 5 different random seeds, paired with the proxy human H$_{Proxy}$. 
Figure (a) shows the performance comparison among MEP and other methods including SP, PBT, TrajeDi, FCP, and MPD. 
Figure (b) shows the ablation tests, where we use $\text{MEP}_{\alpha=0}$ to denote MEP without PE reward and use $\text{MEP}_{\beta=0}$ to denote MEP without prioritized sampling. 
For more detailed experimental results, please refer to Figure 7, Figure 8, and Figure 9 in Appendix E.
}
\label{fig:MEP_performance}
\end{figure*}

\begin{question}
How does MEP compare to other methods?
\end{question}
We pair each agent trained with SP, PBT, TrajeDi, FCP, MPD, and MEP, with the human proxy model H$_{Proxy}$, and evaluate the team performance in all the five layouts, as shown in Figure~\ref{fig:overcooked}.
Following the evaluation protocol proposed by Carroll et al.~\cite{carroll2019utility}, we use the cumulative rewards over a horizon of 400 timesteps as the proxy for coordination ability since good coordination between teammates is essential to achieve high scores in the Overcooked game. 
For all the results, we report the average reward per episode and the standard error across five different random seeds. 
As shown in Figure~\ref{fig:MEP_human_proxy_performance}, MEP outperforms other methods in all five layouts when paired with a human proxy model.
Additionally, according to the ablation test shown in Figure~\ref{fig:MEP_human_proxy_performance_ablation_test}, both the population entropy reward and the prioritized sampling are necessary components for achieving the best performance.

We did a hyper-parameter search for TrajeDi on the discounting factor $\gamma$, see Figure 8 in Appendix E, and report the best results, which use $\gamma$ as 1.0 in Figure~\ref{fig:MEP_human_proxy_performance}. 
We also did a hyper-parameter search for FCP on the population size. 
By default, for TrajeDi, MPD, and MEP, we use a population size of 5. However, we use a population size of 10 for FCP in Figure~\ref{fig:MEP_human_proxy_performance}. 
The performance comparison between FCP with population size of 5 and 10 is shown in Figure 9 in Appendix E. 
As the results show, a larger population size indeed improves the performance.
However, MEP outperforms FCP with only half the population size, see Figure~\ref{fig:MEP_human_proxy_performance}.
Similar to FCP~\cite{strouse2021collaborating}, to increase diversity, we use multiple checkpoints, including the beginner model, the middle model, and the best model of each agent in the population to form the final population for training the AI agent. 
For a fair comparison, we use multiple checkpoints for all four methods, i.e., TrajeDi, FCP, MPD, and MEP.

\begin{question}
How does MEP perform with real humans?
\end{question}
\begin{figure*}
  \centering
      \includegraphics[width=1.\linewidth]{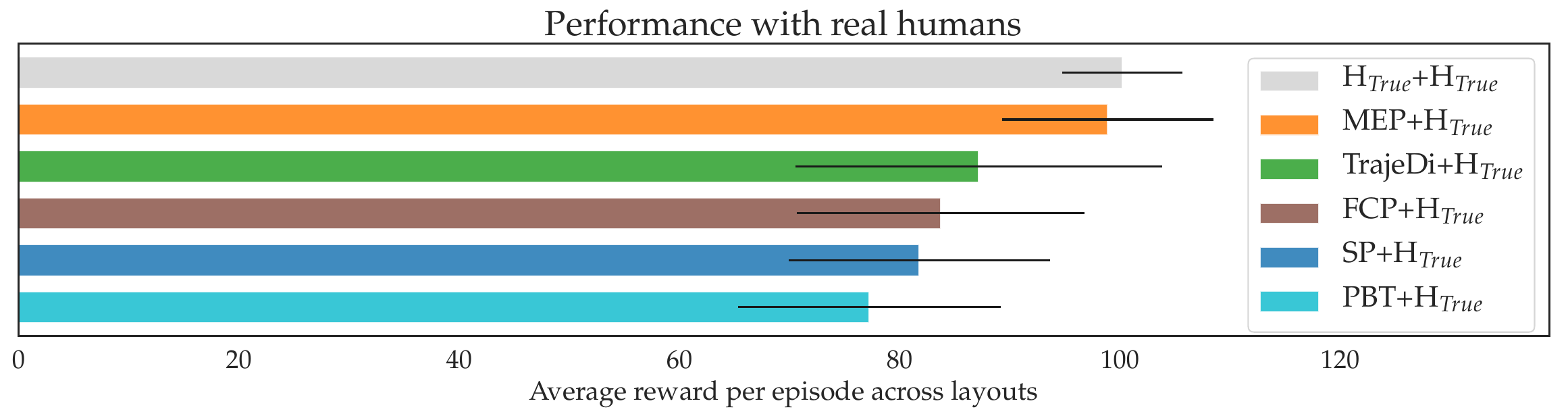}
  \caption{\textbf{Performance with real humans}\label{fig:mep_results}}
  \label{fig:humanai_performance}
\end{figure*}
For this human-AI coordination test, we use Amazon Mechanical Turk (AMT) and follow the same evaluation procedure proposed by Carroll et al.~\cite{carroll2019utility}.
We evaluate TrajeDi, FCP, and MEP-trained AI agents and measure their average episode reward when the agents are paired with a real human player.
We reuse the testing results of SP and PBT from the human-AI evaluation on AMT carried out by Carroll et al.~\cite{carroll2019utility}.
These testing results are compatible because the evaluation procedure is the same and uses a between-subjects design, meaning each user is only paired with a single AI agent.
The results are presented in Figure~\ref{fig:humanai_performance}. 
The chart in Figure~\ref{fig:humanai_performance} shows that, on average, across all five layouts, MEP outperforms other methods, including SP, PBT, FCP, and TrajeDi, and its performance is on par with the Human-Human coordination performance.


\begin{question}
What does AI do when paired with real human players?
\end{question}
Here, we show and analyze some qualitative behaviors observed during the real human-AI coordination experiments, which are shown in the supplementary video from 0:22 to 2:27.
From 0:24 to 0:44, we observe that in the Forced Coordination layout, the MEP-trained agent is more robust and less likely to get stuck during coordination than SP and PBT.
Next, from 0:44 to 1:09, in the Asymmetric Advantage layout, the SP-trained and the PBT-trained agents only learned to put the onion into the pot and did not learn to deliver the onion soup, while the MEP-trained agent learned to put the onion into the pot and learned to deliver the onion soup when its human partner is busy. 
Similarly, from 1:09 to 1:29, in the Cramped Room layout, the SP-trained and PBT-trained agents only learned to use the plate to take the soup, whereas the MEP-trained agent additionally learned to carry the onion to the pot. 
Interestingly, from 1:29 to 1:56, in the Coordination Ring layout, the SP-trained and PBT-trained agent only learned to deliver the onion soup in one direction, while the MEP-trained agent learned to deliver the soup clockwise and counterclockwise, depending on where its human partner stands.
Last but not least, from 2:01 to 2:26, in the Counter Circuit layout, the SP-trained and PBT-trained agents only learned to pass the onion over the ``counter''.
However, the MEP-trained agent also learned to take the plate and deliver the soup.
From all these observations, we conclude that the SP-trained and the PBT-trained agents tend to overfit to the policies they encounter during training, whereas the MEP-trained agent is more robust and flexible in the real world when playing with real human partners.

\section{Related Work}
Over the past few years, deep reinforcement learning has gained many successes in competitive games, such as Go~\cite{AlphaZero}, Dota~\cite{OpenAIFiveFinals}, Quake~\cite{FTW}, and StarCraft~\cite{AlphaStar}. 
In these competitive games, self-play (SP) and population-based training (PBT) have been leveraged to improve performance.
However, the agents trained via SP or PBT tend to learn overly specific policies in collaborative environments~\cite{carroll2019utility}. 
Recent works~\cite{lerer2018learning,tucker2020adversarially,carroll2019utility,knott2021evaluating} tackle the collaboration problem using some behavioral data from the partner to select the equilibrium of the existing agents~\cite{lerer2018learning,tucker2020adversarially} or build and incorporate a human model into the training process~\cite{carroll2019utility,knott2021evaluating}.
However, collecting a large amount of human data in the real world is expensive and onerous. 
In this work, we consider the zero-shot setting, where no behavioral data from the human partner is available during training~\cite{hu2020other}.

From a Bayesian perspective, when we do not have a prior on what the human policies look like, we should train the AI agent to be robust and capable of collaborating with a diverse set of policies~\cite{murphy2012machine}.
One popular approach towards robust AI agents is through maximum entropy reinforcement learning~\cite{ziebart2008maximum,ziebart2010modeling,fox2015taming,haarnoja2017reinforcement,haarnoja2018soft}, and many previous works leverage it as a means of encouraging exploration~\cite{schulman2017equivalence,haarnoja2018soft} or skill discovering~\cite{eysenbach2018diversity,zhao2021mutual}.
However, obtaining a diversified population through entropy maximization is still subjective to research.
In Multi-agent Reinforcement Learning (MARL), a group of agents is trained to achieve a common goal by Centralized Training and Decentralized Execution (CTDE)~\cite{MADDPG,foerster2018counterfactual}.
Taking inspiration from CTDE, we propose to train a population of agents to maximize a centralized surrogate objective -- population entropy, to encourage diversity in the population.
Subsequently, we train the AI agent with the maximum entropy population and dynamically sample the partner agent based on the learning progress, which shares similarities with Prioritized Fictitious Self-Play (PFSP)~\cite{vinyals2019grandmaster}. PFSP is designed exclusively for zero-sum competitive games, whereas we are concerned with cooperative games and derive the relationship between prioritized sampling and cooperation performance lower bound, see Appendix B.
With prioritized sampling, the AI agent can learn a policy that is generally suitable for all the strategies presented in the population.

The idea of MEP shares a common intuition with domain randomization, where some features of the environment are changed randomly during training to make the policy robust to that feature~\cite{tobin2017domain,yu2017preparing,peng2018sim,tan2018sim,openai2019rubiks,tang2020discovering}.
In general, MEP can be seen as a domain randomization technique, where the randomization is conducted over a set of partners' policies.

A recent related work -- TrajeDi~\cite{lupu2021trajectory} has a similar motivation and encourages the trajectories from different agents in the population to be distinct.
TrajeDi directly optimizes the trajectory-level Jensen-Shannon divergence between policies as part of the policy loss, while our method trains the population with reward function augmented by population entropy on the action level. 
However, the variance of the evaluated gradient from TrajeDi could be unbounded due to its trajectory-level importance sampling part, while our formulation does not have importance sampling terms. 
In the experiments, MEP shows superior performance compared to TrajeDi empirically.

Another recent work -- FCP~\cite{strouse2021collaborating} is closely related to our work. Strouse et al.~\cite{strouse2021collaborating} show that with diversity induced by different checkpoints and different random seeds, the agent can generalize well in collaborative games. However, in our experiments, we find out that FCP requires a relatively large population to work well. Compared to FCP, MEP only uses half the population and works better. The experimental results show that MEP is more effective in obtaining a diverse population than FCP, given the same amount of computational resources.

There are also other population diversity-based methods, such as Diversity via Determinants (DvD)~\cite{parker2020effective} and Diversity-Inducing Policy Gradient (DIPG)~\cite{masood2019diversity}. 
DvD is based on the determinant of the kernel matrix, and DIPG is derived from Maximum Mean Discrepancy (MMD).
These two methods are formulated for the single-agent setting, whereas MEP is designed for the multi-agent cooperative setting.
In games with non-transitive dynamics where strategic cycles exist, e.g., Rock-Paper-Scissors, Policy-Space Response Oracle (PSRO)-based methods~\cite{balduzzi2019open,perez2021modelling,liu2021unifying} provide solutions to learn diverse behaviors.
In general, MEP is complementary to these previous works and is applicable to many human-AI coordination tasks.

\section{Conclusion}
This paper introduces Maximum Entropy Population-based training (MEP), a deep reinforcement learning method for robust human-AI coordination.
With the derived population entropy reward encouraging diversity in policies and the learning progress-based prioritized sampling enhancing generalization to unencountered policies, the MEP-trained agents demonstrate more flexibility and robustness to various human strategies. Our result, which bridges maximum entropy RL and PBT, suggests that entropy maximization can be a promising avenue for achieving diversity and robustness in reinforcement learning. Combing MEP with other MARL algorithms could be meaningful directions for future work.

\bibliography{reference}
\bibliographystyle{abbrv}


\appendix

\section*{Appendix}

\section{Population Entropy Lower Bound}
\label{app:lower_bound}
\begin{theorem}
Let the population diversity be defined as:
\begin{align}
\PD(\set{\pi^{(1)}&, \pi^{(2)}, ..., \pi^{(n)}}, s_t) \coloneqq
\frac{1}{n}\sum_{i=1}^n \ent(\pi^{(i)}(\voidarg | s_t)) + \frac{1}{n^2}\sum_{i =1}^n \sum_{j=1}^n D_\text{KL} (\pi^{(i)}(\voidarg | s_t), \pi^{(j)}(\voidarg | s_t))  
\end{align}

Let the population entropy be defined as:
\begin{align}
\PE(\set{\pi^{(1)}, \pi^{(2)}, ..., \pi^{(n)}}, s_t) \coloneqq \ent(\bar{\pi}(\voidarg | s_t)) ,\ \text{where}\ \bar{\pi}(a_t | s_t) \coloneqq \frac{1}{n} \sum_{i=1}^n \pi^{(i)}(a_t|s_t) .
\end{align}
Then, we have 
\begin{align}
\PD(\set{\pi^{(1)}, \pi^{(2)}, ..., \pi^{(n)}}, s_t) \geq \PE(\set{\pi^{(1)}, \pi^{(2)}, ..., \pi^{(n)}}, s_t) ,
\end{align}
where $n$ is the population size.
\end{theorem}

\begin{proof}
KL-divergence is given by: 
\begin{align}
\KL(\pi^{(i)}(\voidarg | s_t), \pi^{(j)}(\voidarg | s_t))= \sum_{a\in\mathcal{A}} \pi^{(i)}(a_t|s_t) \log \frac{\pi^{(i)}(a_t|s_t)}{ \pi^{(j)}(a_t|s_t)} . 
\end{align}
Entropy is given by:
\begin{align}
\ent(\pi^{(i)}(\voidarg|s_t)) = - \sum_{a\in\mathcal{A}} \pi^{(i)}(a_t|s_t) \log \pi^{(i)}(a_t|s_t) .
\end{align}
We can derive the follows:
\begin{align}
&\PD(\set{\pi^{(1)}, \pi^{(2)}, ..., \pi^{(n)}}, s_t) \\
= &\frac{1}{n^2}\sum_{i =1}^n \sum_{j=1}^n D_\text{KL} (\pi^{(i)}(\voidarg | s_t), \pi^{(j)}(\voidarg | s_t)) + \frac{1}{n}\sum_{i=1}^n \ent(\pi^{(i)}(\voidarg | s_t)) 
\\
=&\frac{1}{n^2}\sum_{i =1}^n \sum_{j=1}^n \CE (\pi^{(i)}(\voidarg | s_t), \pi^{(j)}(\voidarg | s_t)) - \frac{1}{n^2}\sum_{i=1}^n \sum_{j=1}^n \ent(\pi^{(i)}(\voidarg | s_t)) + \frac{1}{n}\sum_{i=1}^n \ent(\pi^{(i)}(\voidarg | s_t))
\\
=&\frac{1}{n^2}\sum_{i =1}^n \sum_{j=1}^n \CE (\pi^{(i)}(\voidarg | s_t), \pi^{(j)}(\voidarg | s_t)) - \frac{1}{n}\sum_{i=1}^n \ent(\pi^{(i)}(\voidarg | s_t) + \frac{1}{n}\sum_{i=1}^n \ent(\pi^{(i)}(\voidarg | s_t))  
\\
=&\sum_{a\in\mathcal{A}}  \frac{1}{n^2} (\sum_{i, j}-\pi^{(i)}(a_t | s_t) \log \pi^{(j)}(a_t | s_t)) \\
\label{eq:jensen_1}
=&\sum_{a\in\mathcal{A}}\sum_i - \frac{1}{n} \pi^{(i)}(a_t|s_t) \sum_j \frac{1}{n} \log \pi^{(j)}(a_t|s_t) \\
\label{eq:jensen_2}
\geq& \sum_{a\in\mathcal{A}} \sum_i -\frac{1}{n} \pi^{(i)}(a_t|s_t) \log \sum_j \frac{1}{n} \pi^{(j)}(a_t|s_t) \\
=&\sum_{a\in\mathcal{A}} - \bar{\pi}(a_t|s_t) \log \bar{\pi}(a_t|s_t) \\
=&\ent(\bar{\pi}(\voidarg|s_t)) \\
=&\PE(\set{\pi^{(1)}, \pi^{(2)}, ..., \pi^{(n)}}, s_t)
\end{align}

From Equation~(\ref{eq:jensen_1}) to Equation~(\ref{eq:jensen_2}), we use Jensen's inequality~\cite{murphy2012machine}.
\end{proof}

\section{Prioritized Sampling and Performance Lower Bound}
\label{app:ps_lower_bound}
Here, we use $\pi^{(A)}$ to denote the AI policy and use $\theta$ to represent the parameter of $\pi^{(A)}$.
At the training step $t$, we use $\theta_t$ to denote the current parameter.
$\{\pi^{(1)},...,\pi^{(n)}\}$ is the population used to train $\pi^{(A)}$. 
We want to find the optimal parameter $\theta^{*}$ for the AI policy $\pi^{(A)}$, so that the AI agent could cooperate well with any agent in the population.

At each training step, we first sample $\pi^{(i)}$ from the population, then let  $\pi^{(i)}$ to cooperate with $\pi^{(A)}$. 
Then, we use the sampled trajectory $\tau$ to train $\pi^{(A)}$.
$J(\pi^{(A)}, \pi^{(i)})$ is the excepted sum rewards achieved by $\pi^{(A)}$ and $\pi^{(i)}$ together. 
With prioritized sampling introduced in Section~\ref{sec:prioritized_sampling}, we assign higher priority to the agent that is harder to collaborate with.
To be more specific, let
\begin{align}
i=\mathop{\arg \min}_{i} J(\pi^{(A)}_{\theta_t}, \pi^{(i)}) ,
\end{align}
then we sample $\pi^{(i)}$ to cooperate with $\pi^{(A)}_{\theta_t}$. 
During training, $\theta$ is updated using the gradient ascent method with non-increasing learning rate $\alpha_t$. 
At each training step $t$,  there exists a subset of $\{1,...,n\}$ denoted by $K_{\theta_t}=\{i_k| k={1,...,l, l\leq n}\}$, such that
\begin{align}
J(\pi^{(A)}_{\theta_t}, \pi^{(i)})> J(\pi^{(A)}_{\theta_t}, \pi^{(i_k)})=C_t,\quad \textrm{if} \, i\notin K_{\theta_t}, i_k\in K_{\theta_t} .
\end{align}
Since prioritized sampling is used, one of $i_{k'}$ ($k'\in{1,...,l}$) could be sampled. 
Then the gradient of the current step $t$ is 
\begin{align}
\nabla_{\theta_t} J(\pi^{(A)}_{\theta_t}, \pi^{(i_{k'})})
\end{align}
The parameter is updated as following:
\begin{align}
\theta_{t+1} = \theta_{t} + \alpha_t \nabla\theta_t ,
\end{align}
where $\alpha_t$ is the learning rate at the training time step $t$.

Assume that $\{J(\pi^{(A)}_{\theta_t}, \pi^{(i)})|i=1,...,n\}$ are smooth towards $\theta_t$, then 
\begin{align}
g(\theta) = \min_{i\in\{1,...,n\}}J(\pi^{(A)}_{\theta}, \pi^{(i)})
\end{align}
is a piece-wise smooth function and $\nabla_{\theta_t} J(\pi^{(A)}_{\theta_t}, \pi^{(i_{k'})})$ is equal to the gradient of $g(\theta)$ almost everywhere. 
Next we prove that using $\nabla_{\theta_t} J(\pi^{(A)}_{\theta_t}, \pi^{(i_{k'})})$, it could also converge to a local maximum of $g(\theta)$.

\begin{lemma}
\label{lem:ps_lower_bound}
$\pi^{(A)}$ with parameter $\theta$ is trained with the population $\{\pi^{(1)},...,\pi^{(n)}\}$. 
We use the learning progress-based prioritized sampling to sample the agent from the population for training. 
Assume that $J(\pi^{(A)}_{\theta}, \pi^{(i)})$ is smooth towards the parameter vector $\theta$ and has an L-Lipschitz gradient for all $i$. $\theta$ is optimized using the gradient ascent with a sufficiently small constant step size. If $J(\pi^{(A)}_{\theta}, \pi^{(i)})$ converges and doesn't go to infinity, it would converge to a local maximum of $g(\theta)$. That is  $\theta$ converges to a neighborhood $V_{\hat{\theta}}$ of $\hat{\theta}$, where $\hat{\theta}$ is define as
\begin{align}
\hat{\theta} = \mathop{\arg\max}_{\theta \in V_{\hat{\theta}}} \min_{i\in \{1,...,n\}}J(\pi^{(A)}_{\theta}, \pi^{(i)}) .
\end{align}
\end{lemma}
\begin{proof} 

$\theta_t$ denotes the parameter of $\pi^{(A)}$ at the training step $t$. 
We define the index set of $i$, that $J(\pi^{(A)}_{\theta_t}, \pi^{(i)})$ equal to $\min_{i\in \{1,...,n\}}J(\pi^{(A)}, \pi^{(i)})$:
\begin{align}
K_{\theta_t}=\{i_k| k={1,...,l, l\leq n}\} ,
\end{align}
where $i_k$ satisfies 
\begin{align}
J(\pi^{(A)}_{\theta_t}, \pi^{(i)})> J(\pi^{(A)}_{\theta_t}, \pi^{(i_k)})=C_t,\quad \textrm{if} \, i\notin K_{\theta_t}, i_k\in K_{\theta_t} .
\end{align}
$i_{k'}$ is sampled from $K_{\theta_t}$ and current gradient is $\nabla_{\theta_t} J(\pi^{(A)}_{\theta_t}, \pi^{(i_{k'})})$. 

If $i_{k'}\in K_{\theta_{t+l}}$ for all $l>0$, the optimization process could be regarded as a non-convex optimization problem by gradient ascent, then $\theta_t$ converges to a local maximum almost surely by the assumption $J(\pi^{(A)}_{\theta_t}, \pi^{(i)})$ has an L-Lipschitz gradient~\cite{lee2016gradient}. 

If $\bigcap_{l>0}K_{\theta_{t+l}}=\emptyset$, then first we prove that the sequence $\{\theta_t\}$ can't converge to a saddle point. If $\hat{\theta}$ is a saddle point, $\nabla_{\theta} J(\pi^{(A)}_{\hat{\theta}}, \pi^{(i_{k'})}) = 0$ for all $i_{k'} \in K_{\hat{\theta}}$, which means $J(\pi^{(A)}_{\theta_t}, \pi^{(i_{k'})})$ are identical in a neighborhood of $\hat{\theta}$. This contradicts the local minimum convergence of $\theta_t$~\cite{lee2016gradient}. 

Assume the convergent point $\hat{\theta}$ has a non-zero gradient $\nabla_{\theta} J(\pi^{(A)}_{\theta_t}, \pi^{(i_{k'})})$, since we use a sufficiently small constant learning rate $\alpha$, if $J(\pi^{(A)}_{\hat{\theta} + \Delta\theta}, \pi^{(i_{k'})}) > J(\pi^{(A)}_{\hat{\theta}}, \pi^{(i_{k'})})$, this would contradict the convergence of $\theta$, and if $J(\pi^{(A)}_{\hat{\theta} + \Delta\theta}, \pi^{(i_{k'})}) \leq J(\pi^{(A)}_{\hat{\theta}}, \pi^{(i_{k'})})$, which means $\hat{\theta}$ is the local maximum.

\end{proof}
From Lemma~\ref{lem:ps_lower_bound}, we can see that with prioritized sampling, we could improve the lower bound of the cooperation performance between $\pi^{(A)}$ and the population $\{\pi^{(1)},...,\pi^{(n)}\}$. 
In comparison, uniform sampling does not provide any guarantee on the worst case. We call $\theta'$ the optimal solution of mean sampling:
\begin{align}
\theta' = \argmax_{\theta} \sum_{i\in \{1,...,n\}} J(\pi^{(A)}_\theta, \pi^{(i)}) .
\end{align}
Then, the worst cooperation between $\pi_{\theta'}$ and the population must be no greater than the cooperation between $\pi_{\hat{\theta}}$ and the population. That is:
\begin{align}
\min_{i\in \{1,...,n\}}J(\pi^{(A)}_{\hat{\theta}}, \pi^{(i)}) \geq \min_{i\in \{1,...,n\}}J(\pi^{(A)}_{\theta'}, \pi^{(i)}) .
\end{align}

\section{Relation to Human-AI Coordination Performance}
\label{app:epsilon_close_and_reward}
To illustrate that we could improve the lower bound of human-AI coordination performance, here we introduce the connection between $\epsilon$-close~\cite{ko2006mathematical} and return, i.e., expected sum rewards.

\begin{definition}
We define that $\pi^{(1)}$ is $\epsilon$-close to $\pi^{(2)}$ at the state $s_{t}$ if
\begin{align}
\left| \frac{\pi^{(1)}(a_t|s_t)}{\pi^{(2)}(a_t|s_t)} - 1 \right| < \epsilon
\end{align}
for all $a_t\in \mathcal{A}$. 
If this is satisfied at every $s_t\in \mathcal{S}$,  we call $\pi^{(1)}$ is $\epsilon$-close to $\pi^{(2)}$. 
\end{definition}

\begin{lemma}
\label{lem:epsilon_close_and_reward}
If an MDP has $T$ time steps and $\pi^{(1)}$ is $\epsilon$-close to $\pi^{(2)}$, then for all $\pi^{(A)}$,  we have 
\begin{align}
(1-\epsilon)^T J(\pi^{(2)}, \pi^{(A)}) < J(\pi^{(1)}, \pi^{(A)}) < (1+\epsilon)^T J(\pi^{(2)}, \pi^{(A)}) ,
\end{align}
where $J(\pi^{(i)}, \pi^{(A)})$ denotes the expected sum reward achieved by $\pi^{(i)}$ and $\pi^{(A)}$ collaborating with each other.
\end{lemma}
\begin{proof}
For all trajectory $\tau$, we have
\begin{align}
p_{\pi^{(1)}, \pi^{(A)}}(\tau) = p(s_0)\Pi_{t=0}^{T-1}\pi^{(A)}(a_t|s_t)\pi^{(1)}(a_t|s_t)p(s_{t+1}|s_t, a_t^{(A)}, a_t^{(1)}) .
\end{align}
Since 
\begin{align}
(1-\epsilon)\pi^{(2)}(a_t|s_t)< \pi^{(1)} (a_t|s_t)< (1+\epsilon) \pi^{(2)}(a_t|s_t) ,
\end{align}
we have
\begin{align}
(1-\epsilon)^T p_{\pi^{(2)}, \pi^{(A)}}(\tau) < p_{\pi^{(1)}, \pi^{(A)}}(\tau)  < (1+\epsilon)^T p_{\pi^{(2)}, \pi^{(A)}}(\tau) .
\end{align}
And because
\begin{align}
J(\pi^{(i)}, \pi^{(A)}) = \sum_{\tau}p_{\pi^{(1)}, \pi^{(A)}}(\tau) r(\tau) ,
\end{align}
where 
\begin{align}
r(\tau)=r(s_0,a_0,...,s_{T-1}) = \sum_t r(s_t,a_t)\ \text{and}\ a_t=(a^{(i)}, a^{(A)}) .
\end{align}
Therefore, we have 
\begin{align}
(1-\epsilon)^T J(\pi^{(2)}, \pi^{(A)})<J(\pi^{(1)}, \pi^{(A)})<(1+\epsilon)^T J(\pi^{(2)}, \pi^{(A)}) .
\end{align}
\end{proof}

We use $\pi^{(H)}$ to denote the human player's policy. 
From Lemma~\ref{lem:epsilon_close_and_reward}, we can see that if $\pi^{(H)}$ is similar to any of the policy $\pi^{(i)}$ in the population $\{\pi^{(1)},...,\pi^{(n)}\}$ in a certain degree, measured by $\epsilon$-close, then its cooperation performance with the AI policy $\pi^{(A)}$, which is trained with the population, would not deteriorate too much. 
Furthermore, we derive the following corollary.

\begin{corollary}
\label{cor:hai_performance}
We call the infimum of expected sum rewards of $\pi^{(A)}$ cooperating with the population $\{\pi^{(1)},...,\pi^{(n)}\}$ as:
\begin{align}
\min_{i\in \{1,...,n\}}J(\pi^{(A)}, \pi^{(i)})=C .
\end{align}  
If $\pi^{(H)}$ is $\epsilon$-close to the policy $\pi^{(i)}$ in the population, then we have
\begin{align}
J(\pi^{(A)}, \pi^{(H)}) > C (1-\epsilon) ^ T ,
\end{align}
where $T$ is the total steps in the trajectory.
\end{corollary}
\begin{proof}
If $\pi^{(H)}$ is $\epsilon$-close to $\pi^{(i)}$, based on the property of $\epsilon$-close, see Lemma~\ref{lem:epsilon_close_and_reward}, 
we have
\begin{align}
J(\pi^{(A)}, \pi^{(H)}) >  (1-\epsilon)^T J(\pi^{(A)}, \pi^{(i)}) .
\end{align}
Additionally, since 
\begin{align}
J(\pi^{(A)}, \pi^{(i)}) > C ,
\end{align}
we have 
\begin{align}
J(\pi^{(A)}, \pi^{(H)}) > C (1-\epsilon)^T.
\end{align}
\end{proof}
Since prioritized sampling optimizes the lower bound of expected sum rewards of the AI agent cooperating with the population, see Lemma~\ref{lem:ps_lower_bound}, and with Corollary~\ref{cor:hai_performance}, we could say that it also optimizes the lower bound of expected sum rewards of the AI agent cooperating with the Human player, when the population is diverse and representative enough so that it is close to cover human behaviors.

\section{Experiment Details}
\label{app:experimental_details}
We ran all the methods in each environment with 5 different random seeds and report the average episode reward and the standard deviation. 
The experiments of the maximum entropy population-based training use the following hyper-parameters:
\begin{itemize}
\item The learning rate is 8e-4.
\item The reward shaping horizon is 5e6.
\item The environment steps per agent is 1.1e7.
\item The number of mini-batches is 10.
\item The mini-batch size is 2000.
\item PPO iteration timesteps are 40000. The PPO iteration timesteps refer to the length in environment timesteps for each agent pairing training.
\item The population size is 5 for training the maximum entropy population. We use the beginner model, the middle model, and the best model of each agent in the population to form the final population for training the AI agent.
\item The weight $\alpha$ for the population entropy reward is 0.01 in general. For the Forced Coordination layout, we use 0.04.
\item The number of parallel environments used for simulating rollouts is 50.
\item The discounting factor $\gamma$ is 0.99.
\item The max gradient norm is 0.1.
\item The PPO clipping factor is 0.05.
\item The number of hidden layers is 3.
\item The size of hidden layers is 64.
\item The number of convolution layers is 3.
\item The number of filters is 25.
\item The value function coefficient 0.1.
\item The $\beta$ for prioritized sampling is 3.
\end{itemize}

\section{Experimental Results}
\label{app:experimental_results}

\begin{figure}[h]
\centering
\includegraphics[height=.27\linewidth]{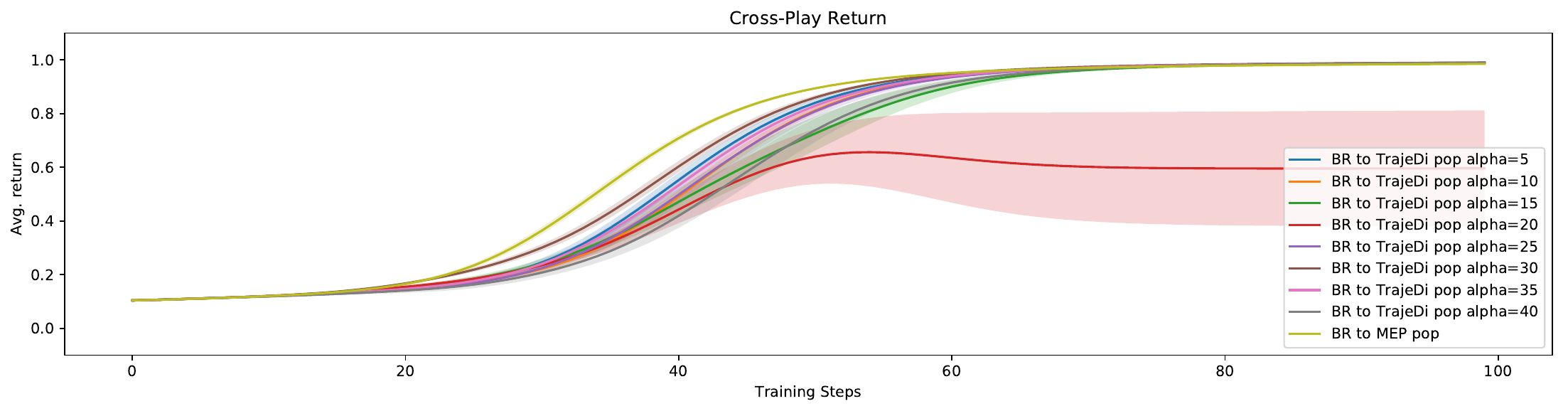}
\caption{\bo{Performance comparison}: We did an extensive hyper-parameter search for TrajeDi. MEP converges faster than TrajeDi under all the parameters.}
\label{fig:matrix_game_hyperparameter}
\end{figure}

\begin{figure}[!h]
\centering
\begin{subfigure}[b]{1.\textwidth}
   \includegraphics[width=\linewidth]{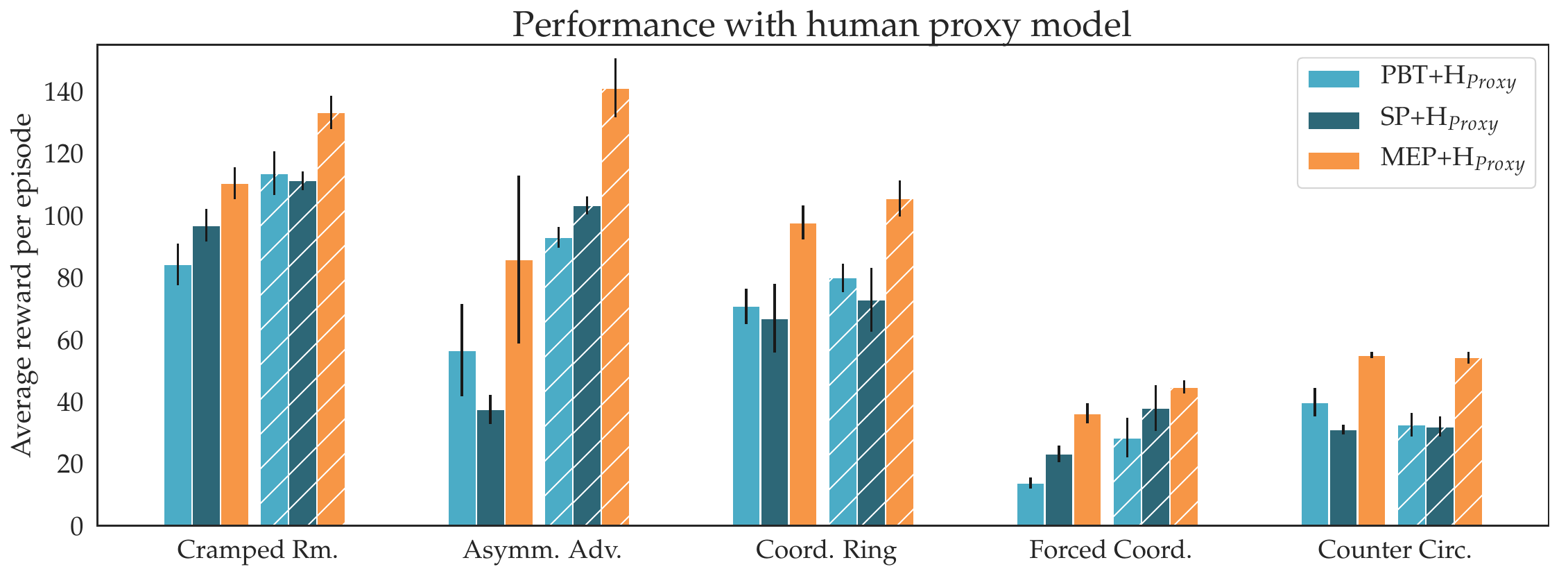}
   \caption{\bo{Performance Comparison}}
   \label{fig:MEP_COMB_2_exp_plot} 
\end{subfigure}
\begin{subfigure}[b]{1.\textwidth}
   \includegraphics[width=\linewidth]{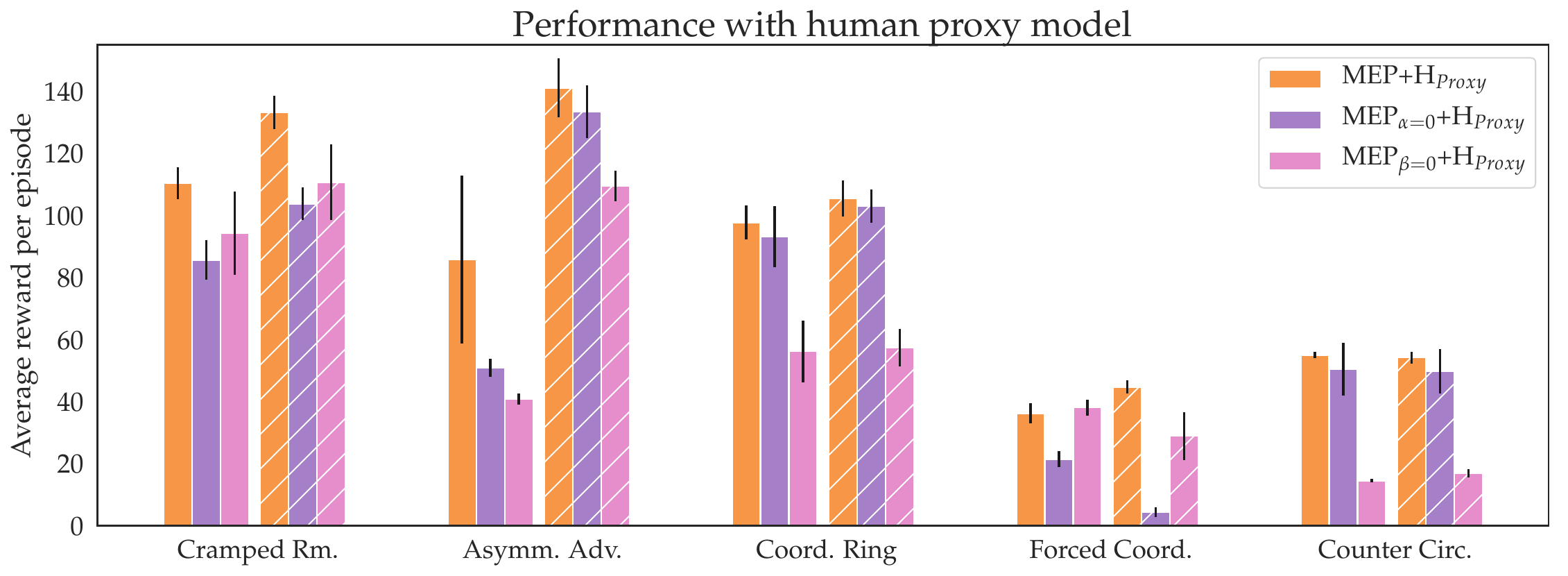}
   \caption{\bo{Ablation tests}}
   \label{fig:MEP_Ablation_COMB_3_exp_plot}
\end{subfigure}
\begin{subfigure}[b]{1.\textwidth}
   \includegraphics[width=\linewidth]{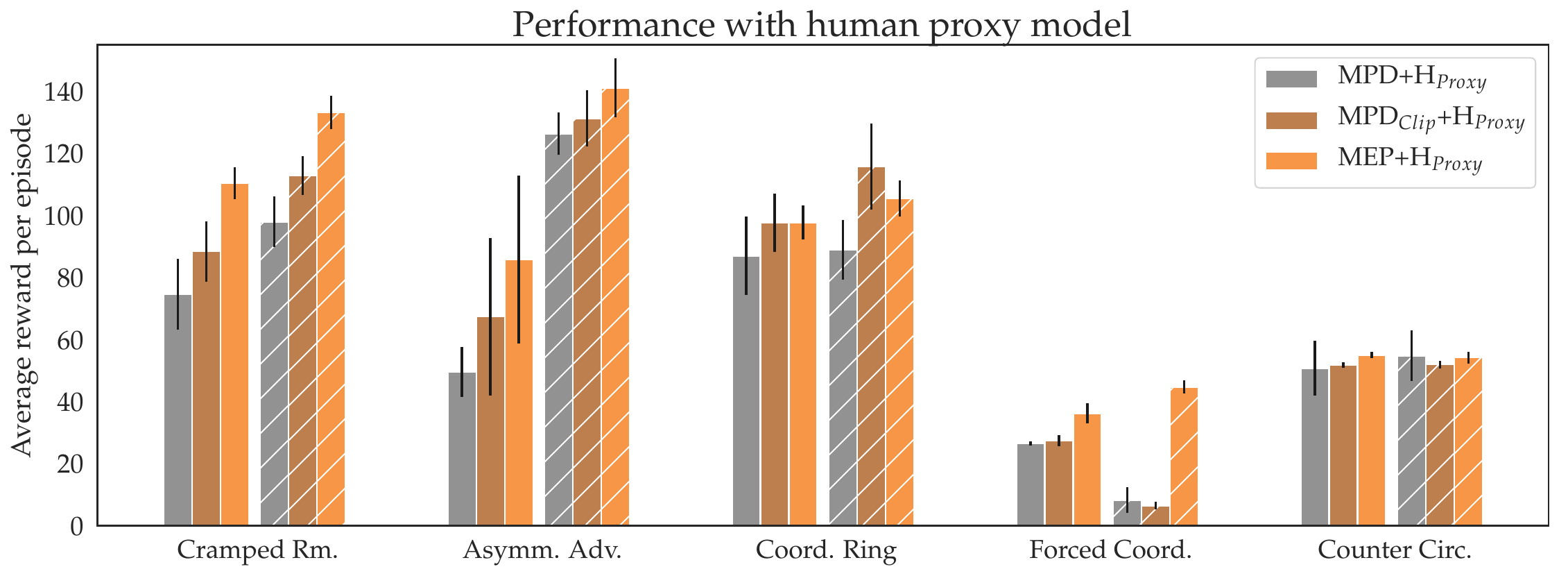}
   \caption{\bo{Performance Comparison}}
   \label{fig:MPD_COMB_12_exp_plot}
\end{subfigure}
\caption{\bo{Performance comparison and ablation tests}: Average episode rewards over 400 timestep (1 min) trajectories for different methods, with standard error over 5 different random seeds, paired with the proxy human H$_{Proxy}$. 
The hashed bars with the slash (/) show results with the starting position of the agents switched.
Figure (a) shows the performance comparison among MEP, SP and PBT. 
Figure (b) shows the ablation tests, where we use $\text{MEP}_{\alpha=0}$ and $\text{MEP}_{\beta=0}$ to denote the MEP model without the population entropy reward and without the prioritized sampling mechanism, respectively. 
Figure (c) shows the performance comparison with Maximum Population Diversity (MPD) with or without clipping the importance weights that are greater than 1.
}
\label{fig:MEP_MEPAblation_MPD_experiments}
\end{figure}

\begin{figure}[!h]
\centering
\includegraphics[width=1.\linewidth]{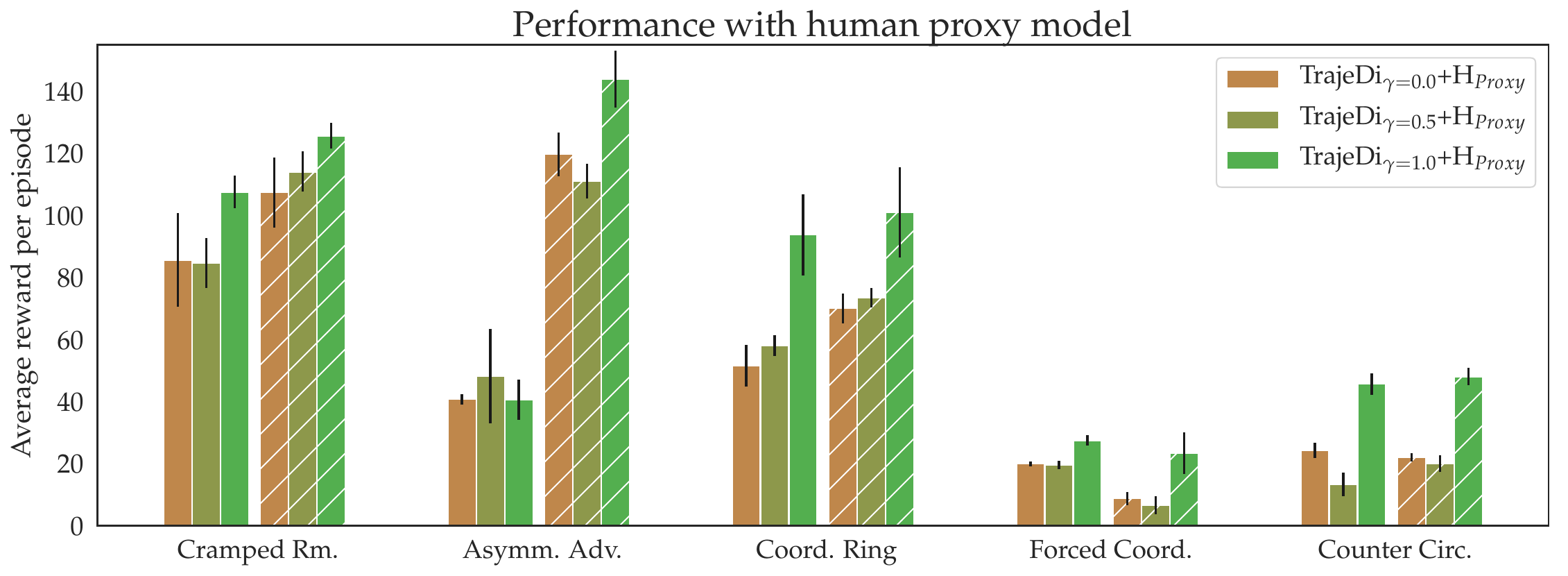}
\caption{\bo{TrajeDi performance comparison}: This figure shows the performance comparison of TrajeDi with different discounting factors $\gamma$, i.e., 0.0, 0.5, 1.0. Average episode rewards over 400 timestep (1 min) trajectories for different methods, with standard error over 5 different random seeds, paired with the proxy human H$_{Proxy}$. 
The hashed bars with the slash (/) show results with the starting position of the agents switched.
}
\label{fig:TrajeDi_experiments}
\end{figure}

\begin{figure}[!h]
\centering
\includegraphics[width=1.\linewidth]{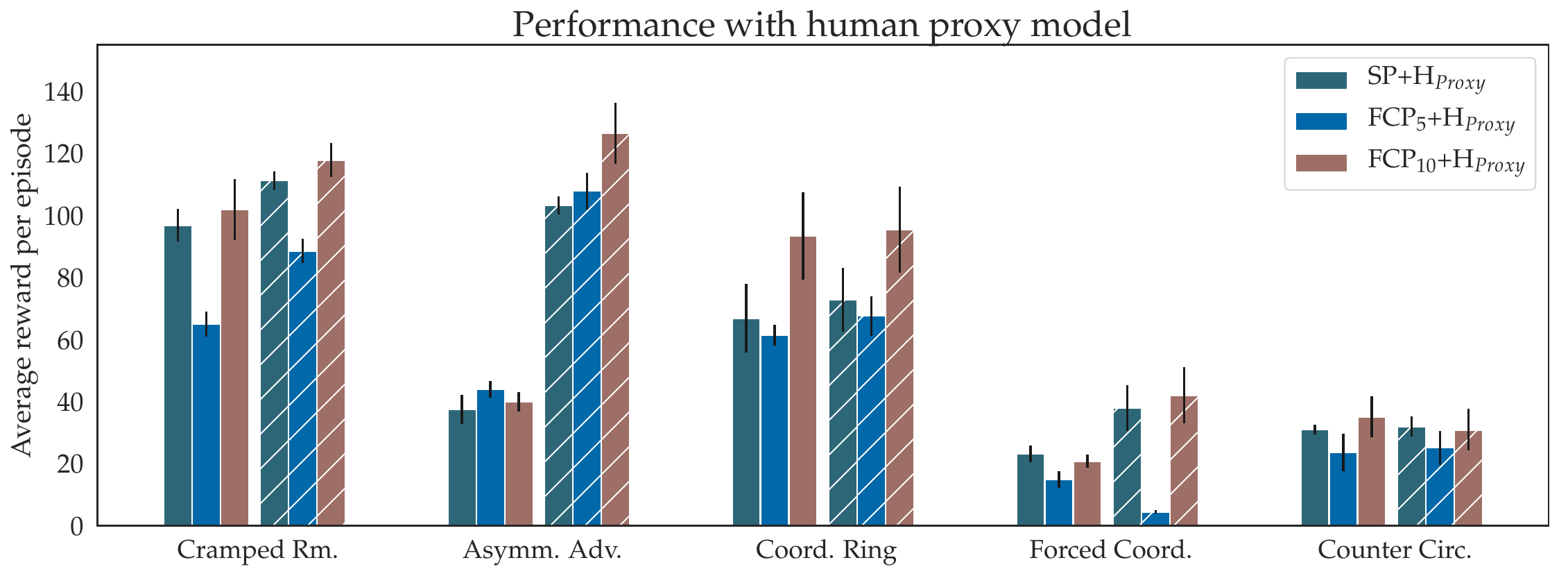}
\caption{\bo{FCP performance comparison}: This figure shows the performance comparison of FCP with different population sizes, i.e., 5 or 10. Average episode rewards over 400 timestep (1 min) trajectories for different methods, with standard error over 5 different random seeds, paired with the proxy human H$_{Proxy}$. 
The hashed bars with the slash (/) show results with the starting position of the agents switched.
}
\label{fig:FCP_experiments}
\end{figure}

\begin{table}[h]
  \centering
  \begin{tabular}{M{0.218\linewidth}M{0.218\linewidth}M{0.218\linewidth}M{0.218\linewidth}}
     \toprule
      $\alpha$=0.000 & $\alpha$=0.001 & $\alpha$=0.005 & $\alpha$=0.010 \\
      \midrule
      \includegraphics[width=\linewidth]{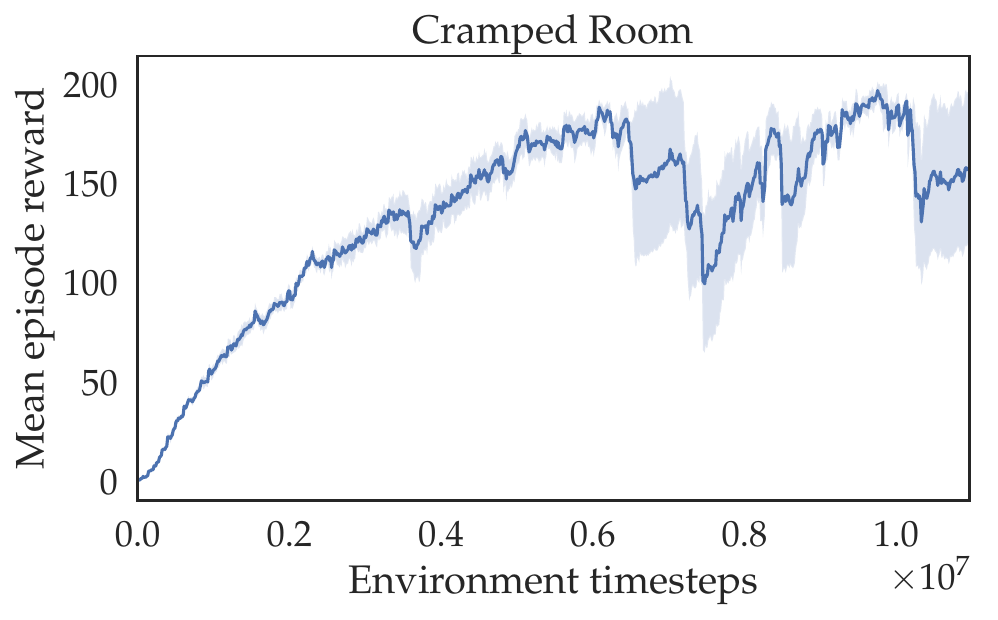} & \includegraphics[width=\linewidth]{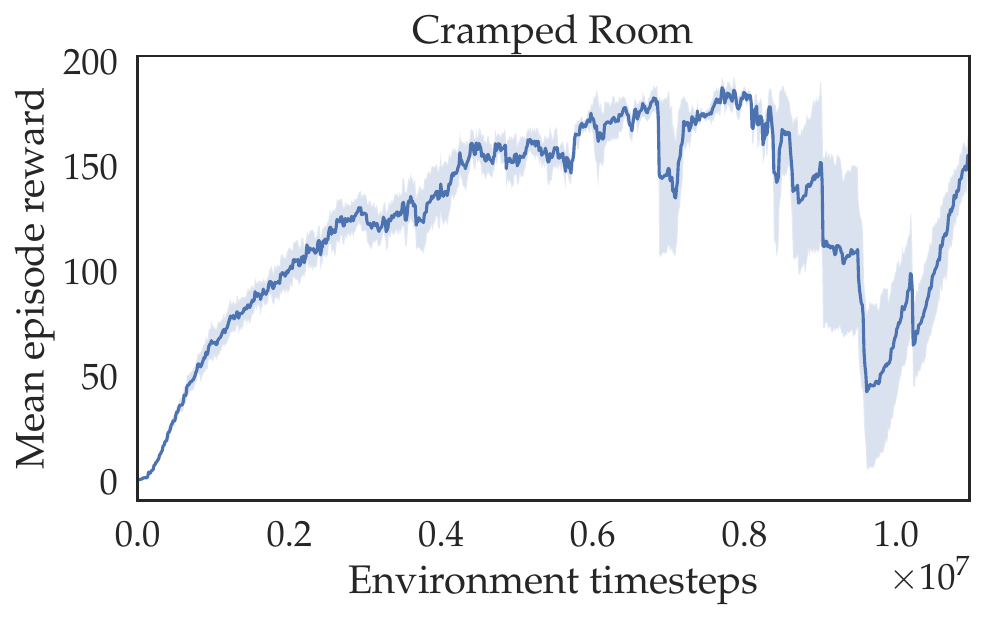} & \includegraphics[width=\linewidth]{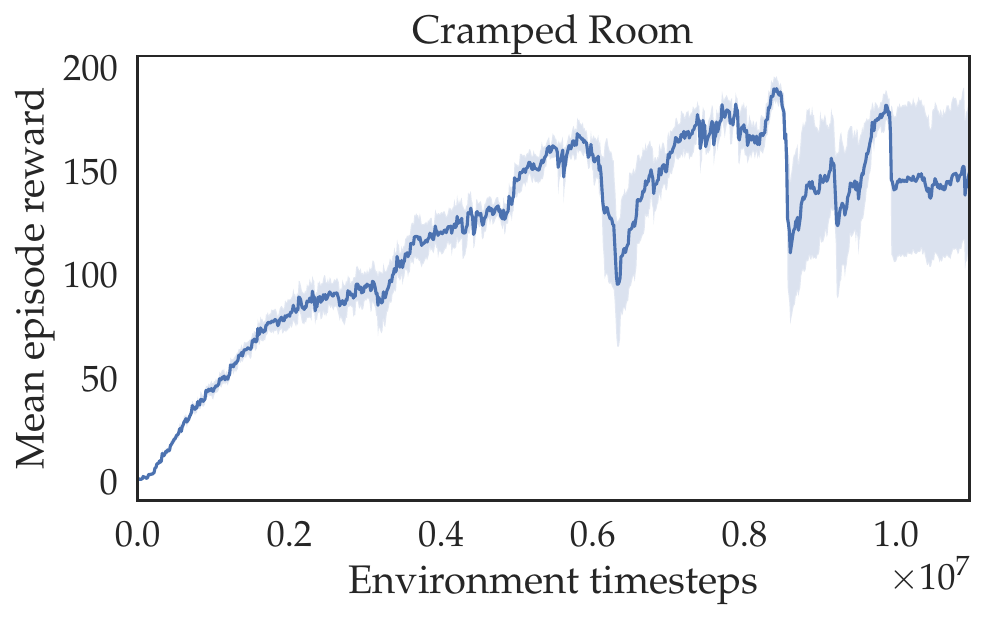} & \includegraphics[width=\linewidth]{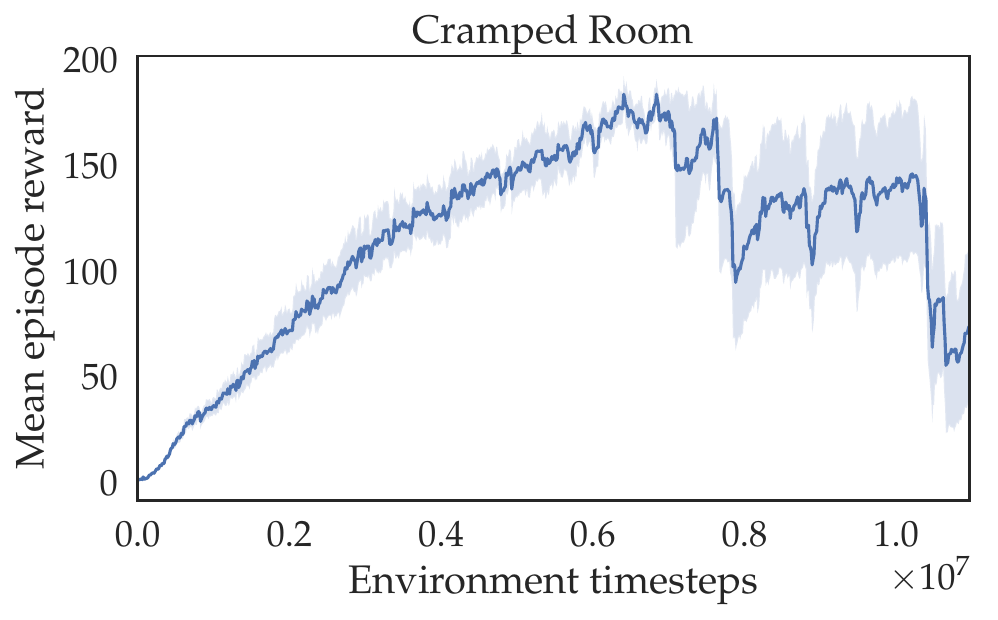} \\
      \includegraphics[width=\linewidth]{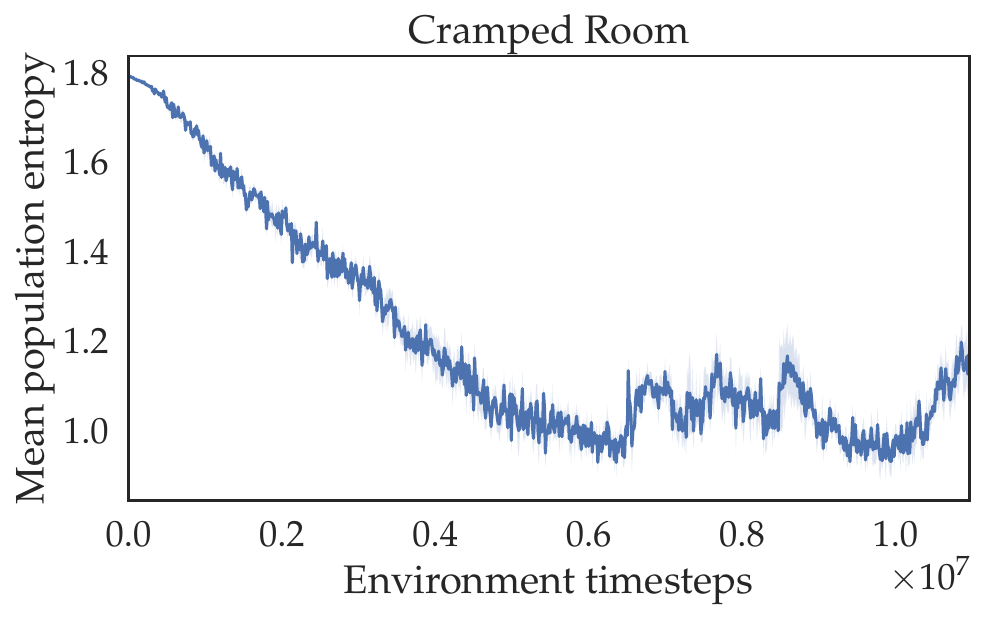} & \includegraphics[width=\linewidth]{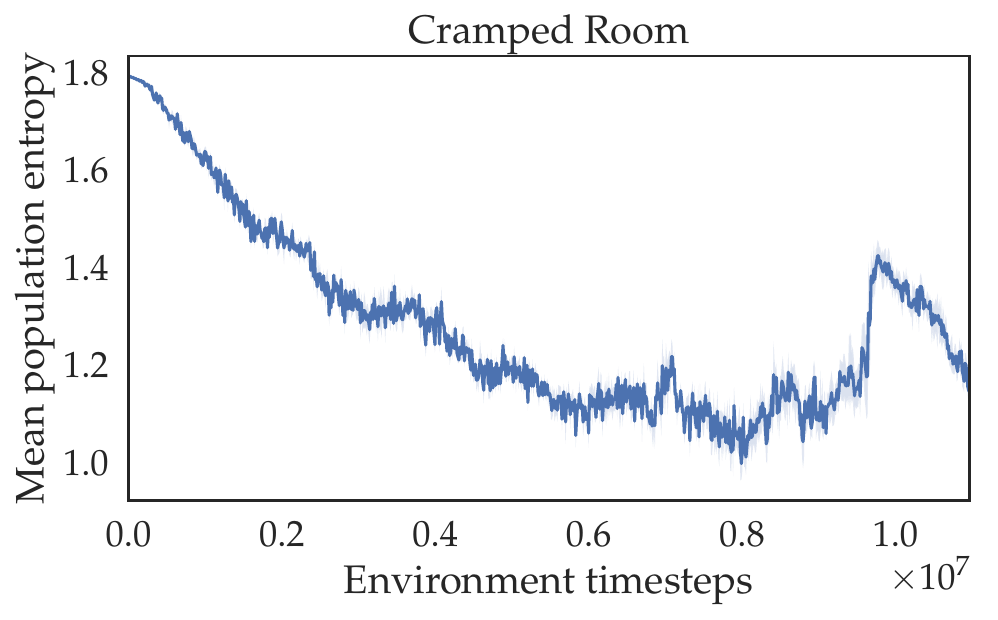} & \includegraphics[width=\linewidth]{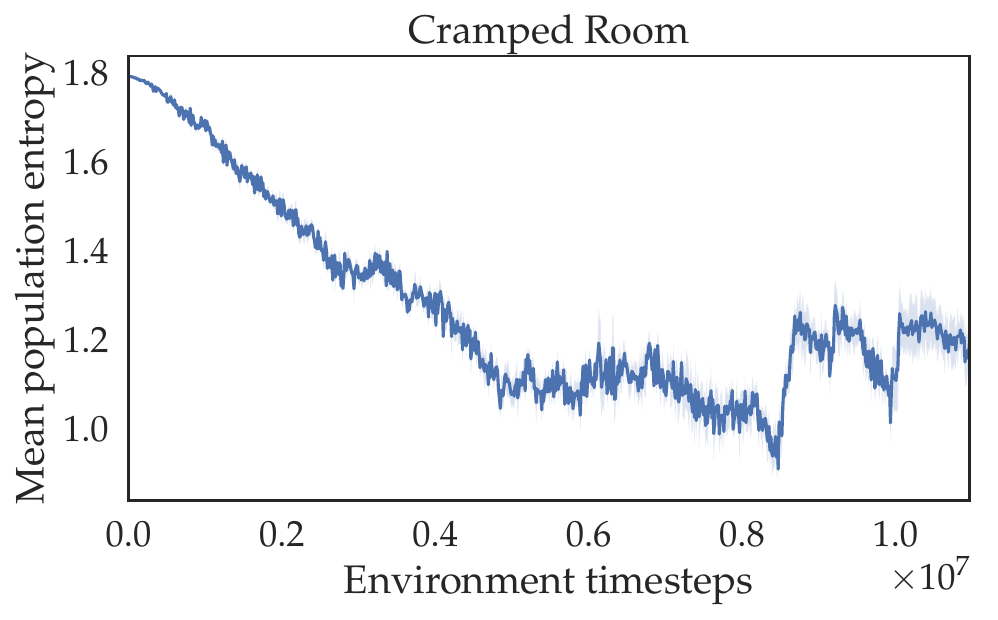} & \includegraphics[width=\linewidth]{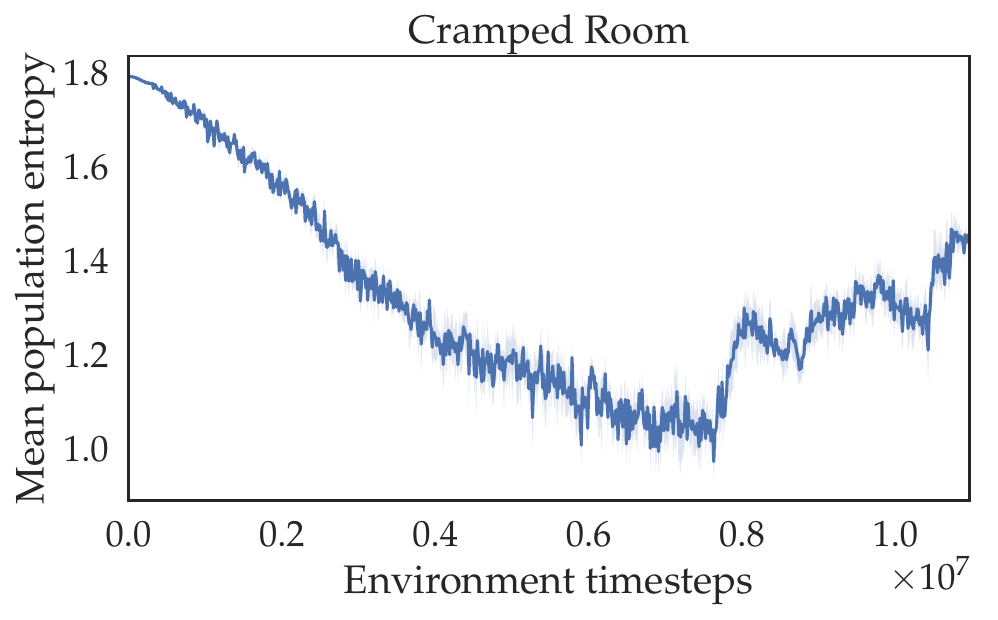} \\ \midrule
      \includegraphics[width=\linewidth]{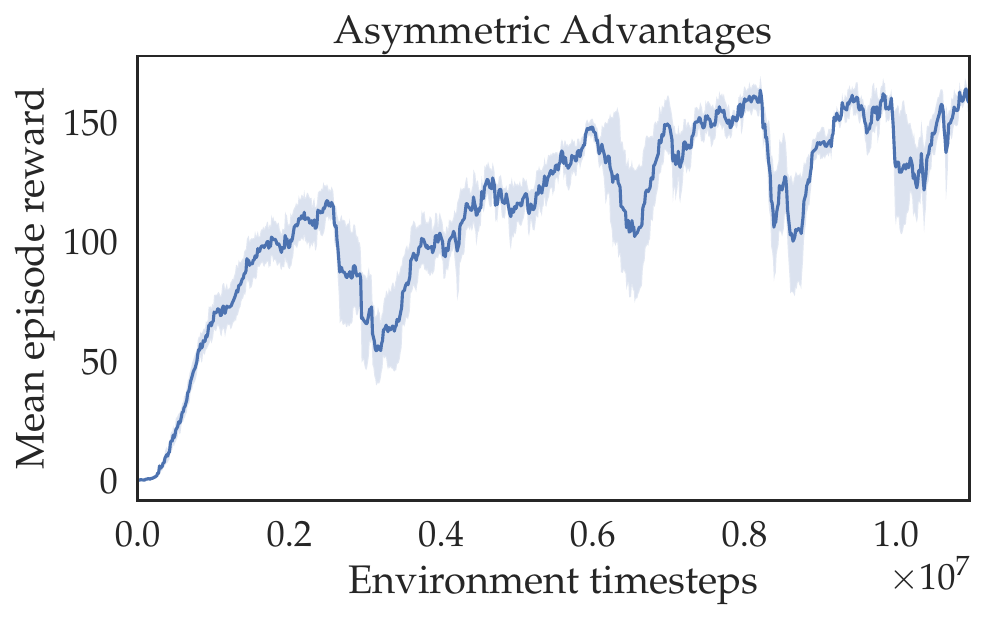} & \includegraphics[width=\linewidth]{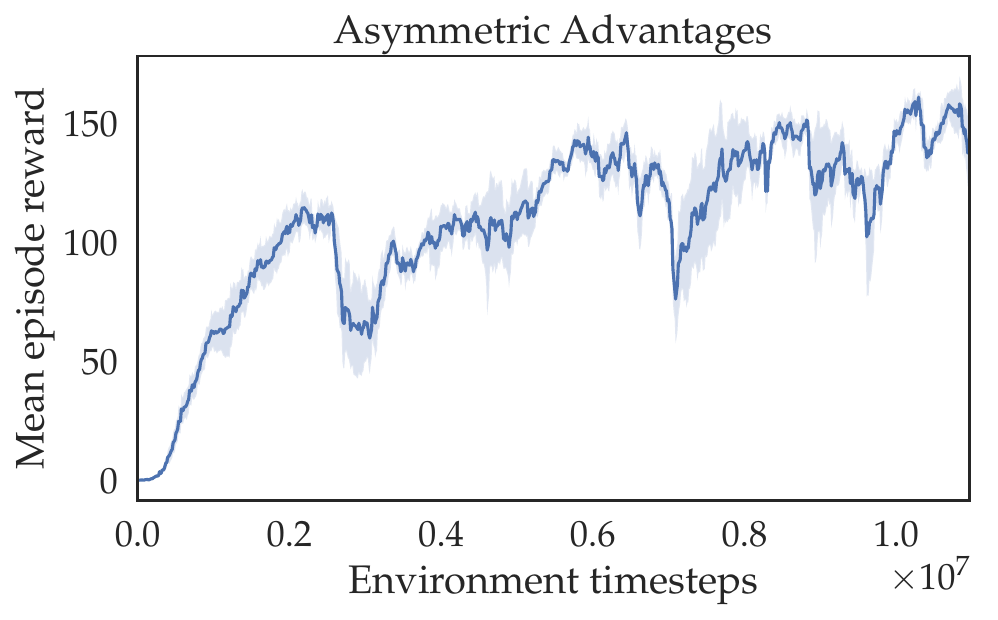} & \includegraphics[width=\linewidth]{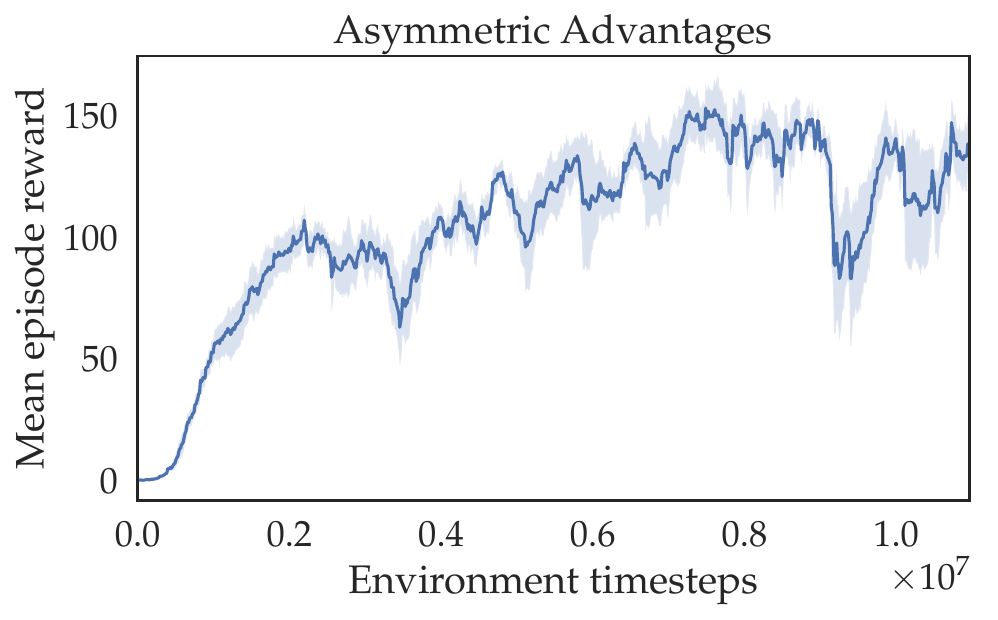} & \includegraphics[width=\linewidth]{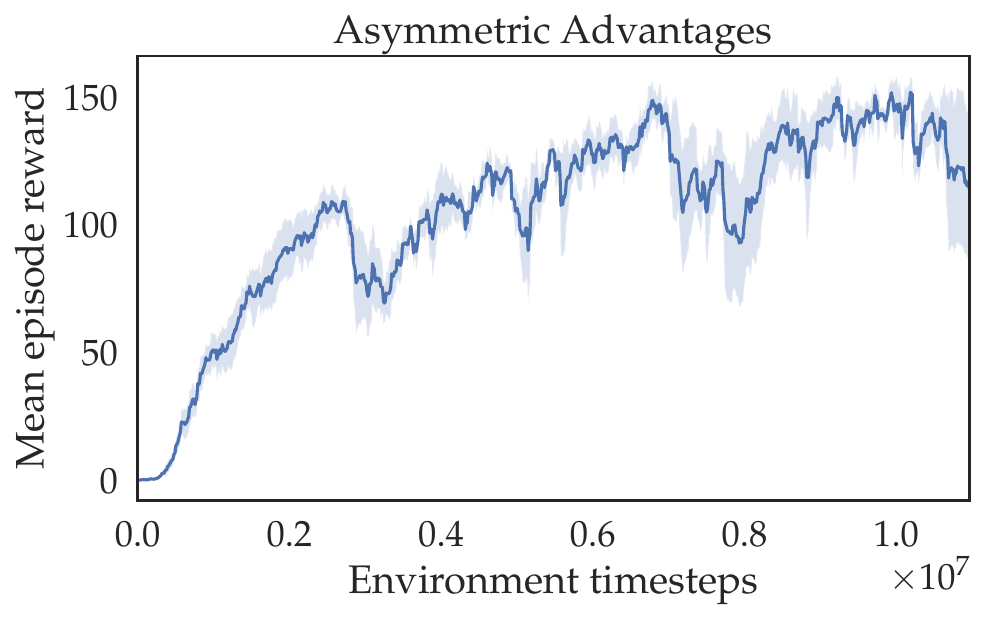} \\
      \includegraphics[width=\linewidth]{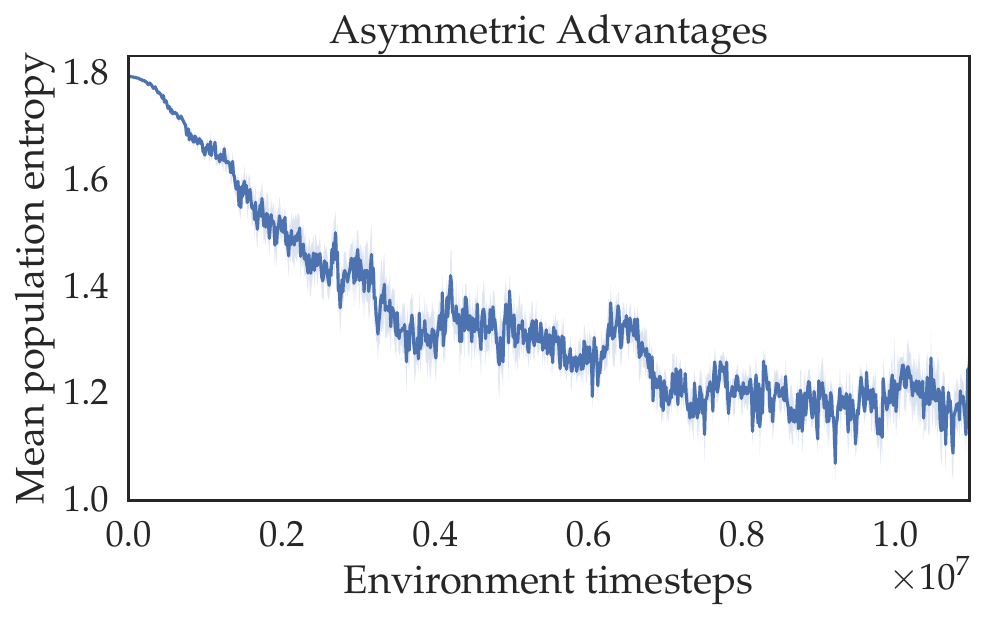} & \includegraphics[width=\linewidth]{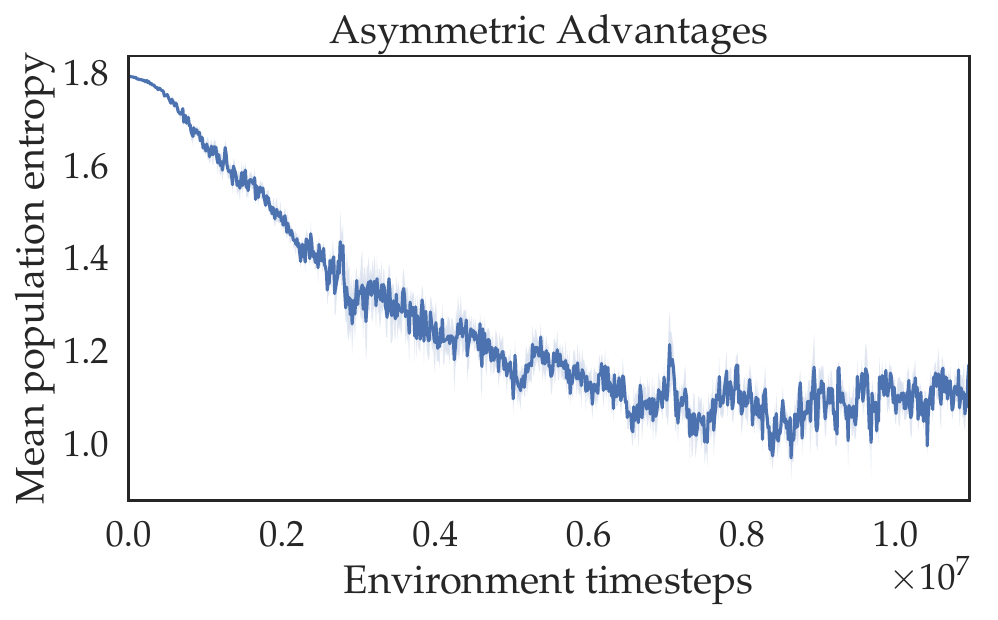} & \includegraphics[width=\linewidth]{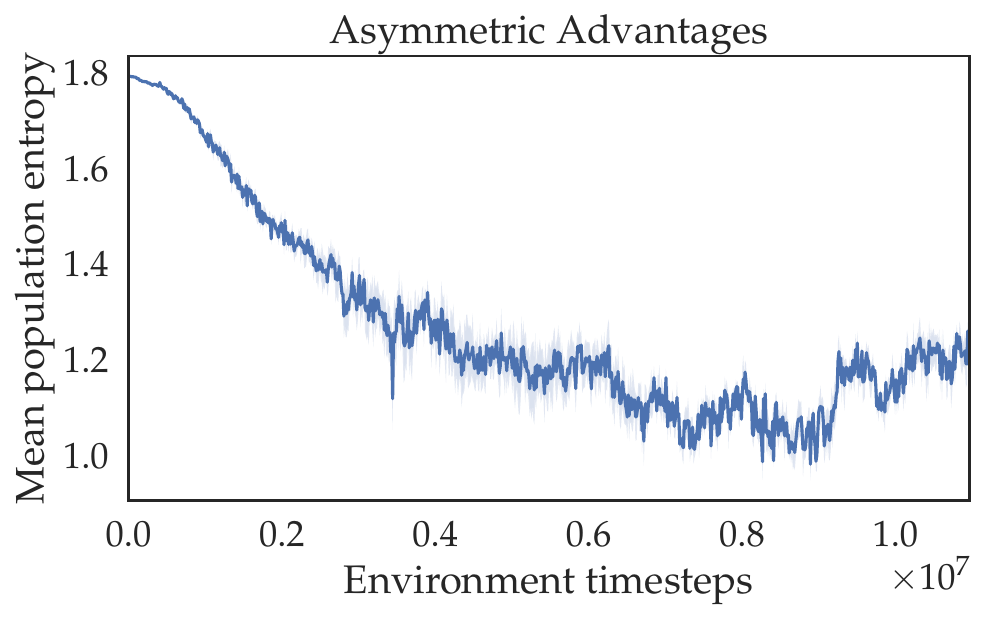} & \includegraphics[width=\linewidth]{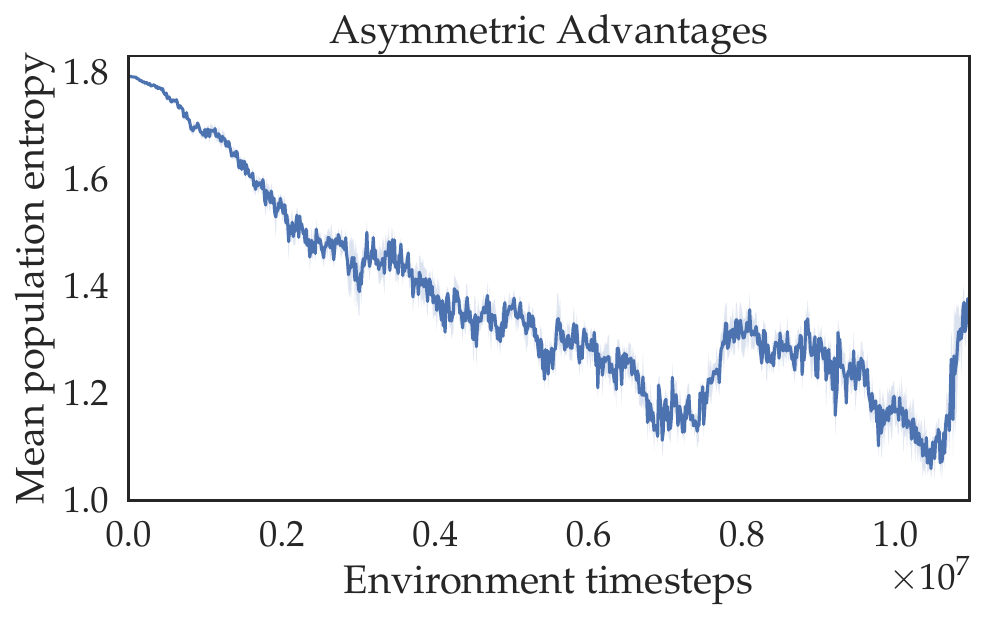} \\ \midrule
      \includegraphics[width=\linewidth]{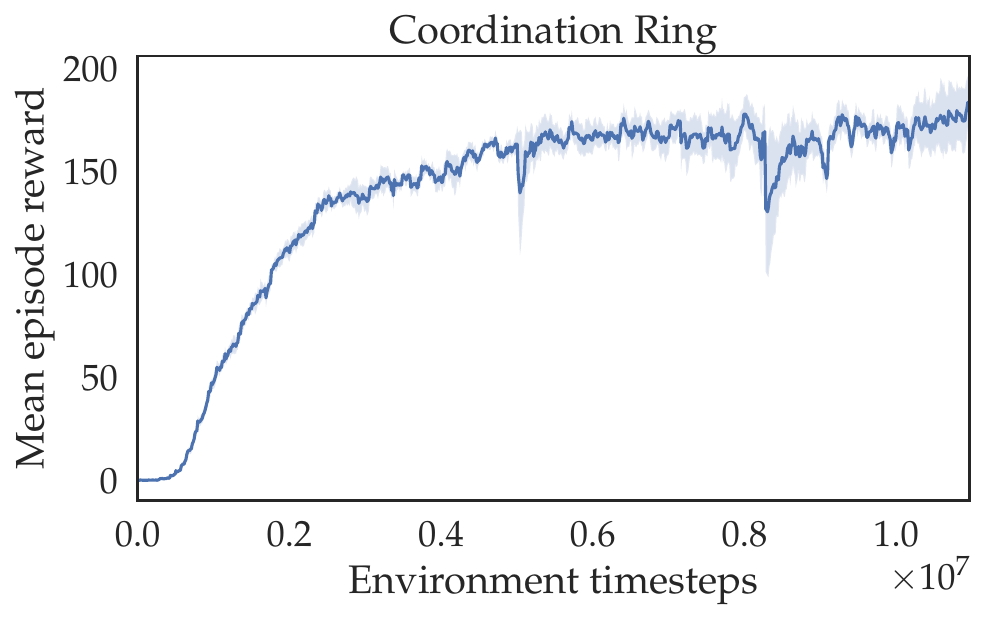} & \includegraphics[width=\linewidth]{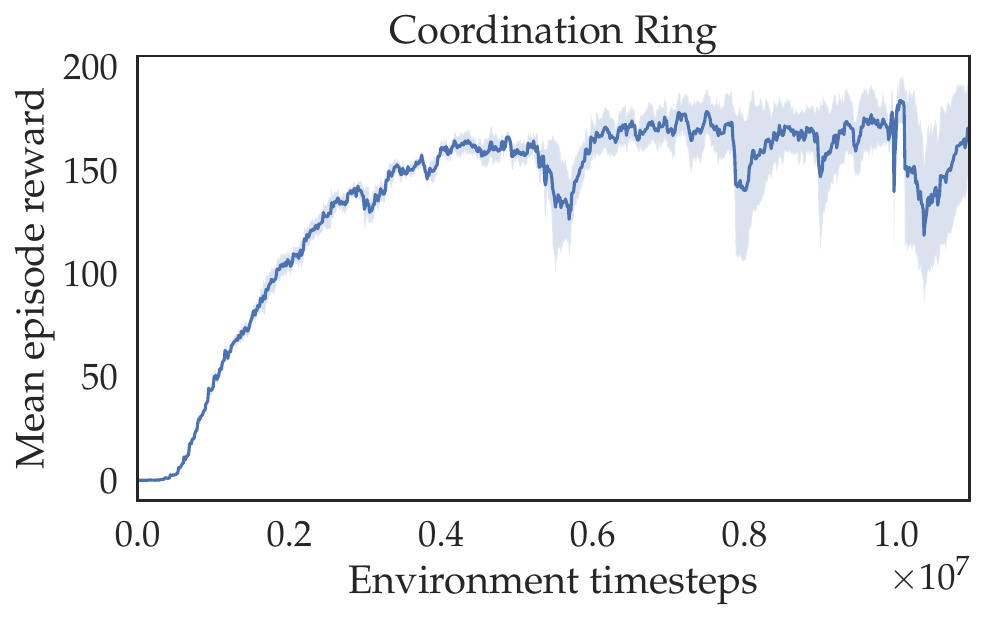} & \includegraphics[width=\linewidth]{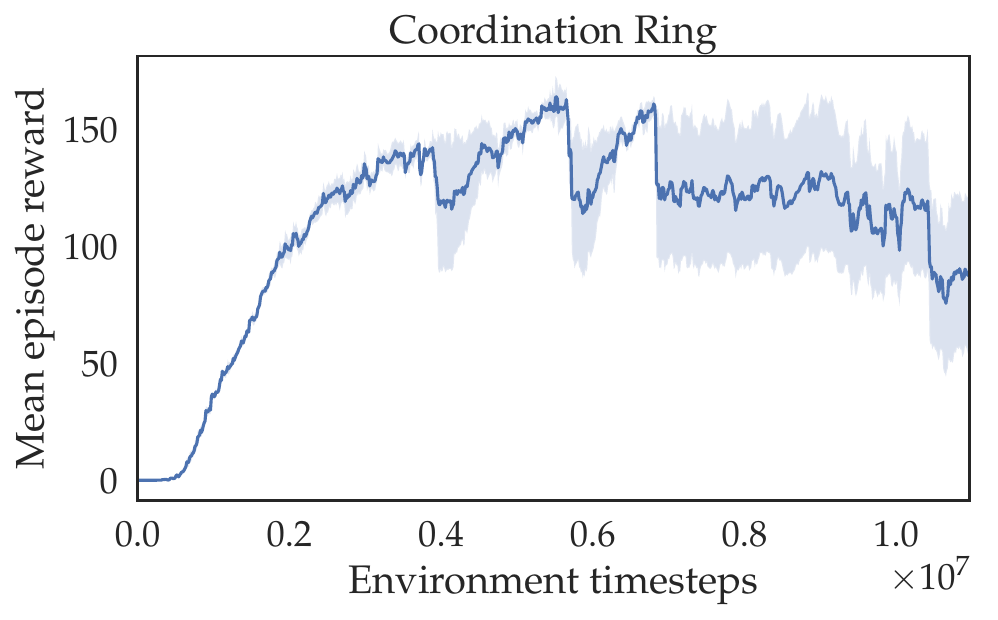} & \includegraphics[width=\linewidth]{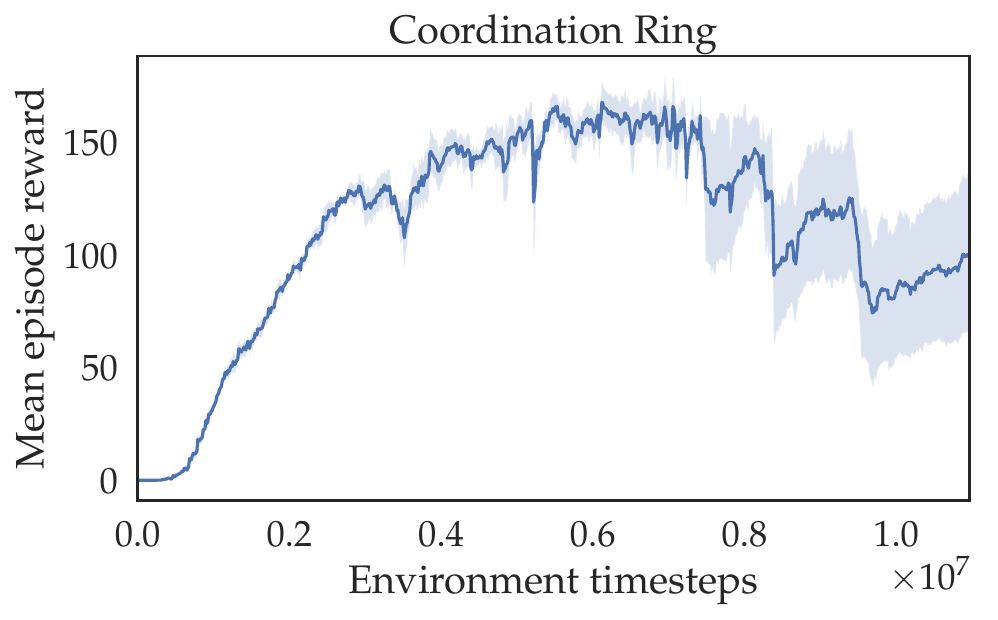} \\
      \includegraphics[width=\linewidth]{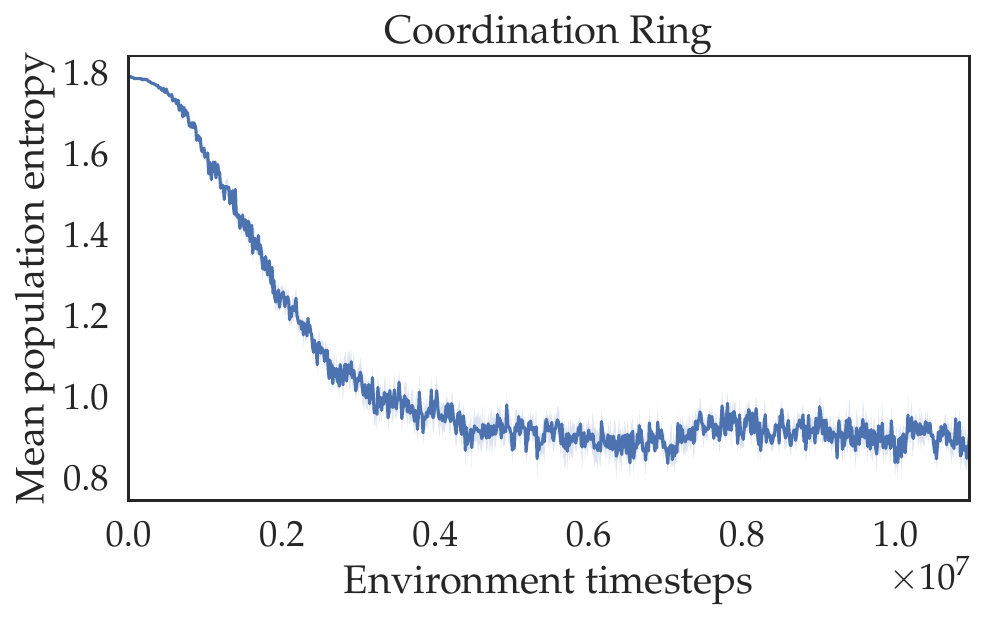} & \includegraphics[width=\linewidth]{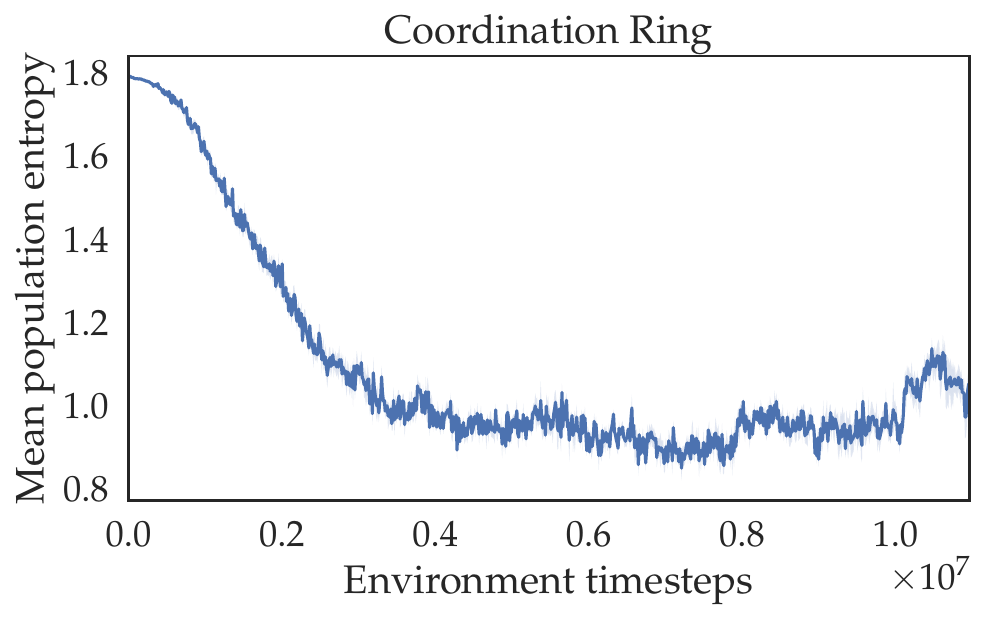} & \includegraphics[width=\linewidth]{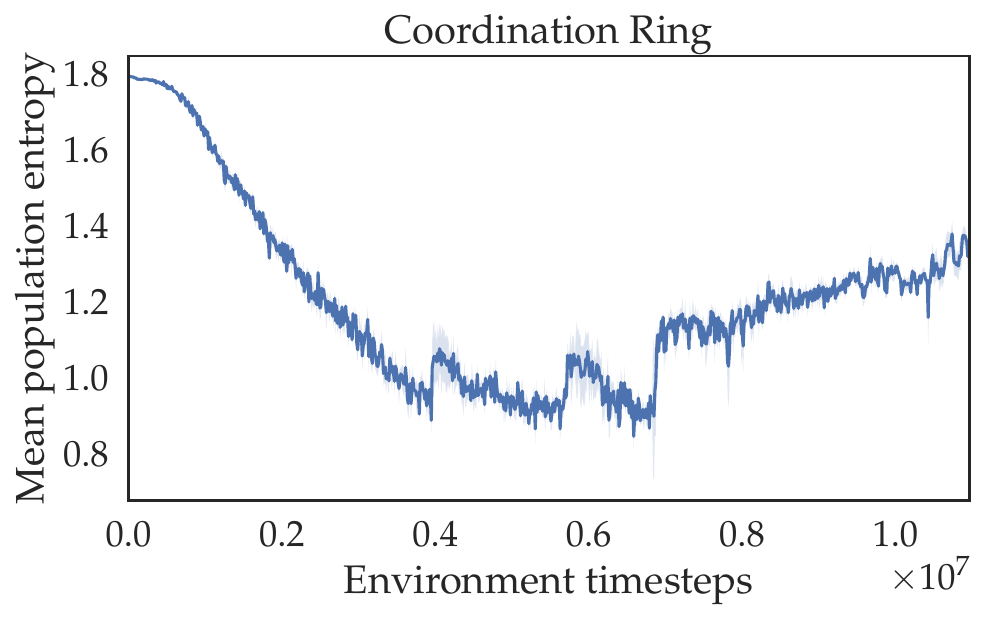} & \includegraphics[width=\linewidth]{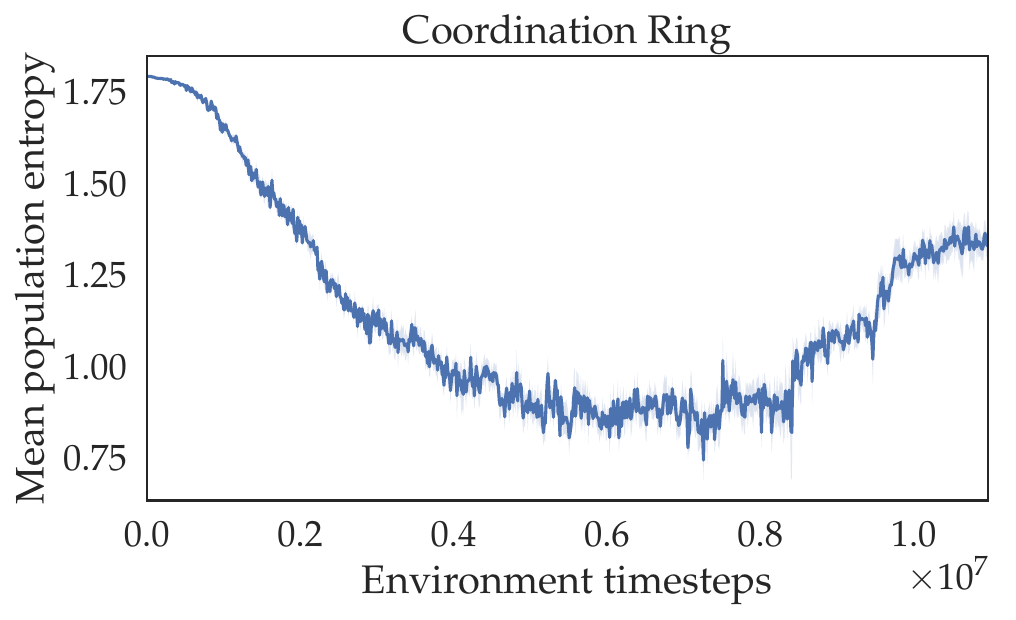} \\ \midrule
      \includegraphics[width=\linewidth]{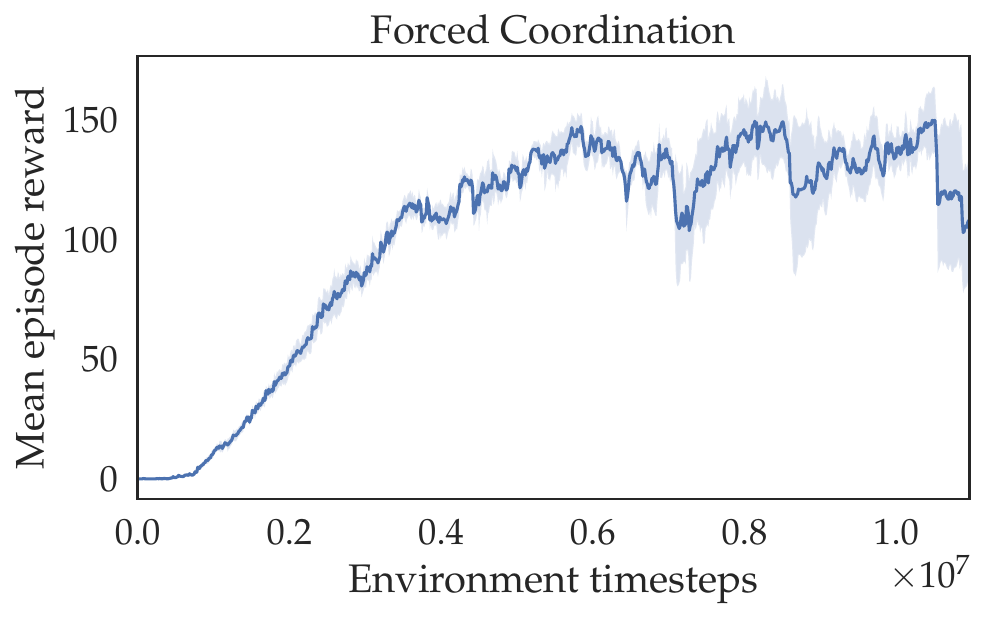} & \includegraphics[width=\linewidth]{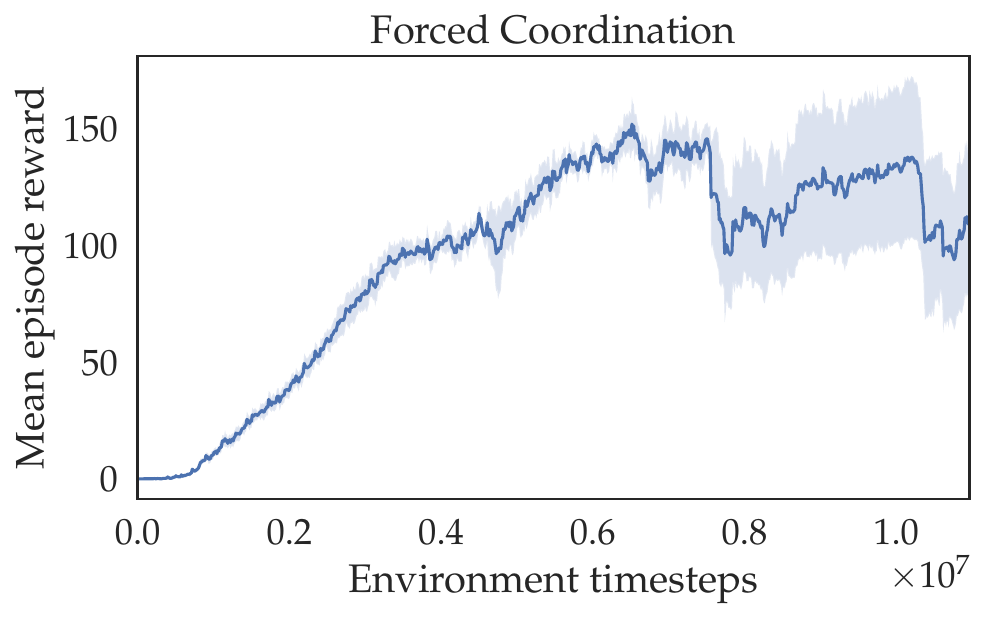} & \includegraphics[width=\linewidth]{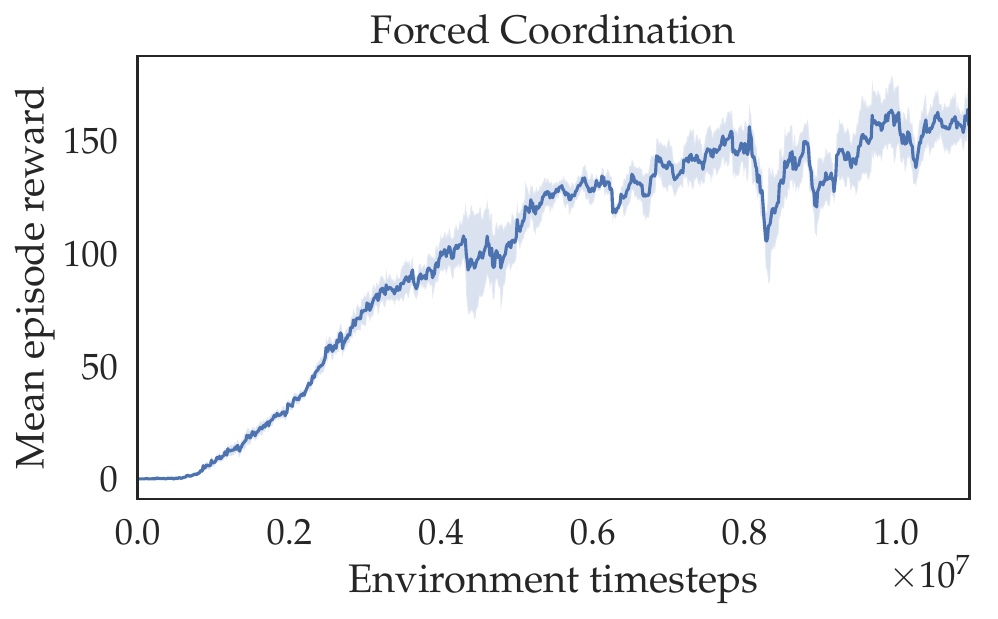} & \includegraphics[width=\linewidth]{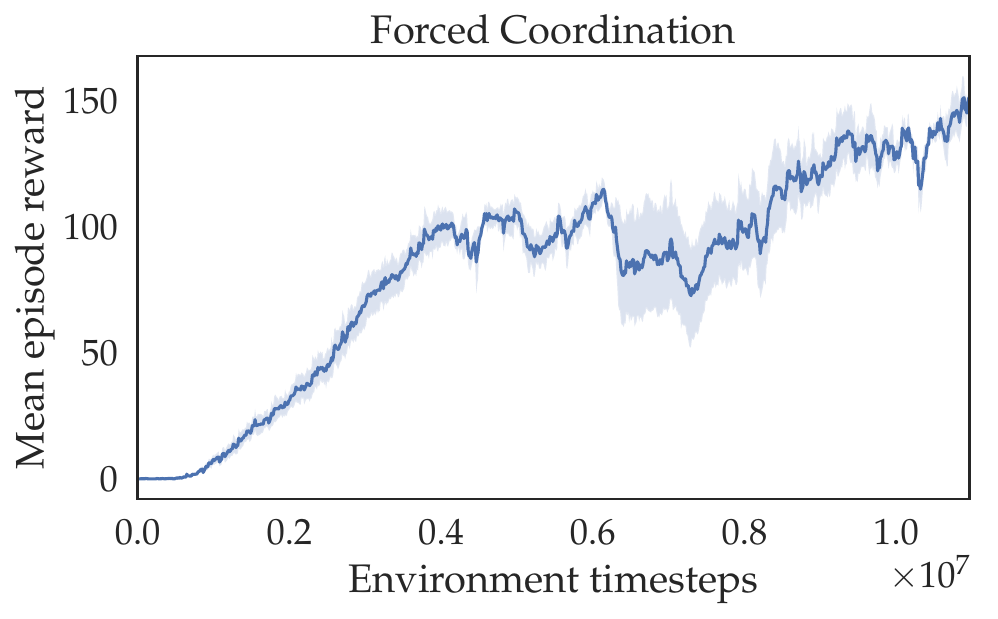} \\
      \includegraphics[width=\linewidth]{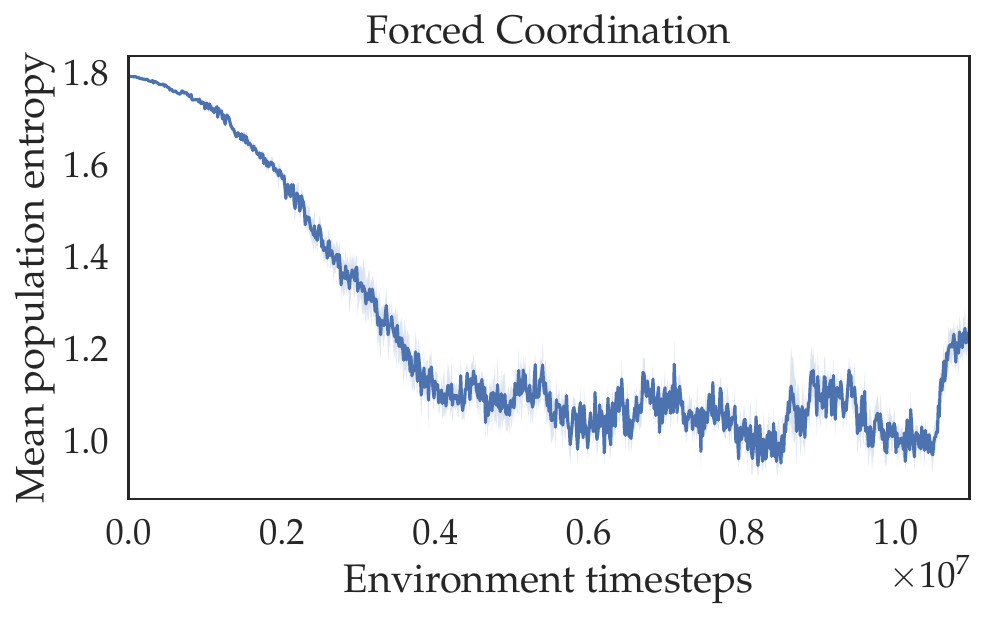} & \includegraphics[width=\linewidth]{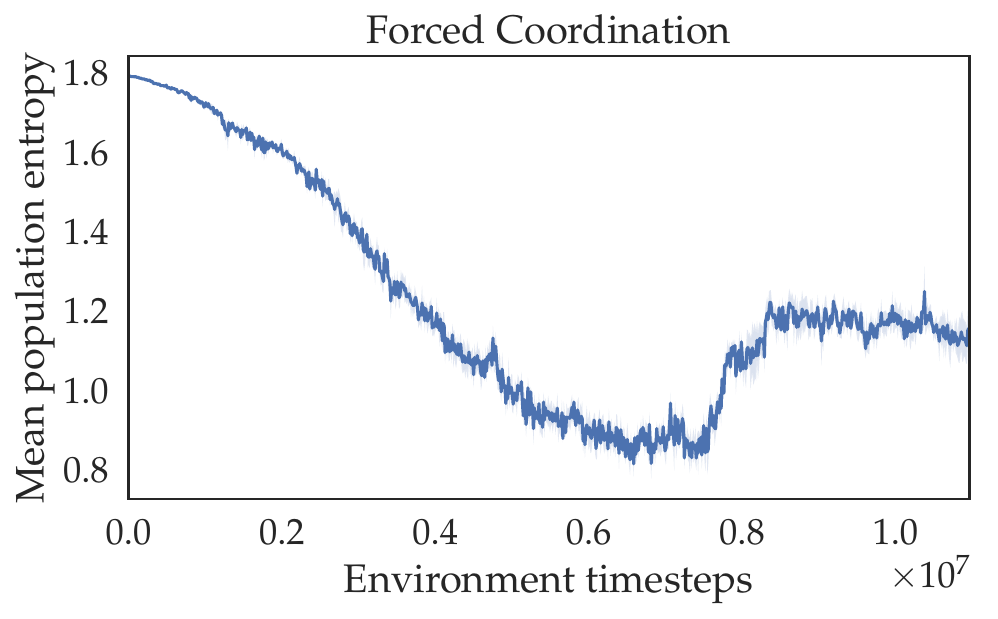} & \includegraphics[width=\linewidth]{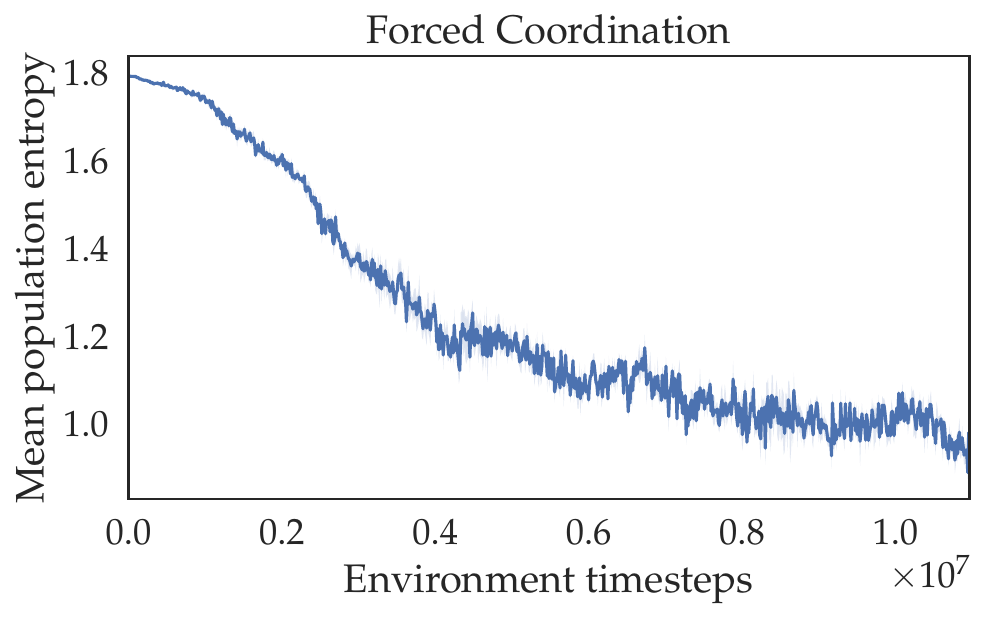} & \includegraphics[width=\linewidth]{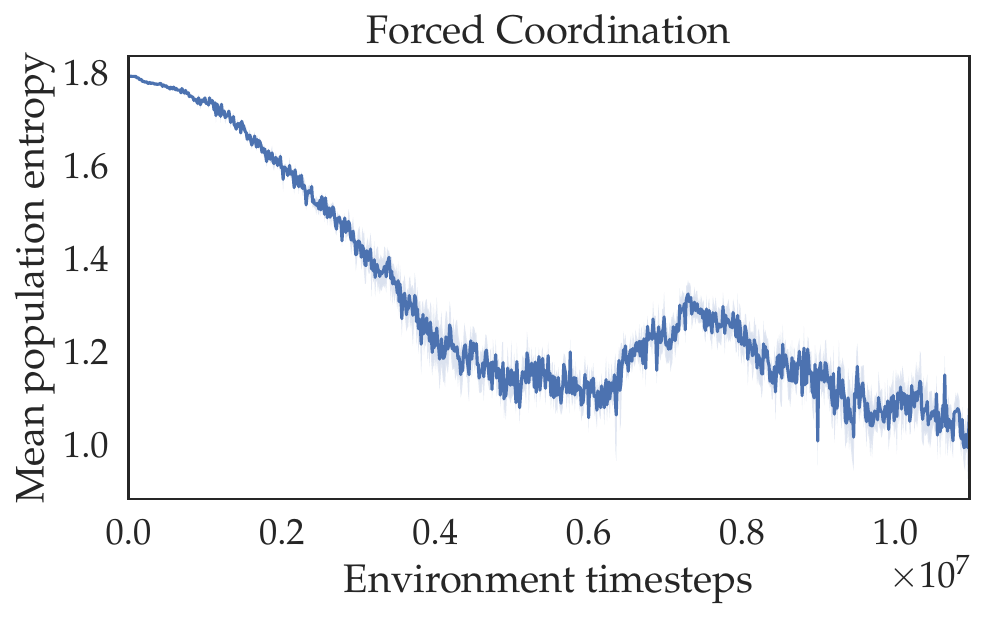} \\ \midrule
      \includegraphics[width=\linewidth]{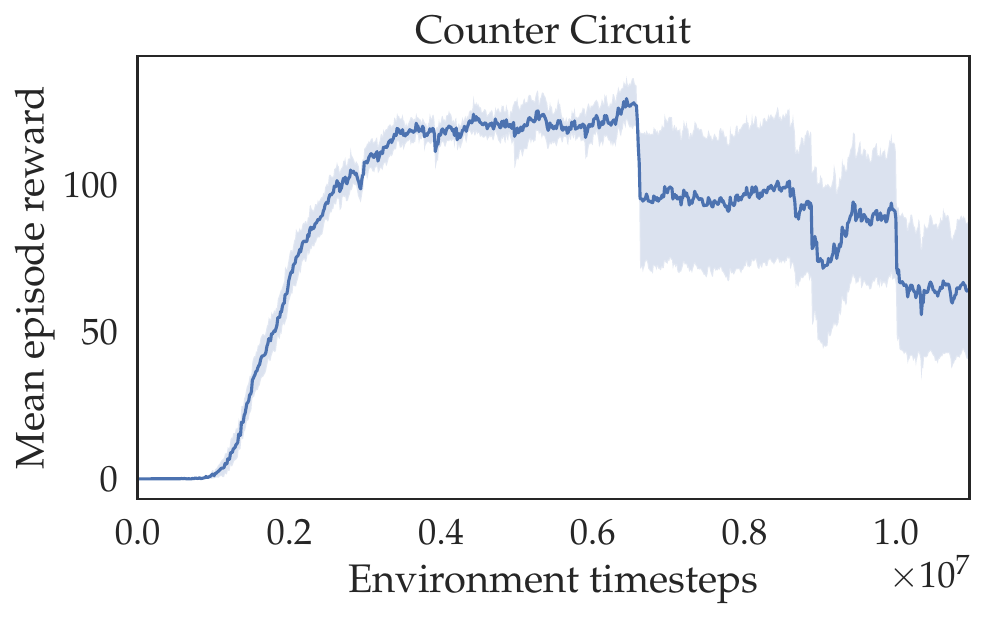} & \includegraphics[width=\linewidth]{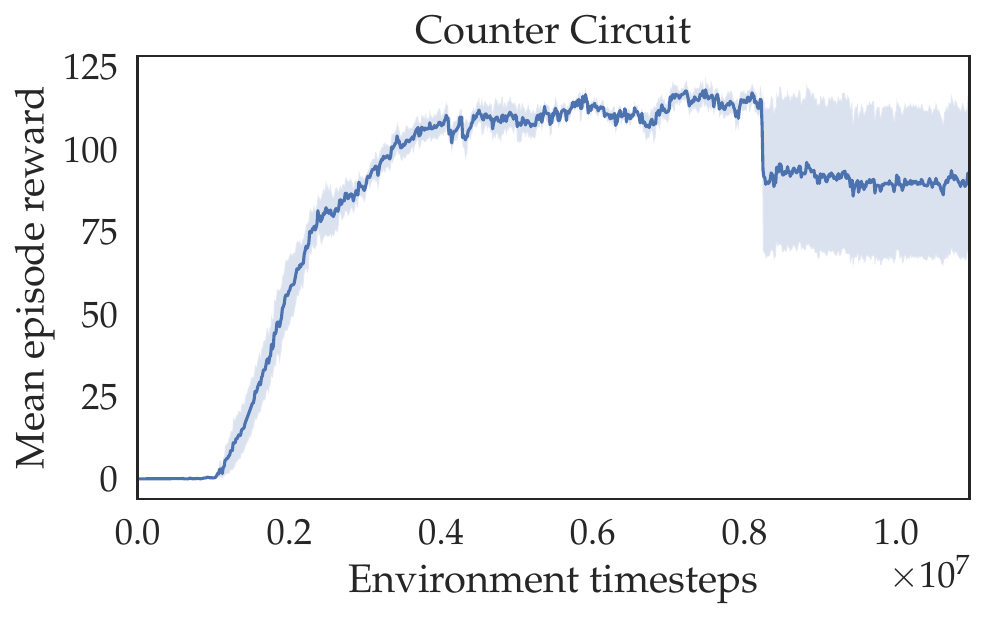} & \includegraphics[width=\linewidth]{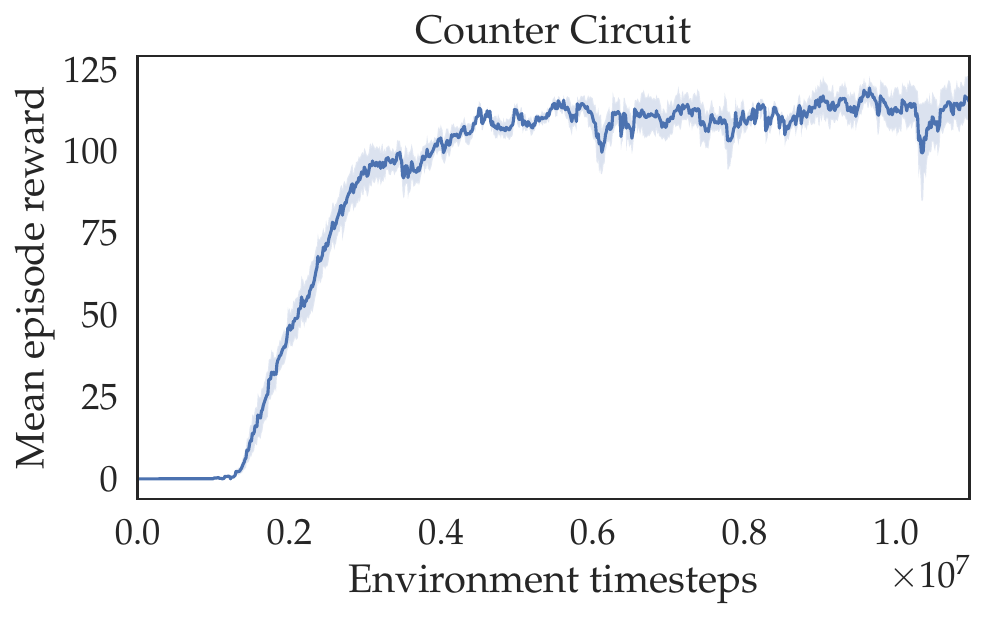} & \includegraphics[width=\linewidth]{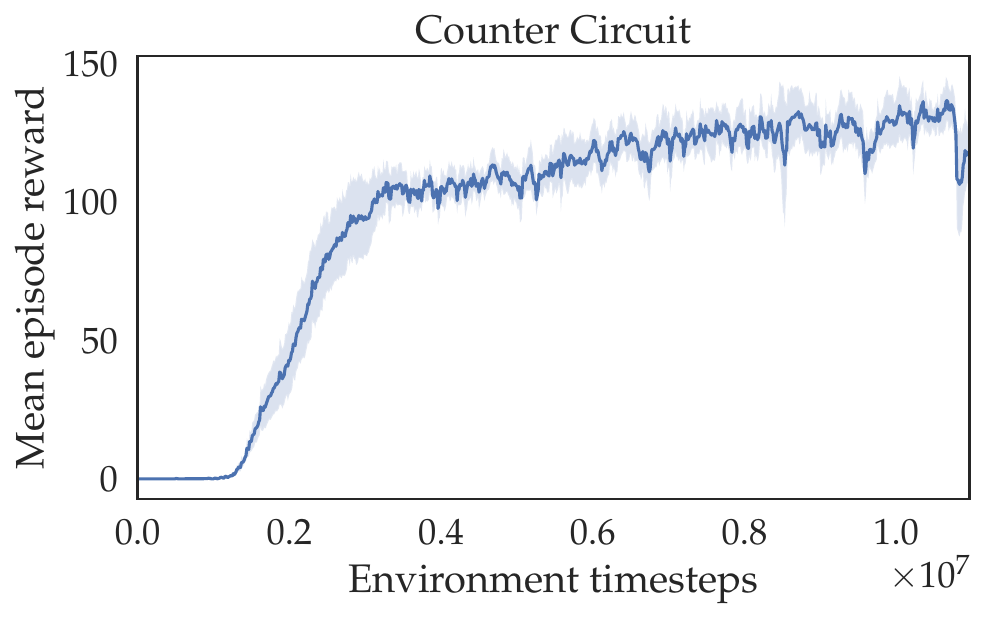} \\
      \includegraphics[width=\linewidth]{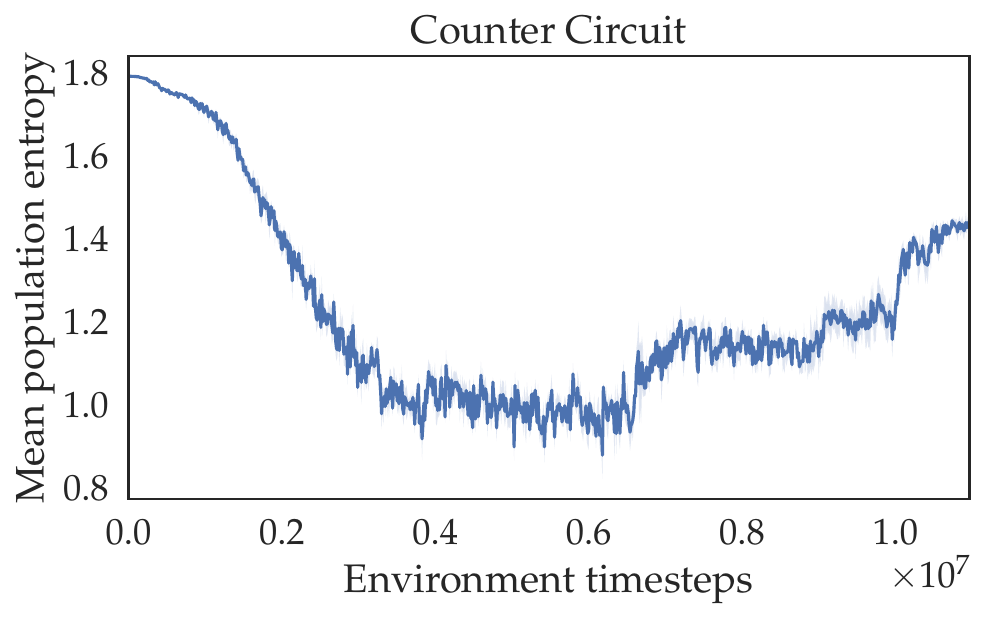} & \includegraphics[width=\linewidth]{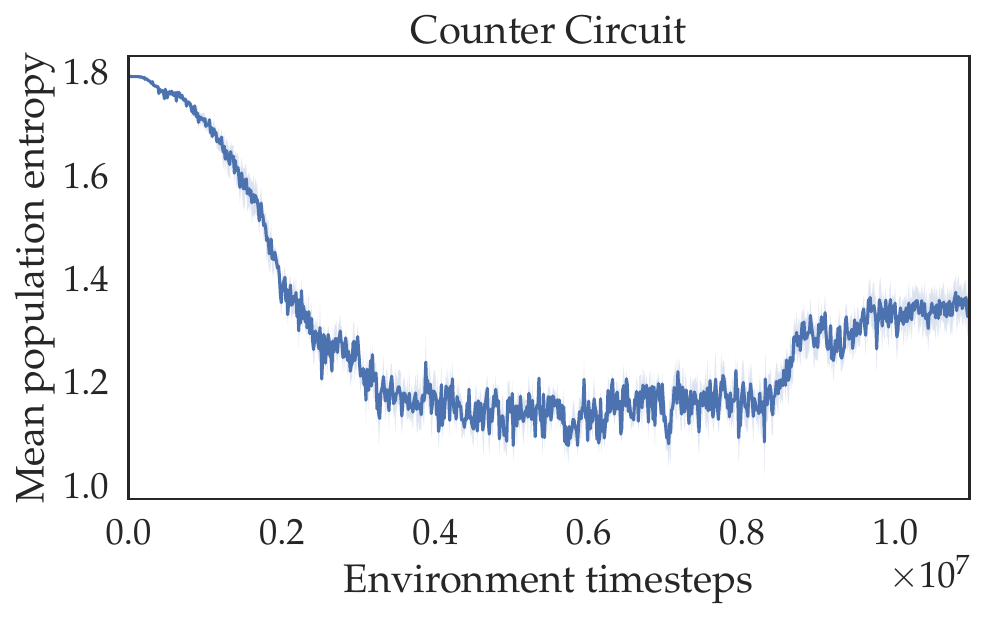} & \includegraphics[width=\linewidth]{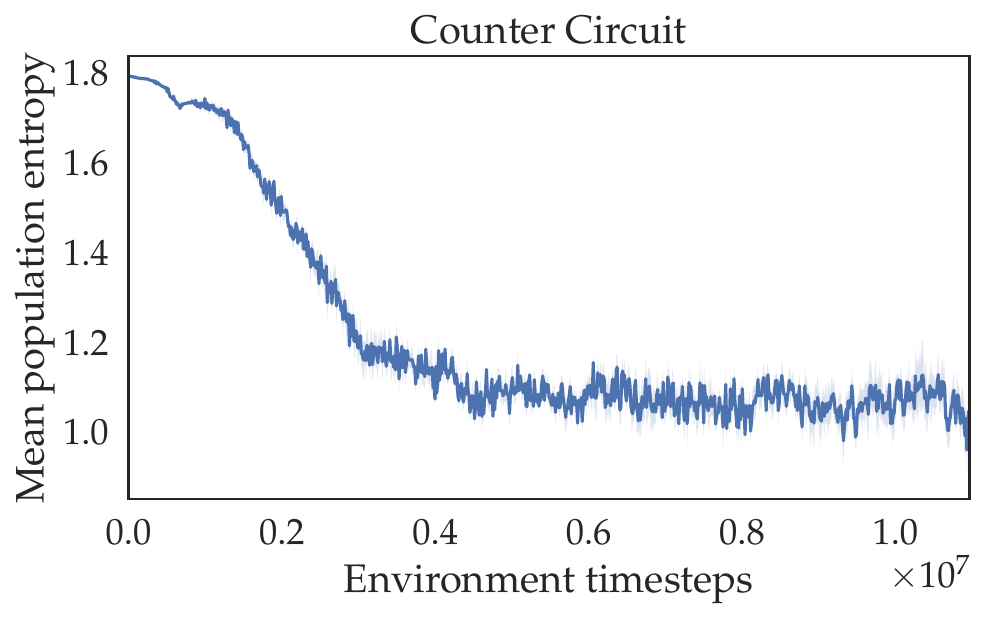} & \includegraphics[width=\linewidth]{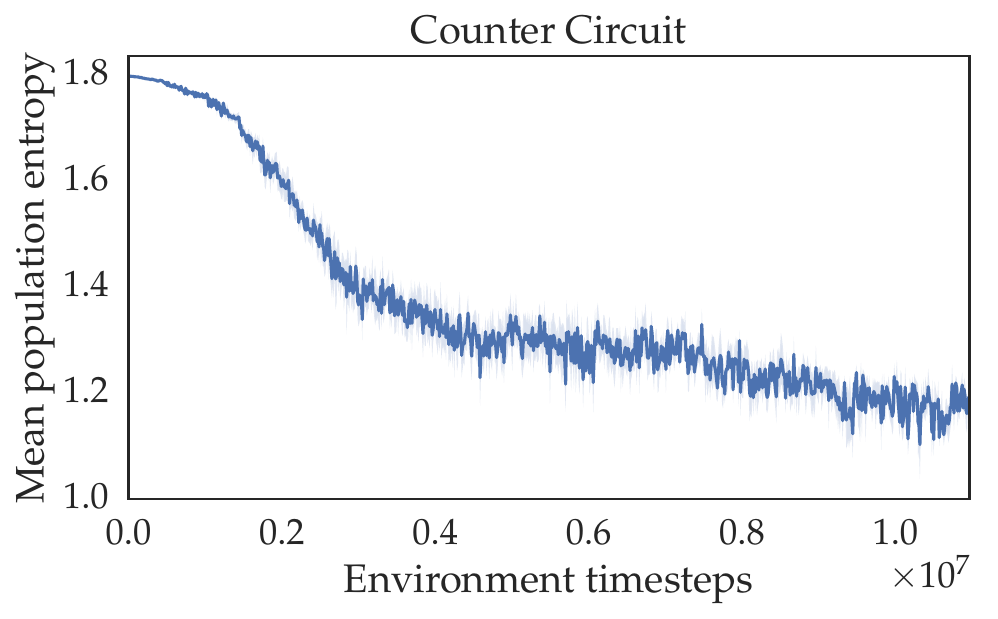} \\
      \bottomrule
  \end{tabular}
  \caption{\bo{Mean episode reward and population entropy with different $\alpha$ in all five layouts:} Each column corresponds to a different value of $\alpha$ in the set of $[0.000,\ 0.001,\ 0.005,\ 0.010]$. There are five row sections, which correspond to the five layouts. Each row section contains two rows, which are the plots of the mean episode reward and the mean population entropy of the layout, respectively.}
  \label{tab:rew_and_ent_1}
\end{table}

\begin{table}[h]
  \centering
  \begin{tabular}{M{0.218\linewidth}M{0.218\linewidth}M{0.218\linewidth}M{0.218\linewidth}}
     \toprule
      $\alpha$=0.020 & $\alpha$=0.030 & $\alpha$=0.040 & $\alpha$=0.050 \\
      \midrule
      \includegraphics[width=\linewidth]{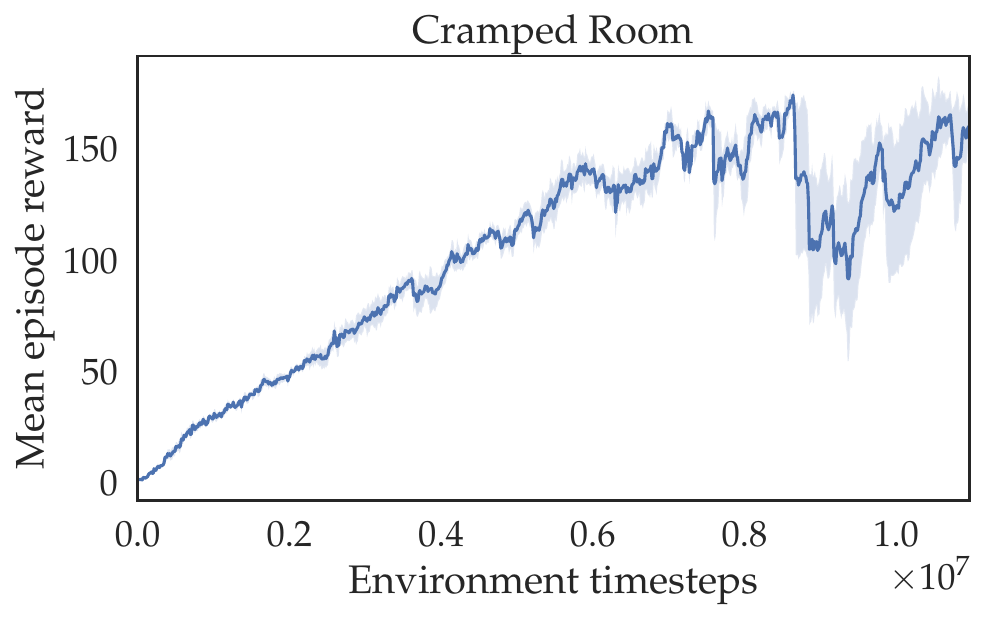} & \includegraphics[width=\linewidth]{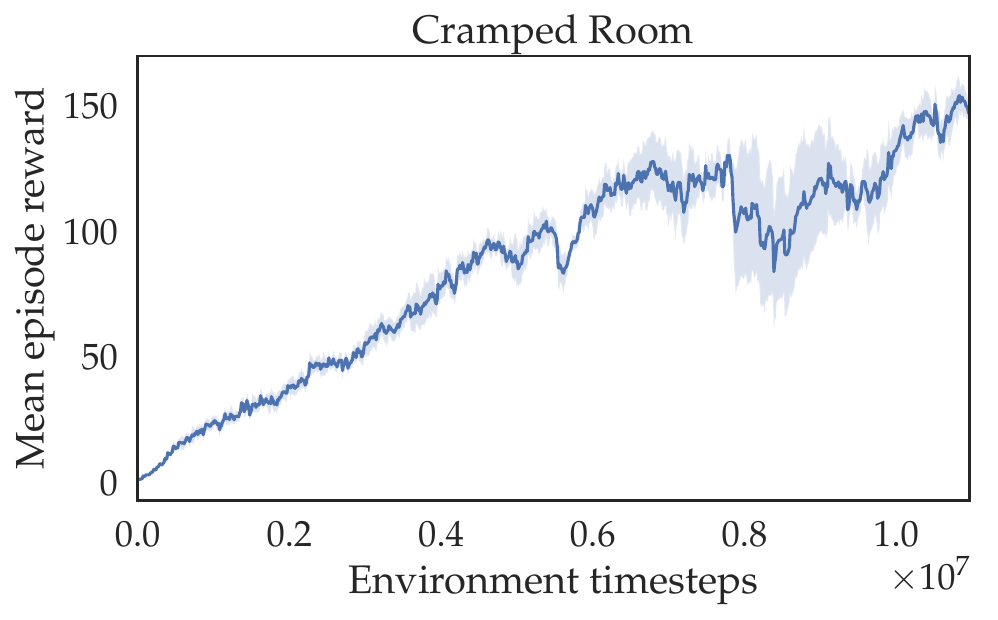} & \includegraphics[width=\linewidth]{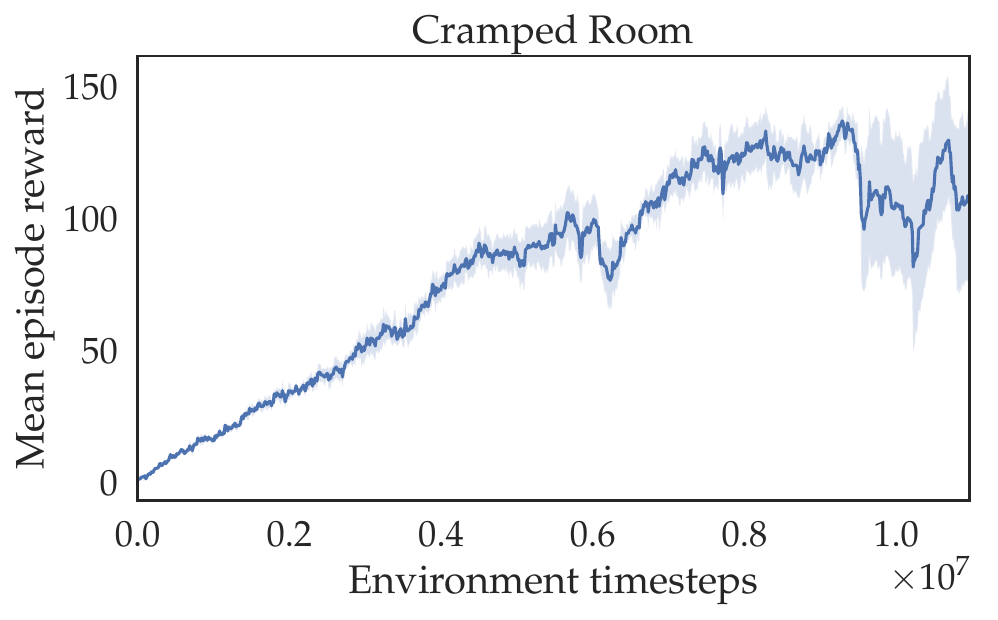} & \includegraphics[width=\linewidth]{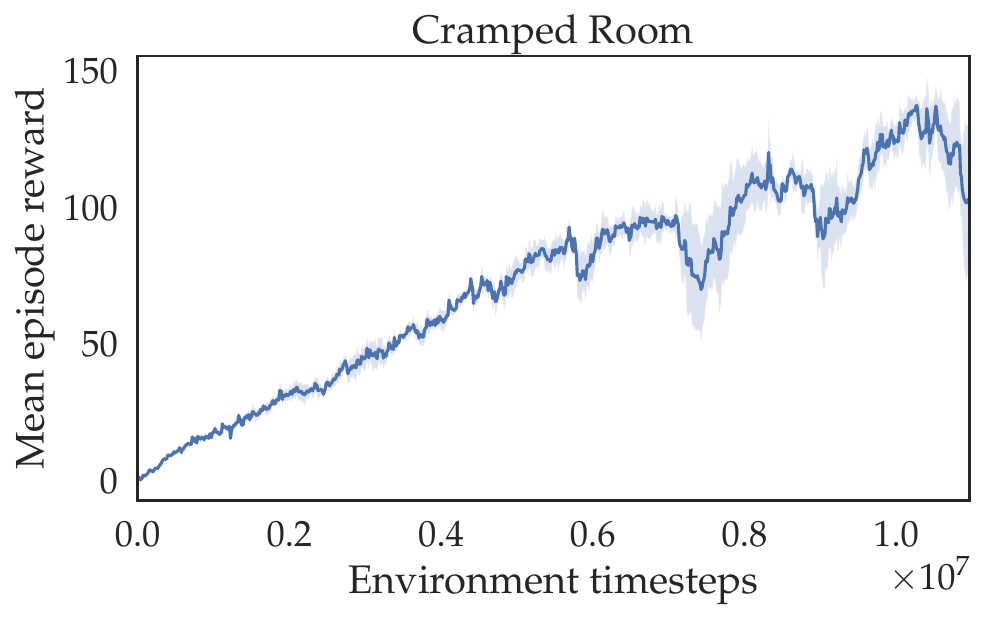} \\
      \includegraphics[width=\linewidth]{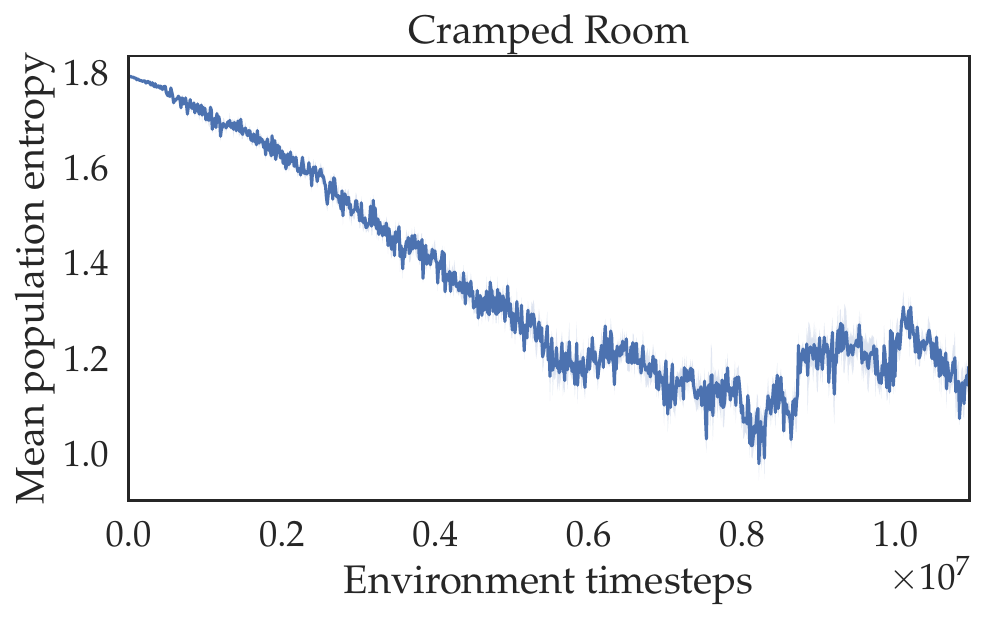} & \includegraphics[width=\linewidth]{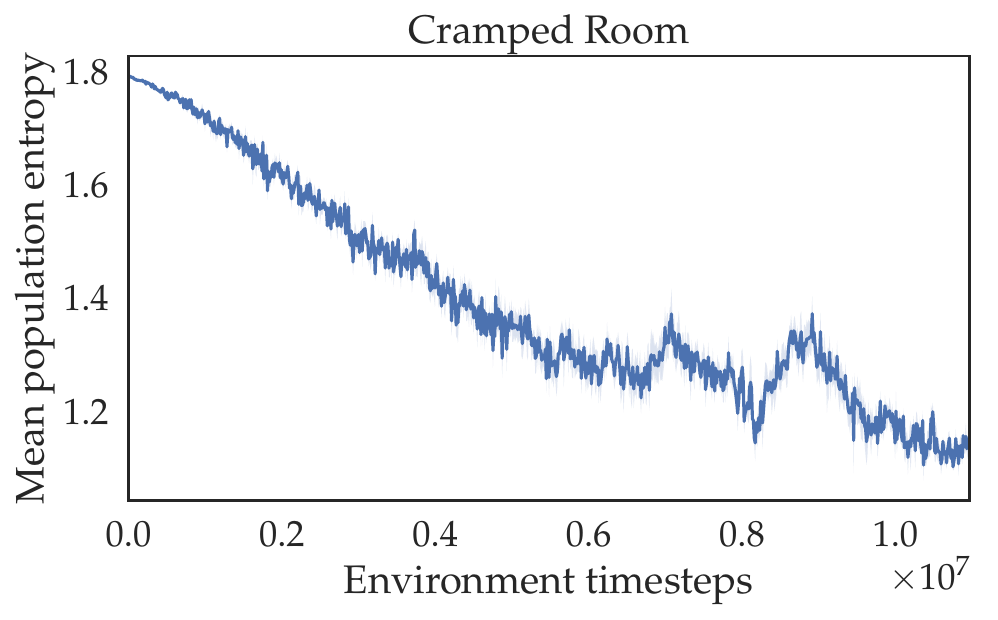} & \includegraphics[width=\linewidth]{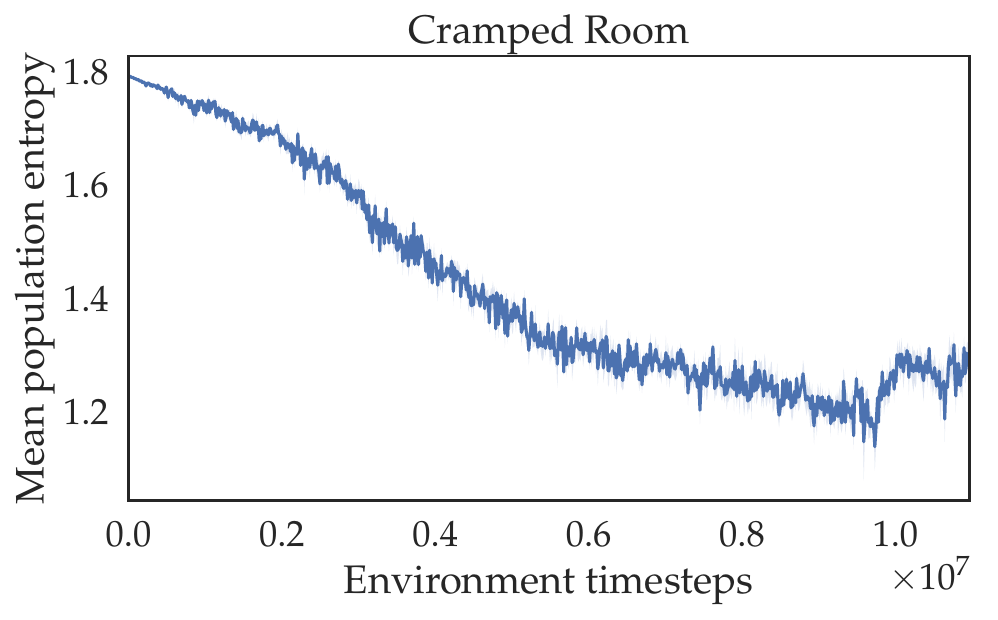} & \includegraphics[width=\linewidth]{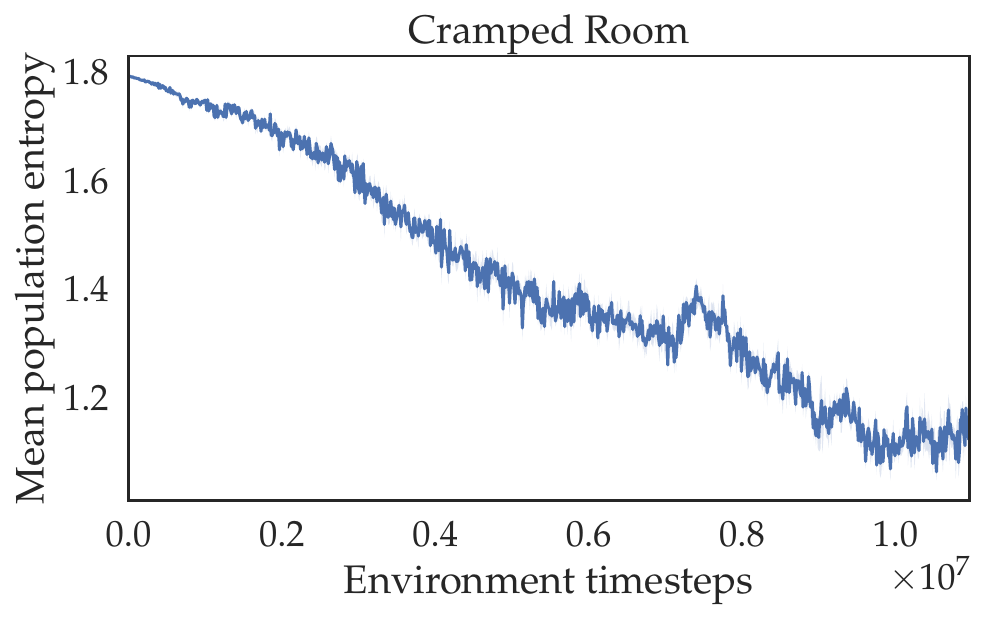} \\ \midrule
      \includegraphics[width=\linewidth]{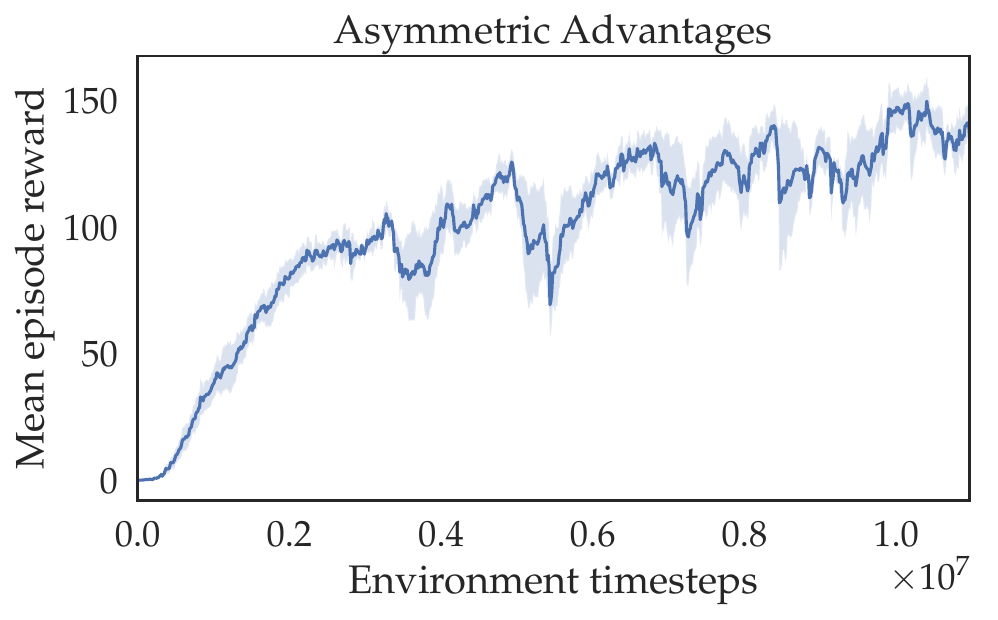} & \includegraphics[width=\linewidth]{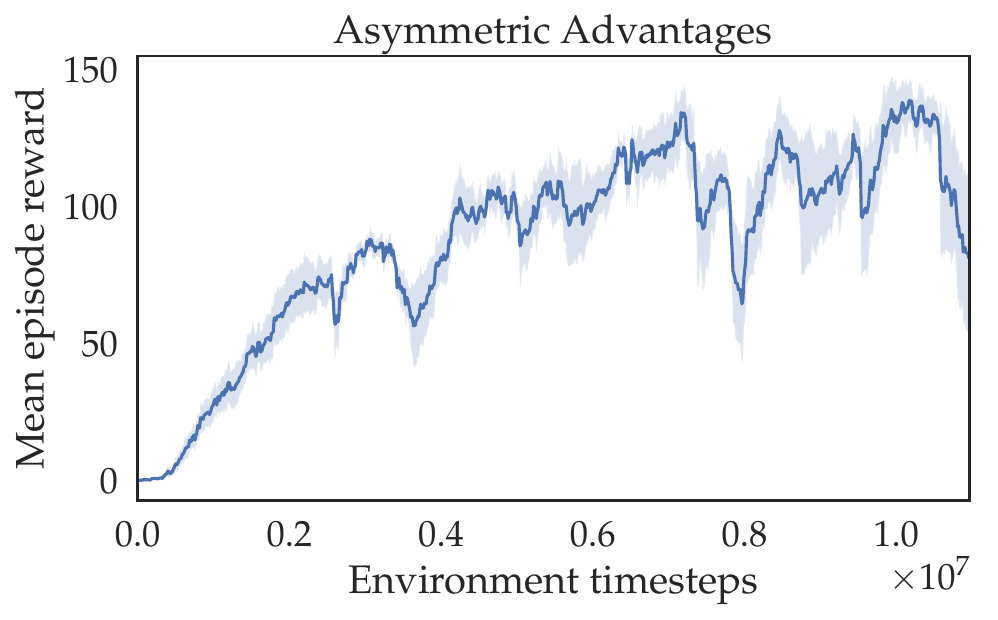} & \includegraphics[width=\linewidth]{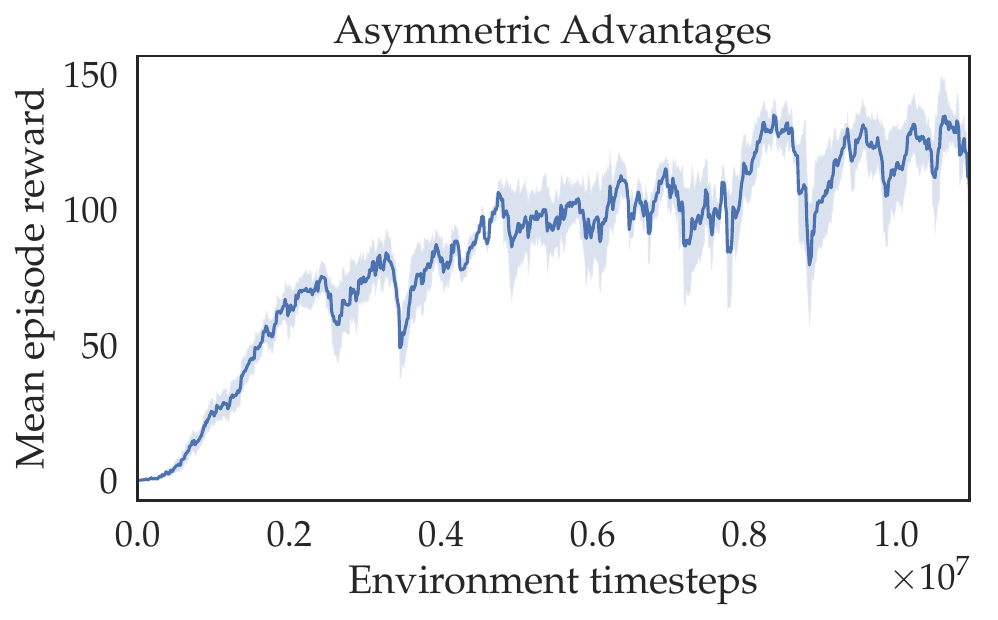} & \includegraphics[width=\linewidth]{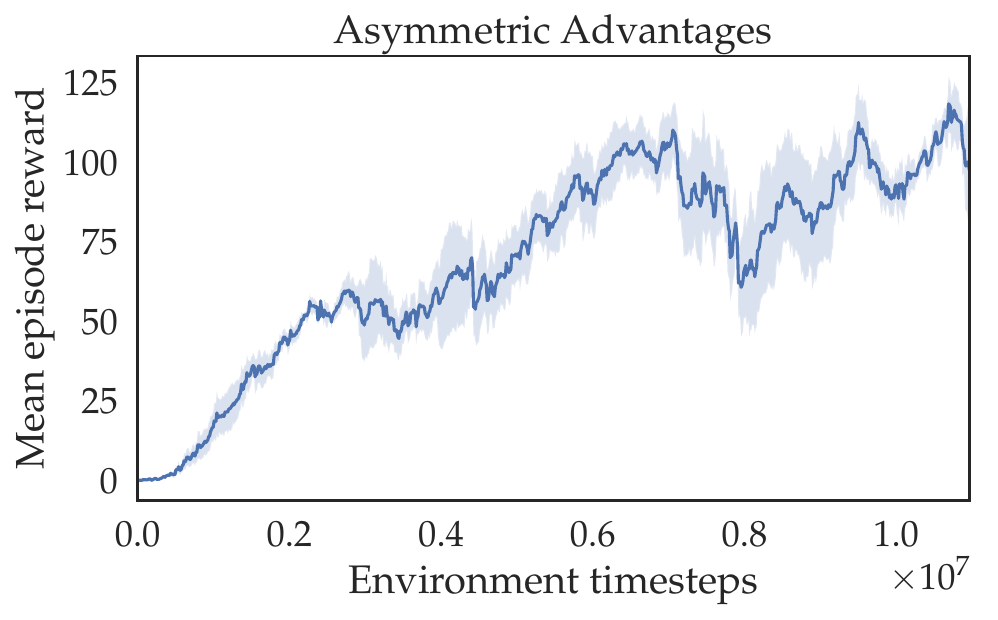} \\
      \includegraphics[width=\linewidth]{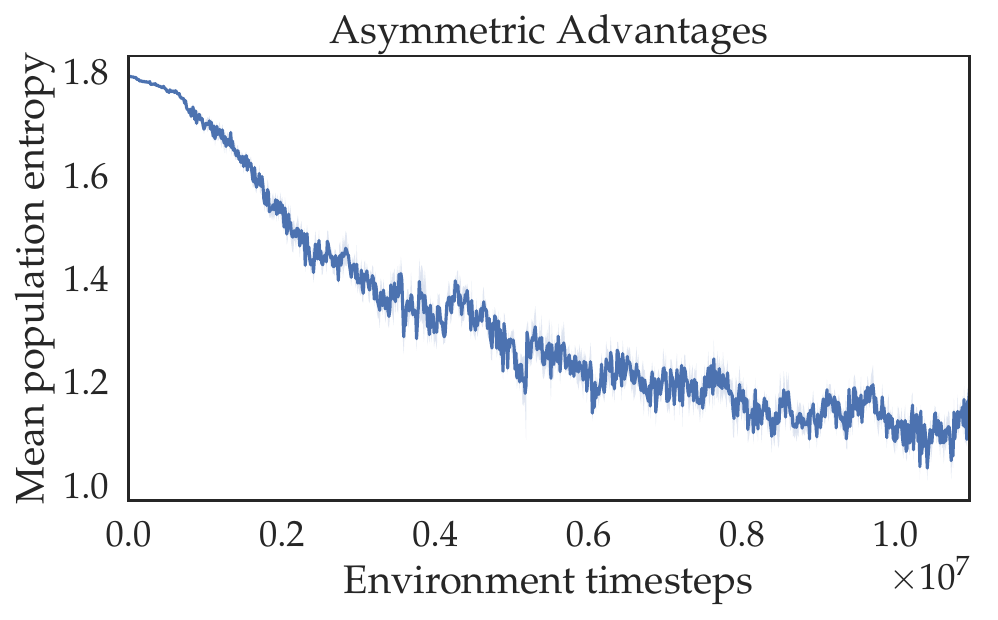} & \includegraphics[width=\linewidth]{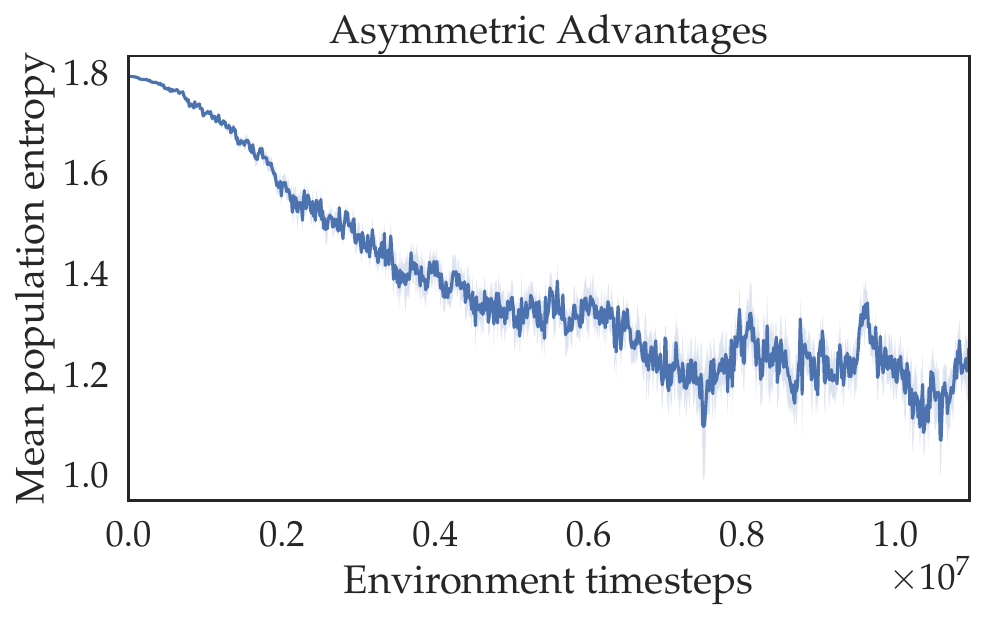} & \includegraphics[width=\linewidth]{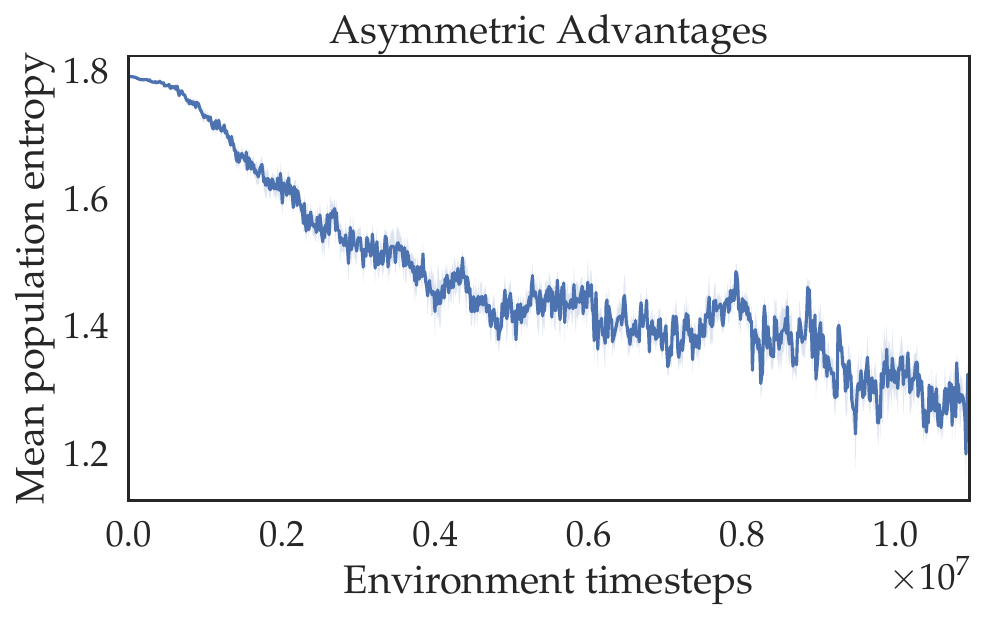} & \includegraphics[width=\linewidth]{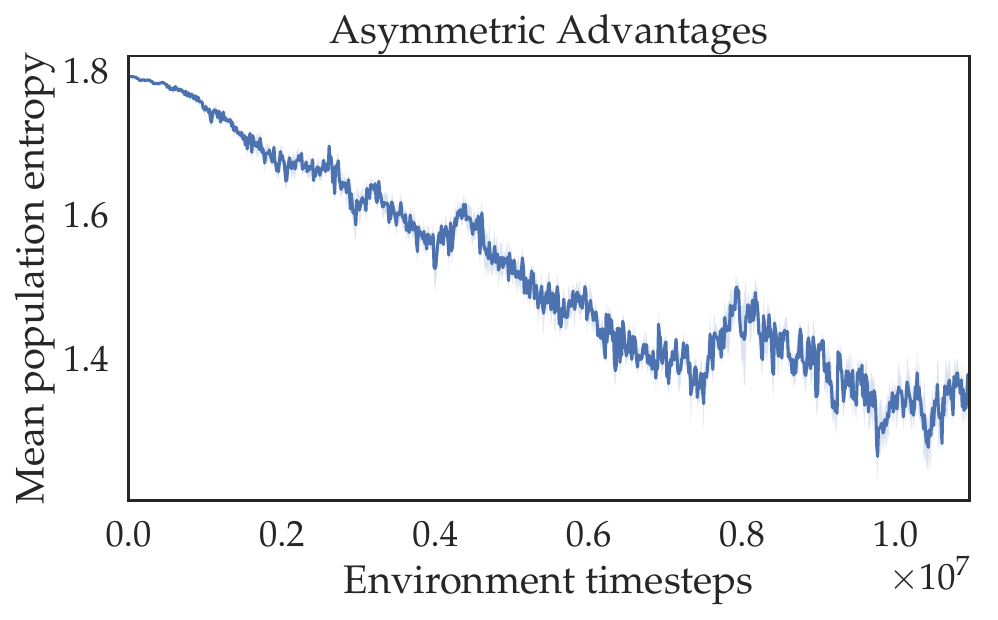} \\ \midrule
      \includegraphics[width=\linewidth]{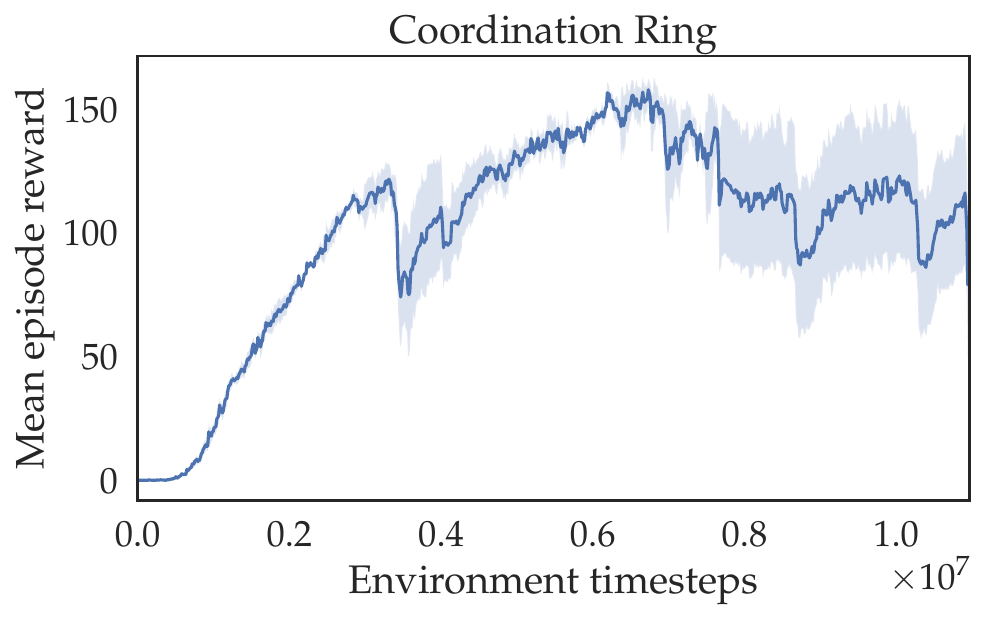} & \includegraphics[width=\linewidth]{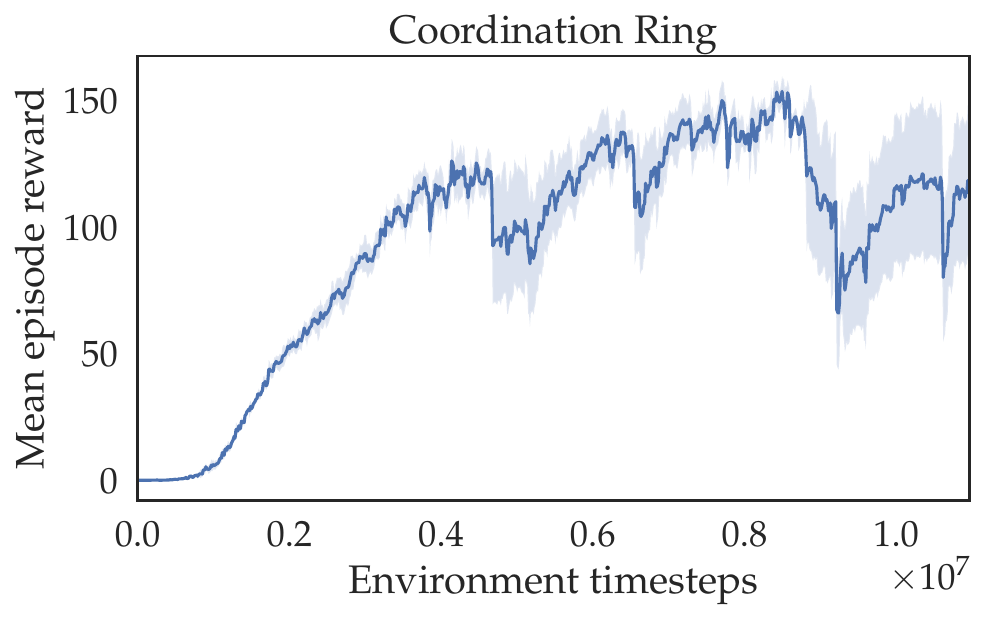} & \includegraphics[width=\linewidth]{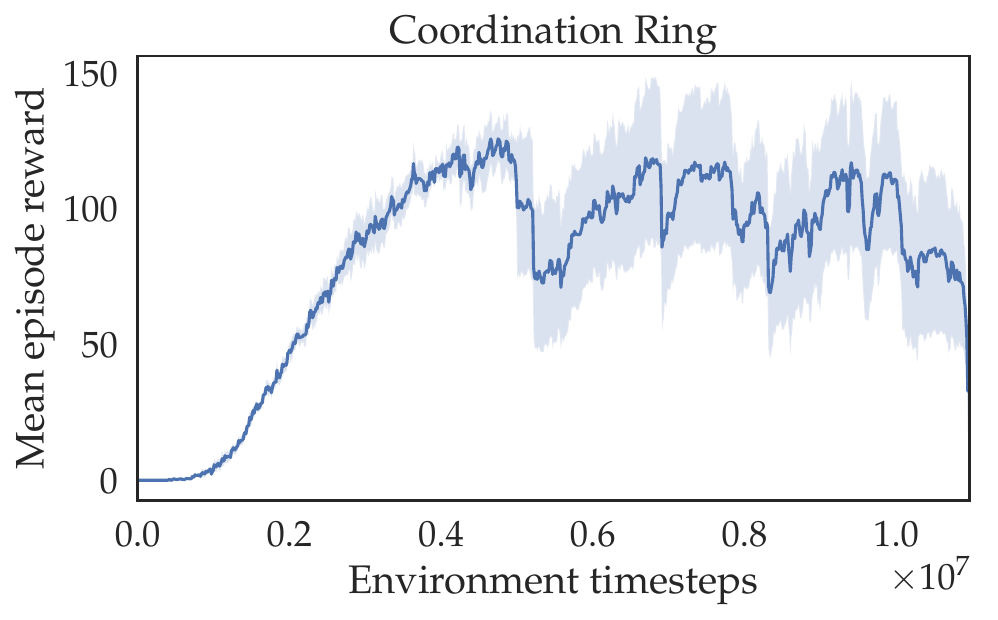} & \includegraphics[width=\linewidth]{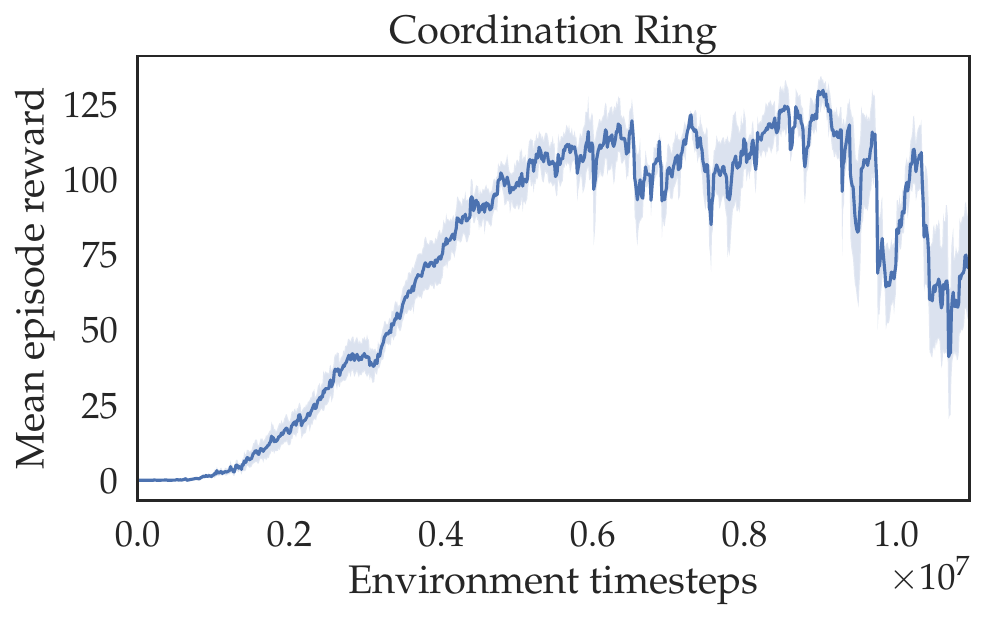} \\
      \includegraphics[width=\linewidth]{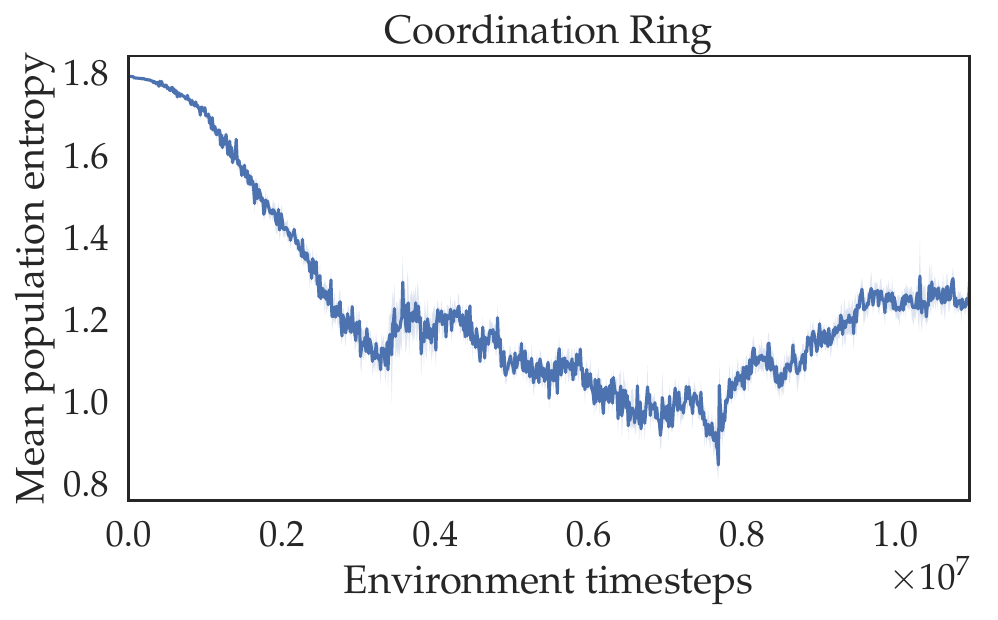} & \includegraphics[width=\linewidth]{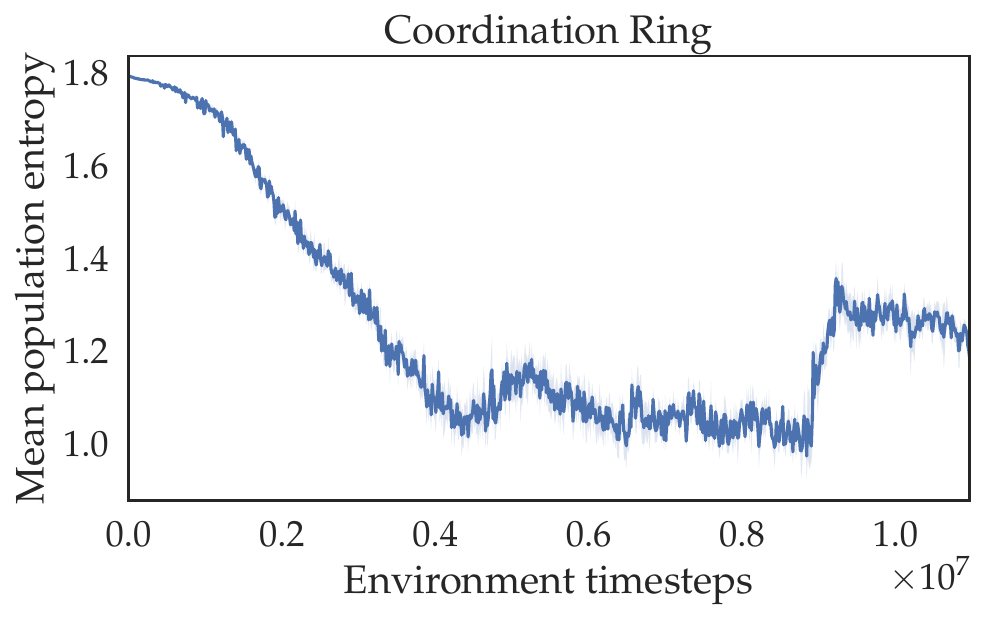} & \includegraphics[width=\linewidth]{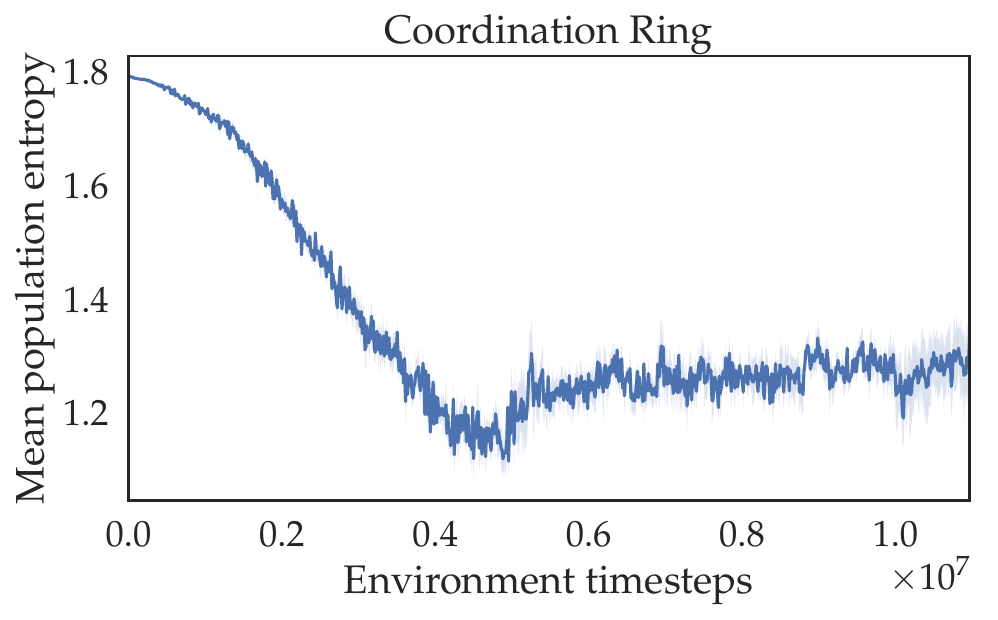} & \includegraphics[width=\linewidth]{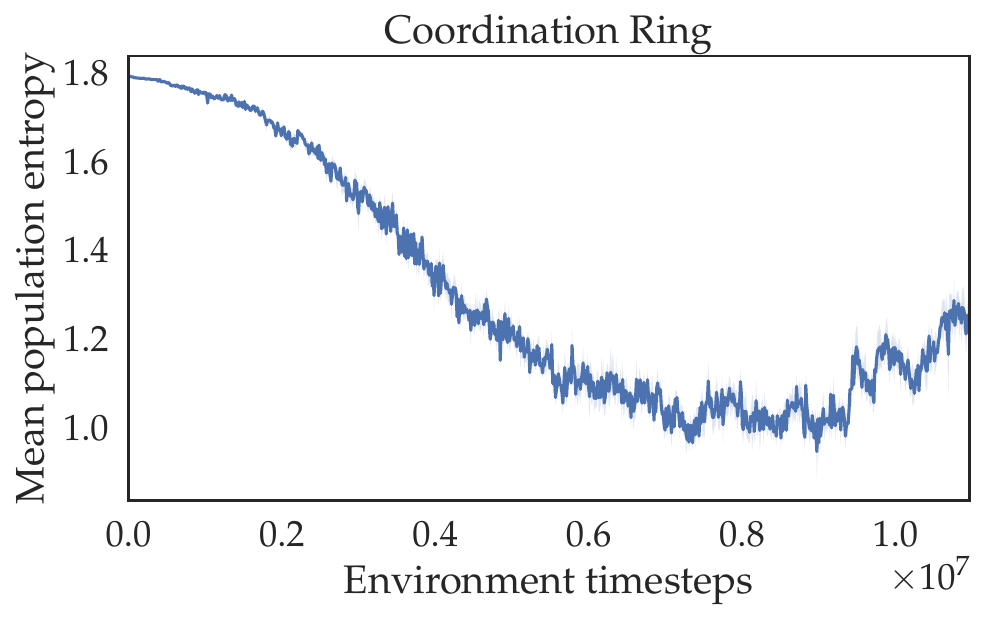} \\ \midrule
      \includegraphics[width=\linewidth]{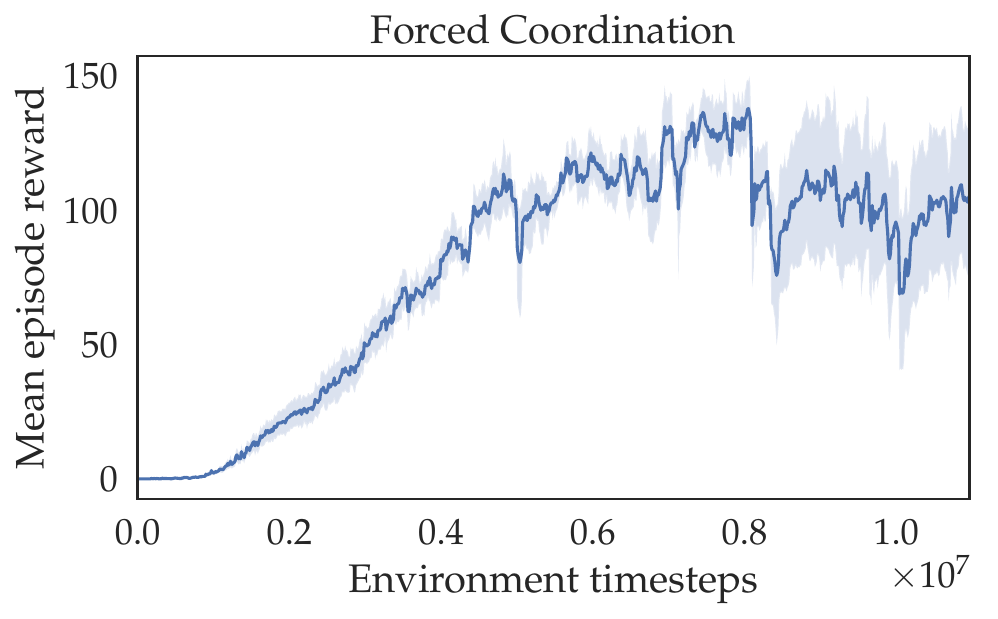} & \includegraphics[width=\linewidth]{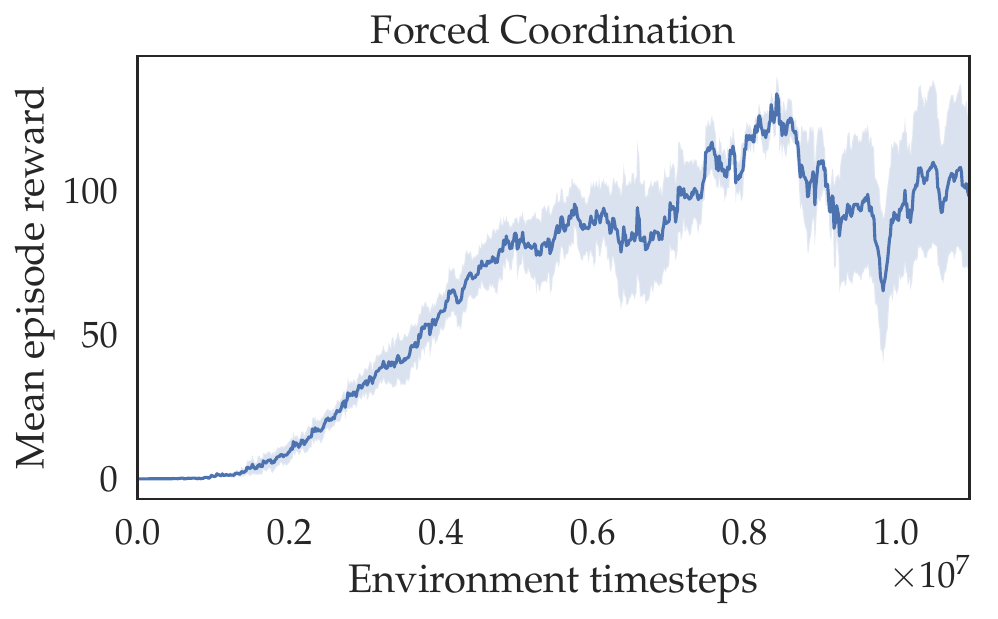} & \includegraphics[width=\linewidth]{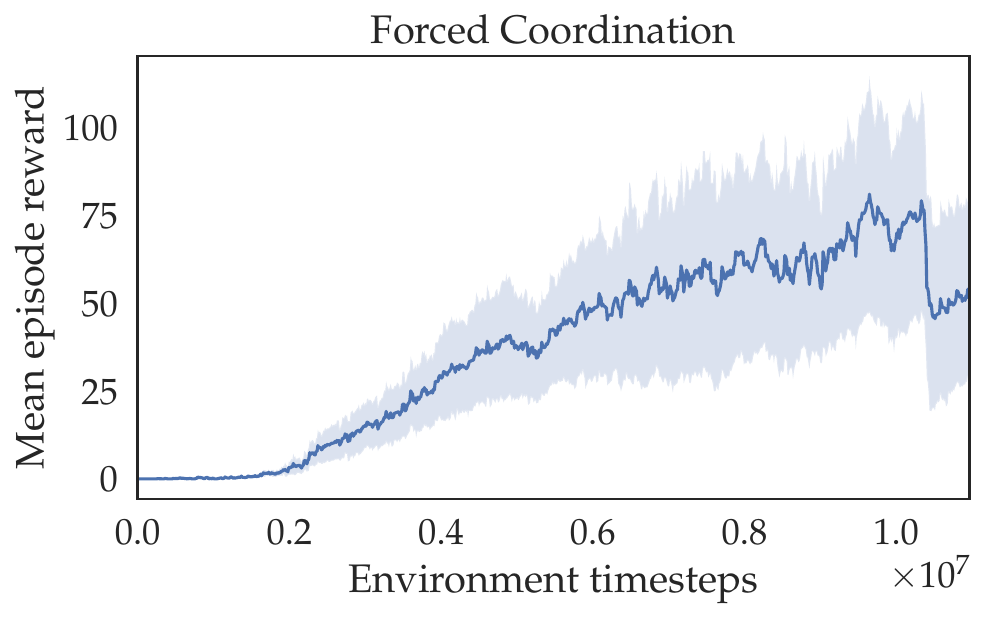} & \includegraphics[width=\linewidth]{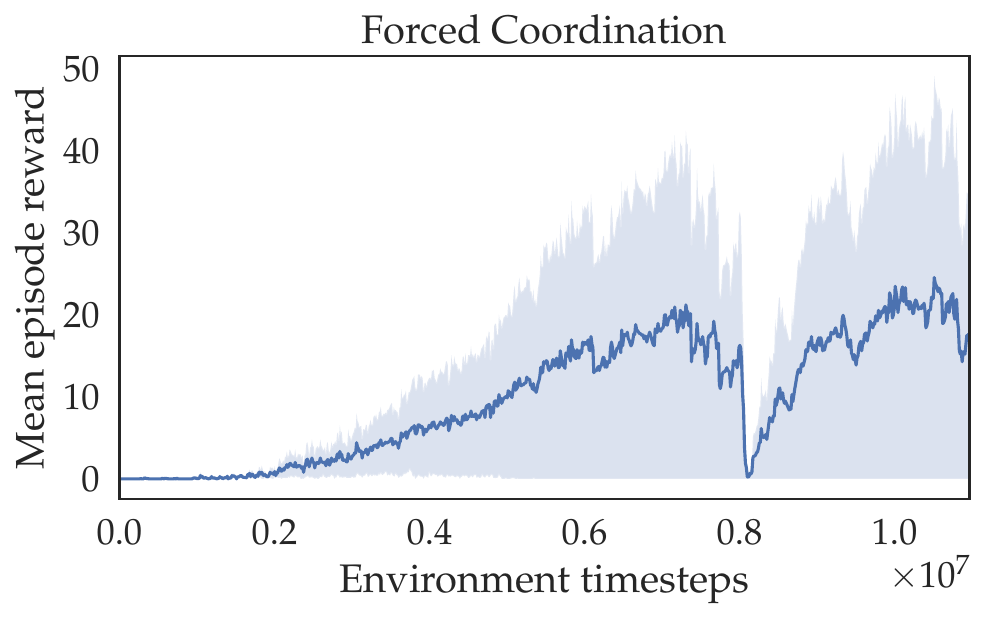} \\
      \includegraphics[width=\linewidth]{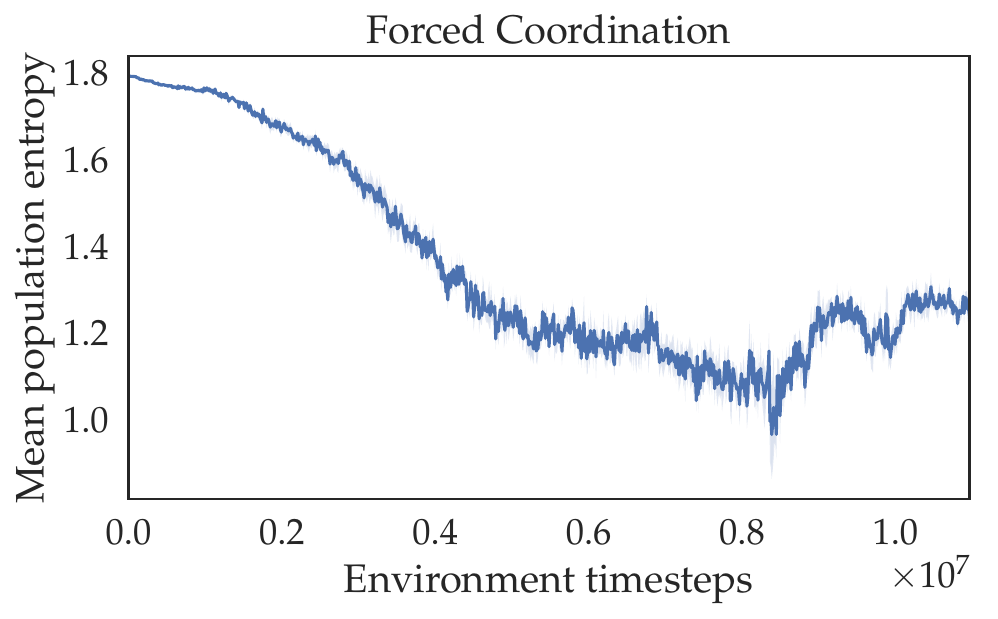} & \includegraphics[width=\linewidth]{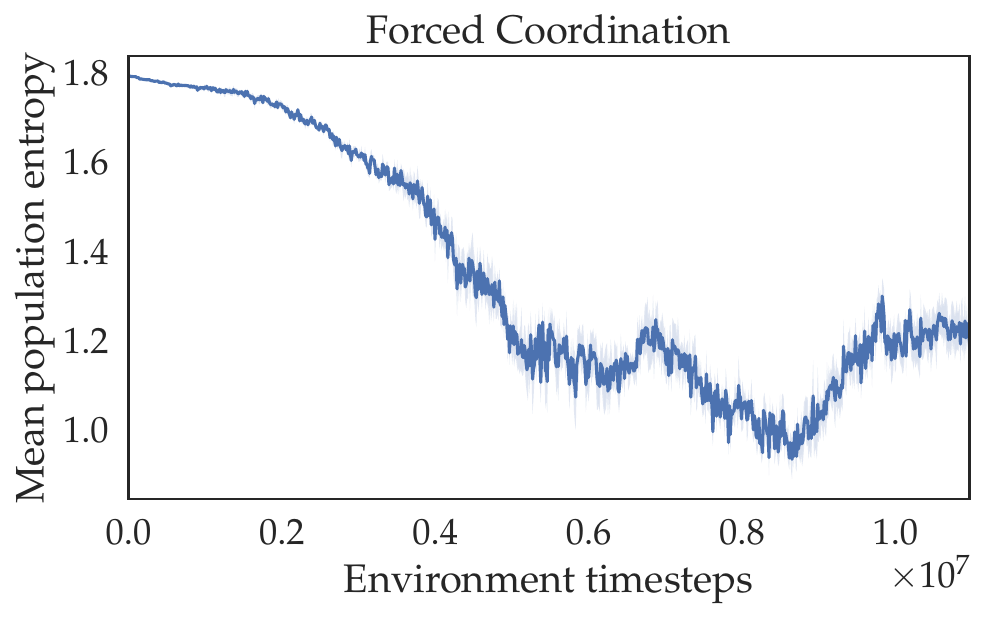} & \includegraphics[width=\linewidth]{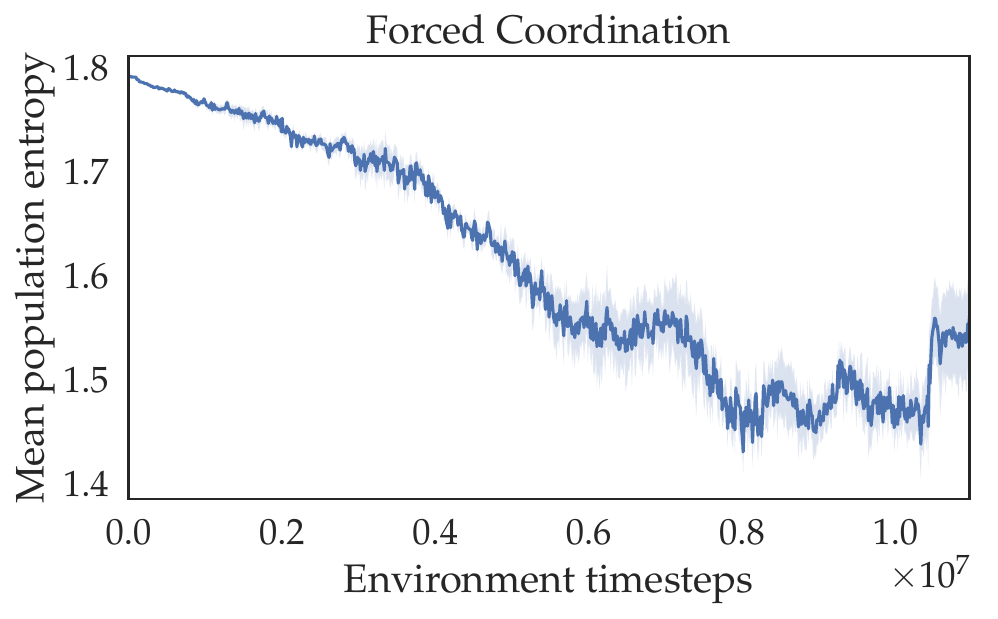} & \includegraphics[width=\linewidth]{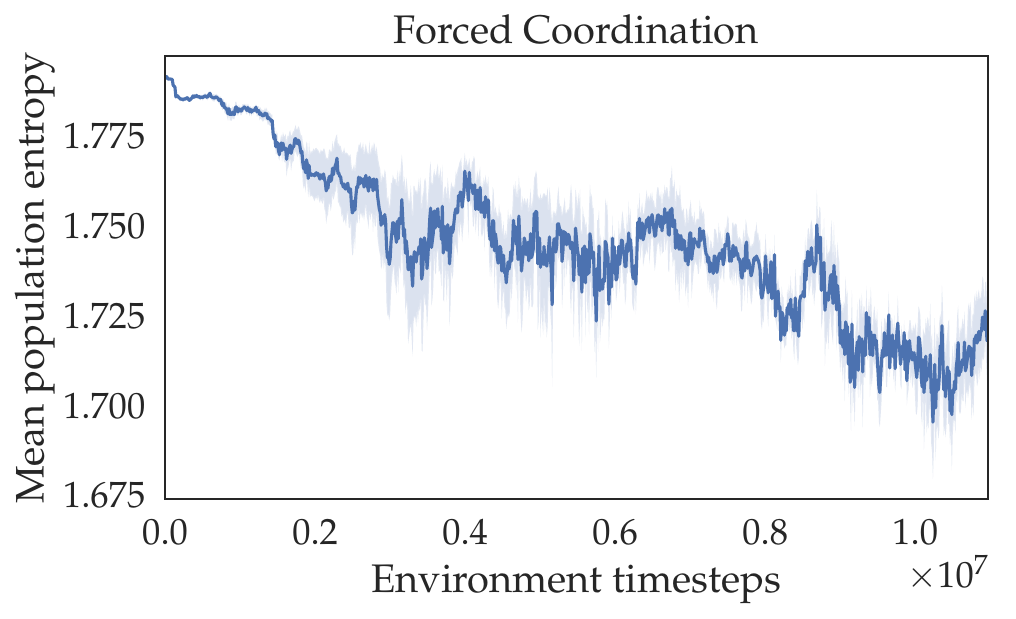} \\ \midrule
      \includegraphics[width=\linewidth]{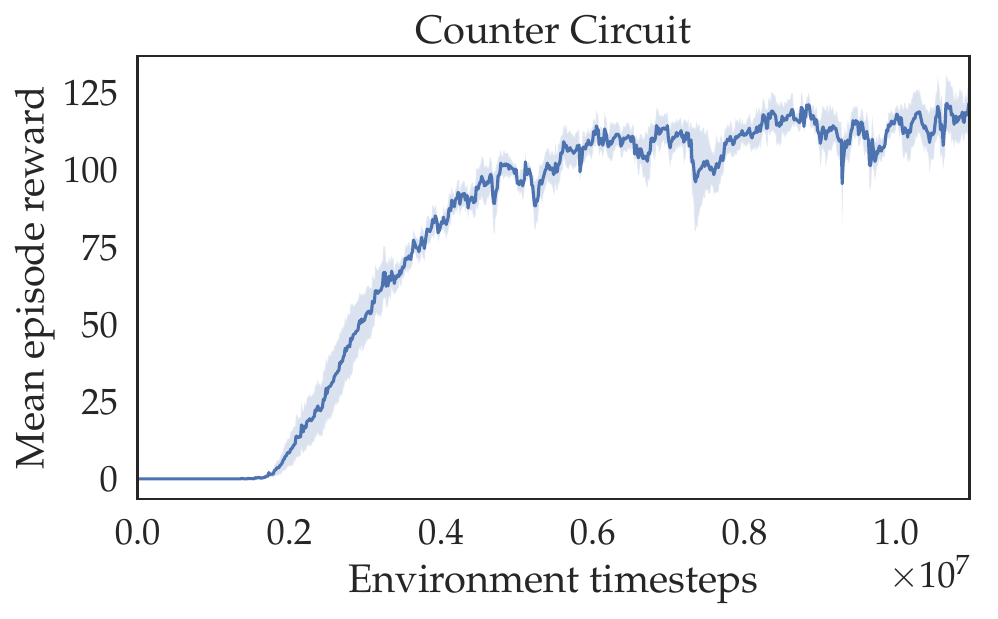} & \includegraphics[width=\linewidth]{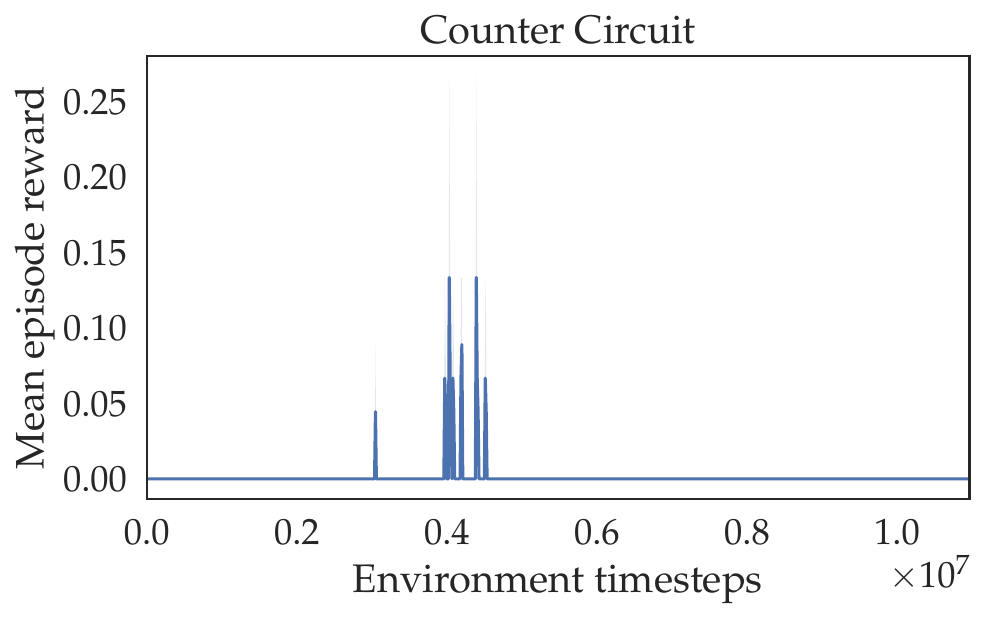} & \includegraphics[width=\linewidth]{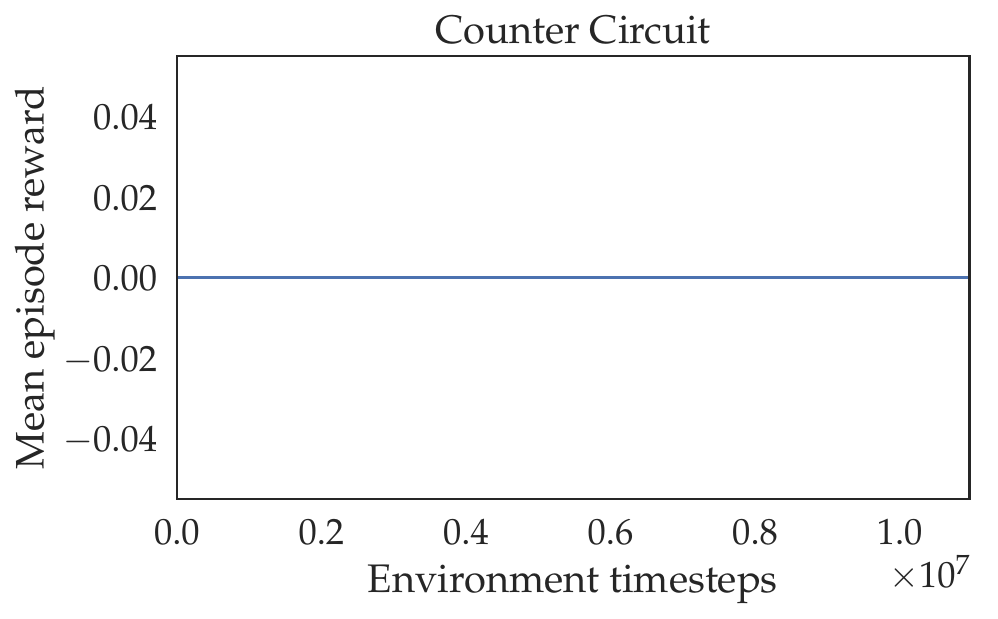} & \includegraphics[width=\linewidth]{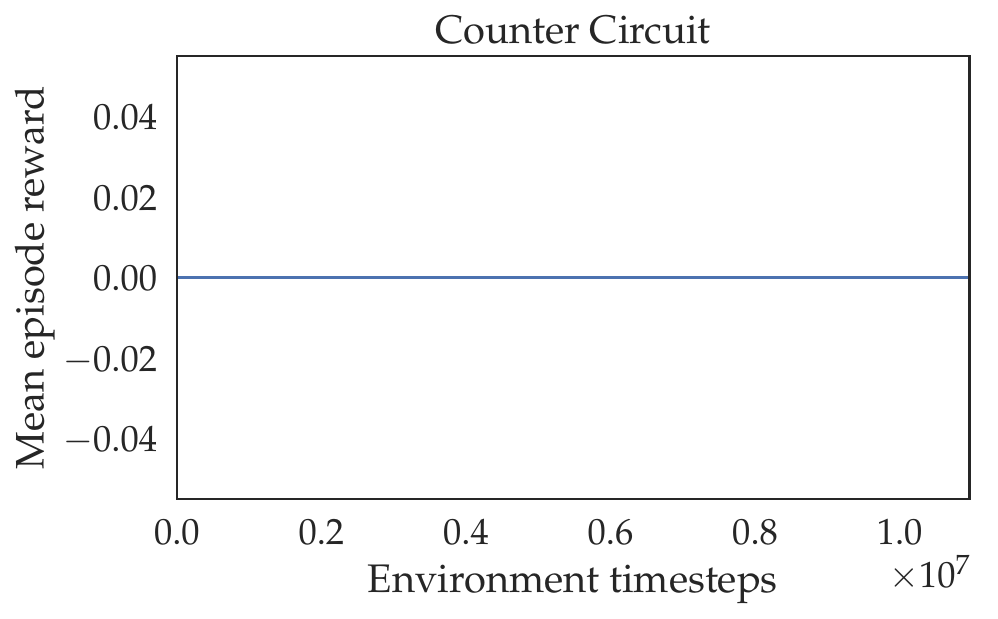} \\
      \includegraphics[width=\linewidth]{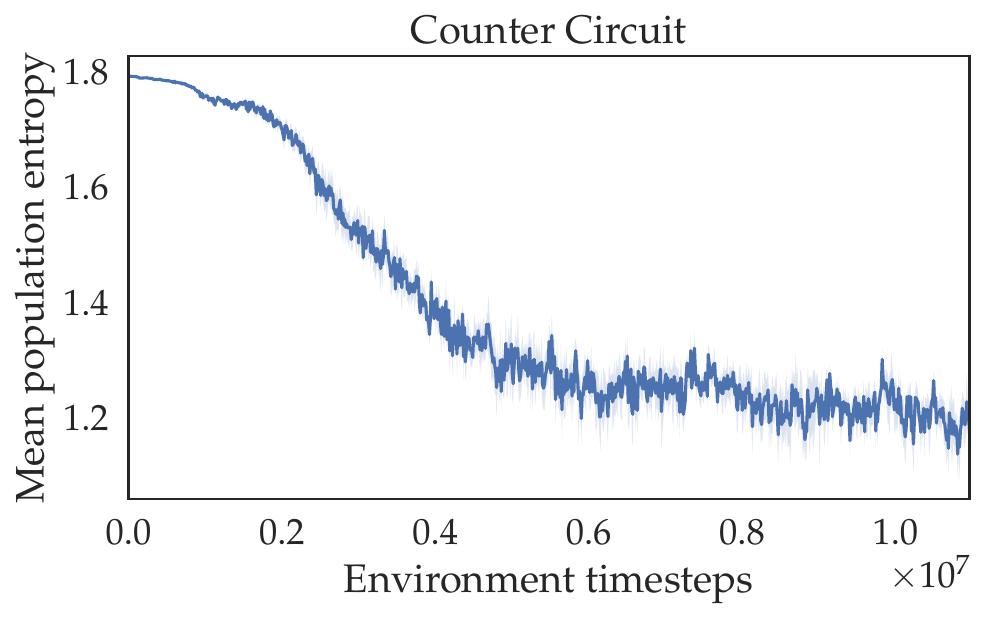} & \includegraphics[width=\linewidth]{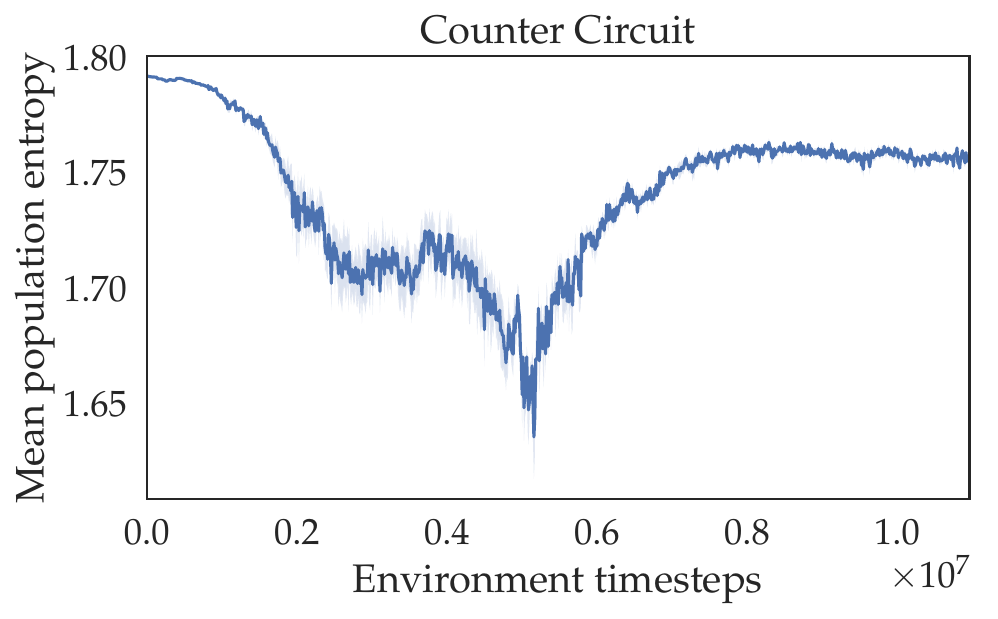} & \includegraphics[width=\linewidth]{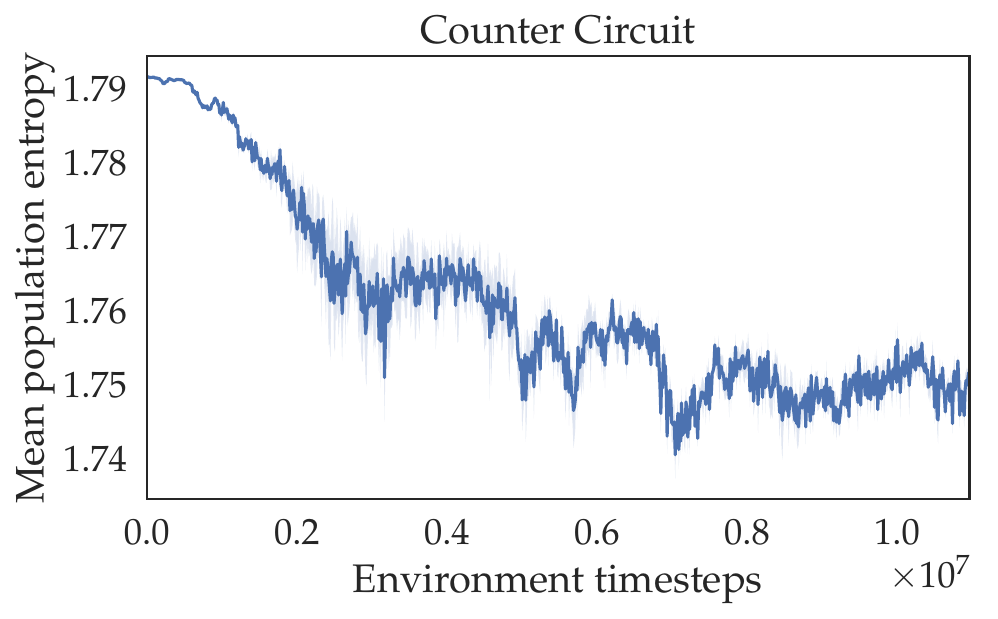} & \includegraphics[width=\linewidth]{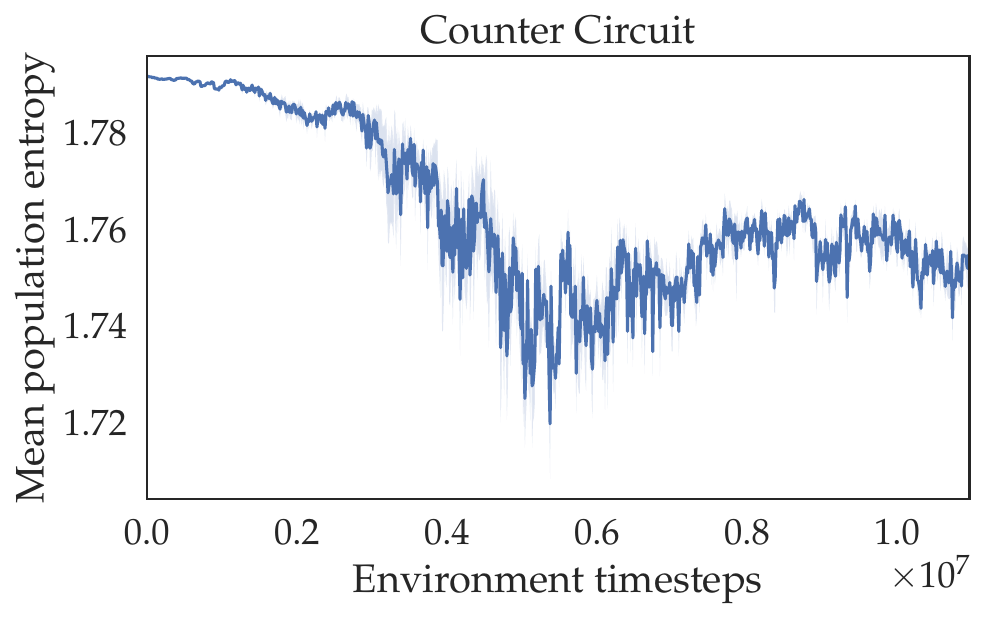} \\
      \bottomrule
  \end{tabular}
  \caption{\bo{Mean episode reward and population entropy with different $\alpha$ in all five layouts:} Each column corresponds to a different value of $\alpha$ in the set of $[0.020,\ 0.030,\ 0.040,\ 0.050]$. There are five row sections, which correspond to the five layouts. Each row section contains two rows, which are the plots of the mean episode reward and the mean population entropy of the layout, respectively.}
  \label{tab:rew_and_ent_2}
\end{table}


\end{document}